\documentclass[11pt]{article}

\usepackage{jcd}
\usepackage{conformal_style}
\usepackage{xspace}
\usepackage{overpic}
\usepackage{tikz}
\usepackage{authblk} 

\usemedgeometry
\bluehyperref

\newcommand{\ltwopxs}[1]{\norms{#1}_{L^2(P_X)}}
\newcommand{\ltwop}[1]{\norm{#1}_{L^2(P)}}

\title{Knowing what you know: valid and validated confidence sets
  \\ in multiclass and multilabel prediction\footnote{Research supported
  by NSF CAREER award CCF-1553086,
  ONR Young Investigator Program award N00014-19-2288, and
  NSF award HDR-1934578 (the Stanford Data Science Collaboratory)}}

\author[1]{Maxime Cauchois\footnote{Equal contribution.}}
\newcommand\CoAuthorMark{\footnotemark[\arabic{footnote}]} 
\author[1]{Suyash Gupta\protect\CoAuthorMark}
\author[1, 2]{John C. Duchi}
\affil[1]{Department of Statistics, Stanford University}
\affil[2]{Department of Electrical Engineering, Stanford University}
\affil[ ]{\texttt{\{maxcauch, suyash28, jduchi\}@stanford.edu}}
\date{}
\begin{document}

\newcommand{\paperVersion}{arxiv}  

\maketitle

\begin{abstract}

We develop conformal prediction methods for constructing valid predictive 
confidence sets in multiclass and multilabel problems without assumptions on
the data generating distribution.  A challenge here is that
typical conformal prediction methods---which give marginal validity
(coverage) guarantees---provide uneven coverage, in that they address easy
examples at the expense of essentially ignoring difficult examples.  By
leveraging ideas from quantile regression, we build methods that always
guarantee correct coverage but additionally provide (asymptotically optimal)
conditional coverage for both multiclass and multilabel prediction problems.
To address the potential challenge of exponentially large confidence sets in
multilabel prediction, we build tree-structured classifiers that efficiently
account for interactions between labels.  Our methods can be
bolted on top of any classification model---neural network, random forest,
boosted tree---to guarantee its validity.  We also provide an empirical
evaluation, simultaneously providing new validation methods,
that suggests the more robust coverage of our confidence sets.

\end{abstract}


\section{Introduction}
\label{sec:introduction}

The average accuracy of a machine-learned model by itself is insufficient to
trust the model's application; instead, we should ask for valid confidence
in its predictions.  Valid here does not mean ``valid under modeling
assumptions,'' or ``trained to predict confidence,'' but honest validity,
independent of the underlying distribution. In particular, for a supervised
learning task with inputs $x \in \mc{X}$, targets $y \in \mc{Y}$, and a
given confidence level $\alpha \in (0, 1)$, we seek confidence sets $C(x)$
such that $P(Y \in C(X)) \ge 1 - \alpha$; that is, we \emph{cover} the true
target $Y$ with a given probability $1-\alpha$. Given the growing importance
of statistical learning in real-world applications---autonomous
vehicles~\cite{KalraPa16}, skin lesion
identification~\cite{EstevaKuNoKoSwBlTh17,OakdenDuCaRe20}, loan repayment
prediction~\cite{HardtPrSr16}---such validity is essential.

The typical approach in supervised learning is to learn
a scoring function $s : \mc{X} \times \mc{Y} \to \R$ where
high scores $s(x, y)$ mean that $y$ is more likely for a given $x$.
Given such a score, a natural goal for prediction with confidence is
to compute a quantile function $q_\alpha$ satisfying
\begin{align}
  \label{eq:intro-scorequantile}
  P\left( s(x,Y) \geq q_\alpha(x) \mid X=x \right) \geq 1-\alpha,
\end{align}
where $\alpha > 0$ is some \emph{a-priori} acceptable error level.  We could
then output conditionally valid confidence sets for each
$x \in \mc{X}$ at level $1-\alpha$ by
returning
\[
\left\{ y \in \mathcal{Y} \mid s(x,y) \geq q_\alpha(x) \right\}.
\]
Unfortunately, such conditional coverage is impossible without
either vacuously loose thresholds~\cite{Vovk12} or strong modeling
assumptions~\citep{Vovk12,BarberCaRaTi19a}, but this idea forms the basis
for our approach.

To address this impossibility, \emph{conformal inference} \cite{VovkGaSh05}
resolves instead to a marginal coverage guarantee: given $n$ observations
and a desired confidence level $1 - \alpha$, conformal  methods
construct confidence sets $C(x)$ such that for a new pair $(X_{n+1},
Y_{n+1})$ from the same distribution, $Y_{n + 1} \in C(X_{n+1})$ with
probability at least $1 - \alpha$, where the probability is \emph{jointly}
over $X$ and $Y$. Conformal inference algorithms can build upon arbitrary
predictors, neural networks, random forests, kernel methods and treat them
as black boxes, ``conformalizing'' them post-hoc.

This distribution-free coverage is only achievable marginally, and standard
conformal predictions for classification achieve it (as we see later) by
providing good coverage on easy examples at the expense of miscoverage on
harder instances.  We wish to provide
more uniform coverage, and we address
this in both multiclass---where each example belongs to a single
class---and multilabel---where each example may belong to several
classes---classification problems. We combine the ideas of
conformal prediction~\cite{VovkGaSh05} with an approach to fit a quantile
function $q$ on the scores $s(x,y)$ of the prediction model, which
\citet{RomanoPaCa19} originate for regression problems, and build
feature-adaptive quantile predictors that output sets of labels, allowing
us to guarantee valid marginal coverage (independent of the data generating
distribution) while better approximating the conditional
coverage~\eqref{eq:intro-scorequantile}.
A challenge is to evaluate whether we indeed do provide better than
marginal coverage, so we provide new validation methodology to test this as
well.

\paragraph{Conformal inference in classification}
For multiclass problems, we propose a method that fits a quantile function on
the scores, conformalizing it on held-out data.  While this immediately
provides valid marginal coverage, the accuracy of the quantile
function---how well it approximates the conditional quantiles---determines
conditional coverage performance.  Under certain consistency assumptions on
the learned scoring functions and quantiles as sample size increases, we
show in Section~\ref{sec:multiclass} that we recover the conditional
coverage~\eqref{eq:intro-scorequantile} asymptotically.

The multilabel case poses new statistical and computational challenges, as a
$K$-class problem entails $2^K$ potential responses.  In this case, we seek
efficiently representable inner and outer sets $C_\text{in}(x)$ and
$C_\text{out}(x) \subset \{1, \ldots, K\}$ such that
\begin{equation}
  \label{eqn:io-intro}
  \P ( C_{\textup{in}}(X) \subset Y \subset C_{\textup{out}}(X) ) \geq 1 - \alpha.
\end{equation}
We propose two approaches to guarantee the
containments~\eqref{eqn:io-intro}.
The first directly fits inner and outer sets by solving two separate
quantile regression problems.  The second begins by observing that labels
are frequently correlated---think, for example, of chairs, which frequently
co-occur with a table---and learns a tree-structured graphical
model~\cite{KollerFr09} to efficiently address such correlation.  We show
how to build these on top of any predictive model. In an extension
when the sets~\eqref{eqn:io-intro} provide too imprecise confidence sets,
we show how to also construct a small number of
sets $C^{(i)}_{\textup{in/out}}$ that similarly
satisfy $\P(\cup_i \{C^{(i)}_{\textup{in}}(X) \subset Y
\subset C^{(i)}_{\textup{out}}(X)) \ge 1 - \alpha$ while
guaranteeing $C_{\textup{in}}(x) \subset \cup_i C^{(i)}_{\textup{in}}(x)$
and $C_{\textup{out}}(x) \supset \cup_i C^{(i)}_{\textup{out}}(x)$.


\paragraph{Related work and background}
\citet{VovkGaSh05} introduce (split-)conformal inference,
which splits the first $n$ samples of the exchangeable pairs $\{(X_i, Y_i)\}_{i=1}^{n+1}$ into two sets (say, each of size $n_1$ and $n_2$ respectively) where the first training set ($\mathcal{I}_1$) is used to learn a scoring function $s : \mc{X} \times \mc{Y} \to \R$ and the second validation set ($\mathcal{I}_2$) to ``conformalize'' the scoring function
and construct a confidence set over potential
labels (or targets) $\mc{Y}$ of the form
\begin{equation*}
  C(x) \defeq \{y \in \mc{Y} \mid s(x, y) \ge t\}
\end{equation*}
for some threshold $t$.  The basic split-conformal method chooses the
$(1+1/n_2)(1-\alpha)$-empirical quantile $\quantmarg$ of the negative scores
$\{-s(X_i, Y_i)\}_{i \in \mathcal{I}_2}$ (on the validation set) and defines
\begin{equation}
  \label{eqn:marginal-method}
  C(x) \defeq \{ y \in \mathcal{Y} \mid s(x,y) \geq
  -\quantmarg \}.
\end{equation}
The argument that these sets $C$ provide valid coverage is beautifully
simple and is extensible given any scoring function $s$: letting $\score_i =
-s(X_i, Y_i)$, if $\quantmarg$ is the
$\left(1+1/n_2\right)(1-\alpha)$-quantile of $\{\score_i\}_{i \in
  \mathcal{I}_2}$ and the pairs $\{(X_i, Y_i)\}_{i=1}^{n+1}$ are
exchangeable, we have
\begin{align*}
 \P(Y_{n + 1} \in C(X_{n+1}))
 & = \P(s(X_{n+1}, Y_{n+1}) \ge -\quantmarg) \\
 & = \P\left( \mbox{Rank~of}~ S_{n+1} ~ \mbox{in} ~ \{S_i\}_{i \in \mc{I}_2 \cup \{n+1\}} \le \ceil{(n_2+1)(1-\alpha)}  \right)  \ge 1-\alpha.
\end{align*}
We refer to the procedure~\eqref{eqn:marginal-method} as the \emph{Marginal}
conformal prediction method. Such a ``conformalization'' scheme typically
outputs a confidence set by listing all the labels that it contains, which
is feasible in a $K$-class multiclass problem, but more challenging in
a $K$-class multilabel one, as the number of configurations ($2^K$)
grows exponentially. This, to the best of our knowledge,
has completely precluded efficient conformal methods for
multilabel problems.

While conformal inference guarantees marginal
coverage without assumption on the distribution generateing the data,
\citet{Vovk12} shows it is virtually impossible to attain distribution-free
conditional coverage, and \citet{BarberCaRaTi19a} prove that in
regression, one can only achieve a weaker form of approximately-conditional
coverage without conditions on the underlying distribution.
Because of this theoretical limitation, work in conformal inference often
focuses on minimizing confidence set sizes or guaranteeing different types
of coverage.  For instance, \citet{SadinleLeWa19} propose conformal
prediction algorithms for multiclass problems that minimize the expected
size of the confidence set and conformalize the scores separately for each
class, providing class-wise coverage.  In the same vein,
\citet{HechtlingerPoWa19} use density estimates of $p(x \mid y)$ as
conformal scores to build a ``cautious'' predictor, the idea being that it
should output an empty set when the new sample differs too much from the
original distribution.  In work building off of the initial post of this
paper to the \texttt{arXiv}, \citet{RomanoSeCa20} build conformal confidence
sets for multi-class problems by leveraging pivotal $p$-value-like
quantities, which provide conditional coverage when models
are well-specified.  In regression problems, \citet{RomanoPaCa19}
conformalize a quantile predictor, which allows them to build marginally
valid confidence sets that are adaptive to feature heterogeneity and
empirically smaller on average than purely marginal confidence sets.  We
adapt this regression approach to classification tasks, learning quantile
functions to construct valid---potentially conditionally valid---confidence
predictions.

\paragraph{Notation}
$\mathcal{P}$ is a set of distributions on $\mathcal{X} \times \mathcal{Y}$,
where $\mathcal{Y} = \{ 0, 1 \}^K$ is a discrete set of labels, $K \geq 2$
is the number of classes, and $\mathcal{X} \subset \R^d$.  We assume that we
observe a finite sequence $(X_i, Y_i)_{1 \leq i \leq n} \simiid P$ from some
distribution $P = P_X \times P_{Y \mid X} \in \mathcal{P}$ and wish to
predict a confidence set for a new example $X_{n+1} \sim P_X$.  $\mathbb{P}$
stands over the randomness of both the new sample $(X_{n+1},Y_{n+1})$ and
the full procedure. We tacitly identify the vector $Y \in
\mathcal{Y} = \{0,1\}^K$ with the subset of $[K] = \{ 1, \ldots, K \}$ it
represents, and the notation $\mathcal{X} \rightrightarrows [K]$ indicates a
set-valued mapping between $\mathcal{X}$ and $[K]$. We define
$\ltwop{f}^2 \defeq \int f^2 dP$.


\section{Conformal multiclass classification}
\label{sec:multiclass}
\label{subsec:CQC}

We begin with multiclass classification problems, developing
\emph{Conformalized Quantile Classification} (CQC) to construct finite
sample marginally valid confidence sets.  CQC is similar to
\citeauthor{RomanoPaCa19}'s Conformalized Quantile Regression
(CQR)~\cite{RomanoPaCa19}: we estimate a quantile function of the scores,
which we use to construct valid confidence sets after conformalization.
We split the data into subsets $\mathcal{I}_1$, $\mathcal{I}_2$,
and $\mathcal{I}_3$ with sample sizes $n_1, n_2$ and $n_3$,
where $\mc{I}_3$ is disjoint from $\mc{I}_1 \cup \mc{I}_2$.
Algorithm~\ref{alg:cqc} outlines the basic idea: we use the set $\mc{I}_1$ for
fitting the scoring function with an (arbitrary) learning algorithm
$\mc{A}$, use $\mc{I}_2$ for fitting a quantile function from a family
$\mc{Q} \subset \mc{X} \to \R$ of possible quantile functions to the
resulting scores, and use $\mc{I}_3$ for calibration.  In the algorithm, we
recall the ``pinball loss''~\cite{KoenkerBa78} $\pinball_\alpha(t) = (1 -
\alpha) \hinge{-t} + \alpha \hinge{t}$, which satisfies $\argmin_{q \in \R}
\E[\pinball_\alpha(Z - q)] = \inf \{q \mid \P(Z \le q) \ge \alpha\}$ for any
random variable $Z$.

\algbox{\label{alg:cqc}
  Split Conformalized Quantile Classification
  (CQC).
}{
  \textbf{Input:} Sample $\{(X_i, Y_i)\}_{i=1}^n$, index sets
  $\mc{I}_1, \mc{I}_2, \mc{I}_3 \subset [n]$, fitting algorithm
  $\mc{A}$, quantile functions $\mc{Q}$, and desired confidence level
  $\alpha$
  \begin{enumerate}[1.]
  \item Fit scoring function via
    \begin{equation}
      \what{s} \defeq \mc{A}\left((X_i, Y_i)_{i \in \mc{I}_1}\right).
      \label{eqn:fit-scores}
    \end{equation}
  \item Fit quantile function via
    \begin{equation}
      \label{eqn:quantile-regression}
      \what{q}_\alpha \in
      \argmin_{q \in \mc{Q}} \bigg\{\frac{1}{|\mc{I}_2|}
      \sum_{i \in \mc{I}_2}
      \pinball_\alpha\left(\what{s}(X_i, Y_i) - q(X_i)\right) \bigg\}
    \end{equation}
  \item Calibrate by computing conformity scores
    $\score_i = \what{q}_\alpha(X_i) - \what{s}(X_i, Y_i)$, defining
    \begin{equation*}
      \quant_{1 - \alpha}(\scores, \mc{I}_3)
      \defeq (1 + 1/n_3) (1 - \alpha) ~ \mbox{empirical~quantile~of~}
      \{\score_i\}_{i \in \mc{I}_3}
    \end{equation*}
    and return prediction set function
    \begin{equation}
      \label{eqn:cqc-prediction-set}
      \what{C}_{1 - \alpha}(x) \defeq
      \left\{k \in [K] \mid \what{s}(x, k) \ge \what{q}_\alpha(x)
      - \quant_{1-\alpha}(\scores, \mc{I}_3)\right\}.
    \end{equation}
  \end{enumerate}
}

\subsection{Finite sample validity of CQC}

We begin by showing that Alg.~\ref{alg:cqc} enjoys the coverage guarantees
we expect, which is more or less immediate by the
marginal guarantees for the method~\eqref{eqn:marginal-method}. 
We include the proof for completeness and because its cleanness highlights
the ease of achieving validity.

None of our guarantees for Alg.~\ref{alg:cqc} explicitly requires that we
fit the scoring function and the quantile function on disjoint subsets
$\mc{I}_1$ and $\mc{I}_2$.
%
%
We assume for simplicity that the full sample $\{(X_i,Y_i)\}_{i=1}^{n+1}$ is
exchangeable, though we only require the sample $\{(X_i,Y_i)\}_{i \in
  \mc{I}_3 \cup \{n+1\}}$ be exchangeable conditionally on
$\{(X_i,Y_i)\}_{i \in \mc{I}_1 \cup \mc{I}_2}$. In general, so long as the
score $\what{s}$ and quantile $\what{q}_{1-\alpha}$ functions are measurable
with respect to $\{(X_i, Y_i)\}_{i \in \mc{I}_1 \cup \mc{I}_2}$
and a $\sigma$-field independent
of $\{(X_i, Y_i)\}_{i \in \mc{I}_3 \cup \{n+1\}}$,
then Theorem~\ref{thm:marginal-coverage}
remains valid if the exchangeability
assumption holds for the instances in $\mc{I}_3 \cup \{ n+1\}$.

\begin{theorem}
  \label{thm:marginal-coverage}
  Assume that $\{(X_i,Y_i)\}_{i=1}^{n+1} \sim P$ are exchangeable, where
  $P$ is an arbitrary distribution. Then the
  prediction set $\what{C}_{1-\alpha}$ of the split CQC
  Algorithm~\ref{alg:cqc} satisfies
  \begin{equation*}
    \P\{Y_{n+1}\in \what{C}_{1-\alpha}(X_{n+1})\} \geq 1-\alpha.
  \end{equation*}
\end{theorem}
\begin{proof}
  The argument is due to \citet[Thm.~1]{RomanoPaCa19}. Observe that
  $Y_{n+1} \in \what{C}_{1-\alpha}(X_{n+1})$ if and only if
  $\score_{n+1} \leq \quant_{1-\alpha}(\scores,\mathcal{I}_3)$.
  Define the $\sigma$-field $\mc{F}_{12} = \sigma \left\{ (X_i,Y_i) \mid i \in
  \mc{I}_1 \cup \mc{I}_2 \right\}$. Then
  \begin{align*}
    \P(Y_{n+1} \in \what{C}_{1-\alpha}(X_{n+1})\mid \mc{F}_{12} ) =
    \P( \score_{n+1} \leq Q_{1-\alpha}(\scores,\mathcal{I}_3)\mid \mc{F}_{12}).
  \end{align*}
  We use the following lemma.
  \begin{lemma}[Lemma 2, \citet{RomanoPaCa19}]
    \label{lem:inflated_quantiles}
    Let $Z_1, \ldots, Z_{n+1}$ be exchangable random variables and
    $\what{\quant}_n(\cdot)$ be the empirical quantile function of $Z_1, \ldots
    Z_n$. Then for any $\alpha \in (0,1)$,
    \begin{align*}
      \P\big(Z_{n+1} \leq \what{\quant}_n((1+n^{-1})\alpha)\big) \geq  \alpha.
    \end{align*}
    If $Z_1, \ldots, Z_n$ are almost surely distinct,
    then
    \begin{align*}
      \P\big(Z_{n+1} \leq \what{\quant}_n((1+n^{-1})\alpha)\big)
      \leq \alpha +\frac{1}{n}.
    \end{align*}
  \end{lemma}
  
  As the original sample is exchangeable, so are the conformity scores
  $\score_i$ for $i \in \mathcal{I}_3$, conditionally on $\mc{F}_{12}$.
  Lemma~\ref{lem:inflated_quantiles} implies $\P(\score_{n+1} \leq
  Q_{1-\alpha}(\scores,\mathcal{I}_3) \mid \mc{F}_{12}) \geq 1-\alpha$, and taking
  expectations over $\mc{F}_{12}$ yields the theorem.
\end{proof}

The conditional distribution of scores given $X$ is discrete, so the
confidence set may be conservative: it is possible
that for any $q$ such that $P(s(X,Y)\geq q\mid X) \geq 1-\alpha$, we have
$P(s(X,Y)\geq q\mid X) \geq 1-\epsilon$ for some $\epsilon \ll \alpha.$
Conversely, as the CQC procedure~\ref{alg:cqc} seeks $1-\alpha$ marginal
coverage, it may sacrifice a few examples to bring the coverage down to
$1-\alpha$.  Moreover, there may be no unique quantile function for the
scores. One way to address these issues
is to estimate a quantile function on $\mc{I}_2$ (recall
step~\eqref{eqn:quantile-regression} of the CQC method) so that we can
guarantee higher coverage (which is more conservative, but is free in the
$\epsilon \ll \alpha$ case).

An alternative to achieve exact $1-\alpha$ asymptotic coverage and a unique
quantile, which we outline here, is to randomize scores without changing the
relative order.
Let $Z_{i} \simiid \pi$ for some distribution $\pi$ with
continuous density  supported on the entire real line, and let $\sigma > 0$ be a noise parameter. Then
for any scoring function $s : \mathcal{X} \times \mathcal{Y} \to
\R$, we define
\begin{align*}
  s^{\sigma}(x, y, z) \defeq s(x, y, z) + \sigma z.
\end{align*}
As $s^\sigma(x, y, z) - s^\sigma(x, y', z) = s(x, y) - s(x, y')$,
this maintains the ordering of label scores, only giving
the score function a conditional density.
Now, consider replacing the quantile
estimator~\eqref{eqn:quantile-regression} with the randomized estimator
\begin{align}
  \label{eqn:quantile-regression-randomized}
  \what{q}^{\sigma}_{\alpha} \in \argmin_{q \in \mathcal{Q}}
  \bigg\{ \frac{1}{n_2} \sum_{i\in \mathcal{I}_2}
  \pinball_\alpha(\what{s}^{\sigma}(X_i,Y_i, Z_i) - q(X_i)) \bigg\},
\end{align}
and let $\score_{i}^{\sigma} \defeq \what{q}_{\alpha}^{\sigma}(X_i)
-\what{s}^\sigma(X_i,Y_i, Z_i)$ be the corresponding conformity
scores. Similarly, replace the prediction set~\eqref{eqn:cqc-prediction-set}
with
\begin{equation}
  \label{eqn:cqc-prediction-set-randomized}
  \what{C}_{1-\alpha}^\sigma(x, z) \defeq
  \left\{k \in [K] \mid \what{s}^{\sigma}(x, k,z) \ge \what{q}^\sigma_\alpha(x)
  - \quant_{1-\alpha}(\scores^\sigma, \mc{I}_3)\right\}.
\end{equation}
where $\quant_{1-\alpha}(\scores^{\sigma},\mathcal{I}_3)$ is the
$(1-\alpha)(1+1/n_3)$-th empirical quantile of $\{\score_i^{\sigma} \}_{i \in
  \mathcal{I}_3}$.  Then for a new input $X_{n+1} \in \mc{X}$, we simulate a
new independent variable $Z_{n+1} \sim \pi$, and give the confidence
set $\what{C}_{1-\alpha}^\sigma(X_{n+1}, Z_{n+1})$.
As the next result shows, this gives nearly perfectly calibrated
coverage.
\begin{corollary}
  \label{corollary:marginal-coverage_rand}
  Assume that $\{(X_i,Y_i)\}_{i=1}^{n+1} \sim P$ are exchangeable, where
  $P$ is an arbitrary distribution.  Let the
  estimators~\eqref{eqn:quantile-regression-randomized} and
  \eqref{eqn:cqc-prediction-set-randomized} replace the
  estimators~\eqref{eqn:quantile-regression}
  and~\eqref{eqn:cqc-prediction-set} in the CQC Algorithm~\ref{alg:cqc},
  respectively. Then the prediction set $\what{C}_{1-\alpha}^{\sigma}$
  satisfies
  \begin{equation*}
    1-\alpha \leq \P\{Y_{n+1}\in \what{C}_{1-\alpha}^{\sigma}(X_{n+1}, Z_{n+1})\}\leq 1-\alpha+\frac{1}{1+n_3}.
  \end{equation*}
\end{corollary}
\begin{proof}
  The argument is identical to that for Theorem~\ref{thm:marginal-coverage},
  except that we apply the second part of Lemma~\ref{lem:inflated_quantiles}
  to achieve the upper bound, as the scores are a.s.\ distinct.
\end{proof}

\subsection{Asymptotic optimality of CQC method}

Under appropriate assumptions typical in proving the consistency of
prediction methods, Conformalized Quantile Classification (CQC) guarantees
conditional coverage asymptotically. To set the stage, assume that as $n
\uparrow \infty$, the fit score functions $\what{s}_n$ in
Eq.~\eqref{eqn:fit-scores} converge to a fixed $s : \mc{X} \times \mc{Y} \to
\R$ (cf.\ Assumption~\ref{assumption:consistency-scores}).
Let
\begin{equation*}
q_{\alpha}(x) \defeq 
\inf \left\{z \in \R \mid \alpha \leq P\left( s(x,Y)\leq z \mid X=x \right) \right\}
\end{equation*}
be the $\alpha$-quantile function of the limiting scores $s(x, Y)$ and
for $\sigma > 0$ define
\begin{equation*}
  q^{\sigma}_{\alpha}(x) \defeq \inf \left\{z \in \R \mid \alpha \leq
  P\left(s^{\sigma}(x,Y, Z)\leq z \mid X=x\right) \right\},
\end{equation*}
where $Z \sim \pi$ (for a continuous distribution $\pi$ as in the preceding
section) is independent of $x, y$, to be the $\alpha$-quantile function of
the noisy scores $s^\sigma(x, Y, Z) = s(x, Y) + \sigma Z$.
With these, we can make the following natural definitions of
our desired asymptotic confidence sets.
\begin{definition}
  \label{definition:oracle-confidence-sets}
  The \emph{randomized-oracle} and \emph{super-oracle} confidence sets are
  $C^{\sigma}_{1-\alpha}(X,Z) \defeq \{k \in [K] \mid s^{\sigma}(X,k,Z) \geq
  q^{\sigma}_{\alpha}(X)\}$ and $C_{1-\alpha}(X)=\{k \in [K] \mid s(X,k)
  \geq q_{\alpha}(X\}$, respectively.
\end{definition}
\noindent
Our aim will be to show that the confidence sets of the
split CQC method~\ref{alg:cqc} (or its randomized variant) converge
to these confidence sets under appropriate consistency
conditions.

To that end, we consider the following consistency assumption.
\begin{assumption}[Consistency of scores and quantile functions]
  \label{assumption:consistency-scores}
  The score functions $\what{s}$ and quantile estimator
  $\what{q}_\alpha^\sigma$ are mean-square consistent, so that
  as $n_1, n_2, n_3 \to \infty$,  
  \begin{equation*}
    \ltwopxs{\what{s} - s}^2
    \defeq \int_{\mc{X}} \linfs{\what{s}(x, \cdot) - s(x, \cdot)}^2
    dP_X(x) \cp 0
  \end{equation*}
  and
  \begin{equation*}
    \ltwopxs{\what{q}_\alpha^\sigma - q^\sigma_\alpha}^2
    \defeq \int_{\mc{X}} (\what{q}_\alpha^\sigma(x) -
    q^\sigma_\alpha(x))^2 dP_X(x) \cp 0.
  \end{equation*}
\end{assumption}
\noindent
With this assumption, we have the following theorem, whose
proof we provide in Appendix~\ref{sec:proof-oracle}.
\begin{theorem}\label{thm:oracle}
  Let Assumption \ref{assumption:consistency-scores} hold. Then the confidence sets $\what{C}^{\sigma}_{1-\alpha}$ satisfy
  \begin{equation*}
    \lim_{n \to \infty}\P(\what{C}^{\sigma}_{1-\alpha}(X_{n+1}, Z_{n+1}) \neq C^{\sigma}_{1-\alpha}(X_{n+1}, Z_{n+1}))= 0.
  \end{equation*}
\end{theorem}

Unlike other work~\cite{SadinleLeWa19, RomanoSeCa20} in conformal
inference, validity of Theorem~\ref{thm:oracle} does not rely on the scores
being consistent for the log-conditional probabilities. Instead, we
require the $(1-\alpha)$-quantile function to be consistent for the scoring
function at hand, which is weaker (though still a strong
assumption).  Of course, the ideal scenario occurs when the limiting score
function $s$ is optimal (cf.~\cite{BartlettJoMc06, Zhang04a, TewariBa07}),
so that $s(x, y) > s(x, y')$ whenever $P(Y = y \mid X = x) > P(Y = y' \mid X
= x)$. Under this additional condition, the super-oracle confidence set
$C_{1-\alpha}(X)$ in Def.~\ref{definition:oracle-confidence-sets} is the
smallest conditionally valid confidence set at level
$1-\alpha$. Conveniently, our randomization is consistent as $\sigma
\downarrow 0$: the super oracle confidence set $C_{1-\alpha}$ contains
${C}^{\sigma}_{1-\alpha}$ (with high probability). We provide the proof of
the following result in Appendix~\ref{sec:proof-super-oracle}.
\begin{proposition}
  \label{proposition:super-oracle}
  The confidence sets ${C}^{\sigma}_{1-\alpha}$ satisfy 
  \begin{equation*}
    \lim_{\sigma \to 0}\P({C}^{\sigma}_{1-\alpha}(X_{n+1},Z_{n+1}) \subseteq C_{1-\alpha}(X_{n+1}))= 1.
  \end{equation*}
\end{proposition}
\noindent
Because we always have $\P(Y \in C^\sigma_{1-\alpha}(X, Z)) \ge 1 - \alpha$,
Proposition~\ref{proposition:super-oracle} shows that we maintain validity while
potentially shrinking the confidence sets.




\newcommand{\Cinner}{\what{C}_{\textup{in}}}
\newcommand{\Couter}{\what{C}_{\textup{out}}}

\section{The multilabel setting}
\label{sec:multilabel}

In multilabel classification, we observe a single feature vector
$x \in \mathcal{X}$ and wish to predict a vector $y \in \{-1, 1\}^K$ where $y_k
= 1$ indicates that label $k$ is present and $y_k = -1$ indicates its
absence. For example, in object detection problems~\cite{RedmonDiGiFa16,
TanPaLe19, LinMaBeHaPeRaDoZi14}, we wish to detect several entities in an
image, while in text classification tasks~\cite{LodhiSTCrWa02, McCallumNi98,
JoulinGrBoMi16}, a single document can potentially share multiple topics.
To conformalize such predictions, we wish to output an aggregated
set of predictions $\{\what{y}(x)\} \subset \{-1, 1\}^K$---a collection
of $-1$-$1$-valued vectors---that contains the true configuration
$Y = (Y_1, \ldots, Y_K)$ with probability at least $1 - \alpha$.

The multilabel setting poses statistical and computational challenges.
On the statistical side, multilabel prediction engenders a multiple testing
challenge: even if each task has an individual confidence set $C_k(x)$ such
that $P(Y_k \in C_k(X)) \geq 1-\alpha$, in general we can only conclude that
$P\left( Y \in C_1(X) \times \dots \times C_K(X)\right) \geq 1 - K\alpha$;
as all predictions share the same features $x$, we wish to leverage
correlation through $x$.  Additionally, as we discuss in the introduction
(recall Eq.~\eqref{eqn:marginal-method}), we require a scoring function $s:
\mathcal{X} \times \mathcal{Y} \to \R$.  Given a
predictor $\what{y} : \mc{X} \to \{-1, 1\}^k$, a naive scoring function
for multilabel problems is to
use $s(x, y) = \indic{\what{y}(x) = y}$, but this fails, as the
confidence sets contain either all configurations or a single configuration.
On the computational side, the total number of label configurations ($2^K$)
grows exponentially, so a ``standard'' multiclass-like approach outputting
confidence sets $\what{C}(X) \subset \mathcal{Y} = \{-1, 1\}^K$, although
feasible for small values of $K$, is computationally impractical even for
moderate $K$.

We instead propose using inner and outer set functions
$\what{C}_{\text{in}}, \what{C}_{\text{out}} : \mathcal{X} \rightrightarrows
[K]$ to efficiently describe a confidence set on $\mathcal{Y}$, where we
require they satisfy the coverage guarantee
\begin{subequations}
  \label{eqn:inner-outer-coverage}
  \begin{align}
    \P \left( \Cinner(X) \subset Y \subset
    \Couter(X) \right) \geq 1 - \alpha,
  \end{align}
  or equivalently, we learn two functions
  $\what{y}_{\text{in}}, \what{y}_{\text{out}} : \mathcal{X} \rightarrow
  \mathcal{Y} = \{-1,1\}^K$ such that
  \begin{equation}
    \P \left( \what{y}_{\text{in}}(X) \preceq Y   \preceq
    \what{y}_{\text{out}}(X) \right) \geq 1 - \alpha.
  \end{equation}
\end{subequations}
We thus say that coverage is \emph{valid} if the inner set exclusively contains
positive labels and the outer set contains at least all positive labels.
For example, $\what{y}_{\text{in}}(x) = \zeros$ and
$\what{y}_{\text{out}}(x) = \ones$ are always valid, though uninformative, while
the smaller the set difference between
$\Cinner(X)$ and $\Couter(X)$ the more confident we may be in a single
prediction. As we mention in the introduction, we extend
the inner/outer coverage guarantees~\eqref{eqn:inner-outer-coverage}
to construct to unions of such rectangles to allow more nuanced coverage;
see Sec.~\ref{sec:union-inner-outer-sets}.


In the remainder of this section, we propose methods to conformalize
multilabel predictors in varying generality, using tree-structured graphical
models to both address correlations among labels and computational
efficiency. We begin in Section~\ref{sec:generic-multilabel} with a general
method for conformalizing an arbitrary scoring function $s : \mc{X} \times
\mc{Y} \to \R$ on multilable vectors, which guarantees validity no matter
the score. We then show different strategies to efficiently implement the
method, depending on the structure of available scores, in
Section~\ref{sec:efficient-inner-outer-sets}, showing how
``tree-structured'' scores allow computational efficiency while modeling
correlations among the task labels $y$. Finally, in
Section~\ref{sec:multilabel-arbitrary}, we show how to build such a
tree-structured score function $s$ from both arbitrary predictive models
(e.g.~\cite{BoutellLuShBr04, ZhangZh06}) and those more common multilabel
predictors---in particular, those based on neural networks---that learn and
output per-task scores $s_k : \mc{X} \to \R$ for each
task~\cite{HerreraChRiDe16, ReadPfHoFr11}.





\subsection{A generic split-conformal method for multilabel sets}
\label{sec:generic-multilabel}

We begin by assuming we have a general score function $s : \mc{X} \times
\mc{Y} \to \R$ for $\mc{Y} = \{-1, 1\}^K$ that evaluates the quality of a
given set of labels $y \in \mc{Y}$ for an instance $x$; in the next
subsection, we describe how to construct such scores from multilabel
prediction methods, presenting our general method first.  We consider the
variant of the CQC method~\ref{alg:cqc} in Alg.~\ref{alg:tree-cqc}.

There are two considerations for Algorithm~\ref{alg:tree-cqc}: its
computational efficiency and its validity guarantees. Deferring the
efficiency questions to the coming subsections, we begin with the latter.
The naive approach is to simply use the ``standard'' or
implicit conformalization
approach, used for regression or classification, by defining
\begin{equation}
  \label{eqn:standard-classification-set}
  \Cimplicit(x) \defeq \left\{y \in \mc{Y}
  \mid s(x, y) \ge \what{q}_\alpha(x) - \what{\quant}_{1-\alpha}\right\},
\end{equation}
where $\what{q}_\alpha$ and $\what{\quant}_{1-\alpha}$ are as in the CQioC
method~\ref{alg:tree-cqc}. This does guarantee validity, as we have the
following corollary of Theorem~\ref{thm:marginal-coverage}.
\algbox{\label{alg:tree-cqc} Split Conformalized Inner/Outer
  method for classification (CQioC)}{
  \textbf{Input:} Sample $\{(X_i, Y_i)\}_{i=1}^n$, disjoint index
  sets $\mc{I}_2, \mc{I}_3 \subset [n]$, quantile functions $\mc{Q}$,
  desired confidence level $\alpha$, and score function
  $s : \mc{X} \times \mc{Y} \to \R$.
  \begin{enumerate}[1.]
  \item Fit quantile function $\what{q}_\alpha \in \argmin_{q \in
    \mc{Q}} \{ \sum_{i \in \mc{I}_2} \pinball_\alpha(s(X_i, Y_i) -
    q(X_i))\}$, as in~\eqref{eqn:quantile-regression}.
  \item Compute conformity scores
    $\score_i = \what{q}_\alpha(X_i) - s(X_i, Y_i)$, define
    $\what{\quant}_{1-\alpha}$ as
    $(1 + 1/|\mc{I}_3|) \cdot (1 - \alpha)$ empirical quantile of
    $\scores = \{\score_i\}_{i \in \mc{I}_3}$, and
    return prediction set function
    \begin{equation}
      \label{eqn:multilabel-io-set}
      \Cinout(x)
      \defeq 
      \left\{ y \in \mc{Y}  \mid
      \yin(x) \preceq y \preceq \yout(x) \right\},
    \end{equation}
    where $\yin$ and $\yout$ satisfy
    \begin{equation}
      \label{eqn:efficient-inner-outer-vecs}
      \begin{split}
        \yin(x)_k = 1 & ~ \mbox{if~and~only~if~}
        \max_{y \in \mc{Y} : y_k = -1} s(x, y) < \what{q}_\alpha(x)
        - \what{\quant}_{1-\alpha} \\
        \yout(x)_k = -1 & ~\mbox{if~and~only~if~}
        \max_{y \in \mc{Y} : y_k = 1} s(x, y) < \what{q}_{\alpha}(x)
        - \what{\quant}_{1-\alpha}.
      \end{split}
    \end{equation}
  \end{enumerate}
}

\begin{corollary}
  \label{corollary:smallest-set-valued}
  Assume
  that $\{(X_i, Y_i)\}_{i=1}^{n+1} \sim P$ are exchangeable, where
  $P$ is an arbitrary distribution.
  Then for any confidence set $\what{C} : \rightrightarrows \mc{Y}$
  satisfying $\what{C}(x) \supset \Cimplicit(x)$ for all $x$,
  \begin{equation*}
    \P(Y_{n+1} \in \what{C}(X_{n+1})) \ge 1 - \alpha.
  \end{equation*}
\end{corollary}
\noindent
Instead of the inner/outer set $\Cinout$ in
Eq.~\eqref{eqn:multilabel-io-set}, we could use any confidence set
$\what{C}(x) \supset \Cimplicit(x)$ and maintain
validity. Unfortunately, as we note
above, the set $\Cimplicit$ may be exponentially complex to
represent and compute, necessitating a simplifying construction, such as our
inner/outer approach. Conveniently, the set~\eqref{eqn:multilabel-io-set} we
construct via $\yin$ and $\yout$ satisfying the
conditions~\eqref{eqn:efficient-inner-outer-vecs} satisfies
Corollary~\ref{corollary:smallest-set-valued}. Indeed, we have the
following.
\begin{proposition}
  \label{proposition:inout-fine}
  Let $\yin$ and $\yout$ satisfy the
  conditions~\eqref{eqn:efficient-inner-outer-vecs}.  Then the confidence
  set $\Cinout(x)$ in Algorithm~\ref{alg:tree-cqc} is the smallest set
  containing $\Cimplicit(x)$ and admitting the
  form~\eqref{eqn:multilabel-io-set}.
\end{proposition}
\begin{proof}
  The conditions~\eqref{eqn:efficient-inner-outer-vecs} immediately
  imply
  \begin{equation*}
    \yin(x)_k = \min_{y \in \Cimplicit(x)} y_k
    ~~ \mbox{and} ~~
    \yout(x)_k = \max_{y \in \Cimplicit(x)} y_k,
  \end{equation*}
  which shows that $\Cimplicit(x) \subset \Cinout(x)$.  On the
  other hand, suppose that $\tilde{y}_\textup{in}(x)$ and
  $\tilde{y}_\textup{out}(x)$ are configurations inducing a confidence set
  $\tilde{C}(x)$ that satisfies $\Cimplicit(x) \subset
  \tilde{C}(x)$.  Then for any label $k$ included in
  $\tilde{y}_\textup{in}(x)$, all configurations $y \in
  \Cimplicit(x)$ satisfy $y_k=1$, as
  $\Cimplicit(x)\subset \tilde{C}(x)$, so we must have
  $\tilde{y}_\textup{in}(x)_k \le \yin(x)_k$.  The argument to prove
  $\tilde{y}_\text{out}(x)_k \ge \yout(x)_k$ is similar.
\end{proof}

The expansion from $\Cimplicit(x)$ to $\Cinout(x)$ may increase the size of
the confidence set, most notably in cases when labels repel each other.  As
a paradoxical worst-case, if $\Cimplicit(x)$ includes only each of the $K$
single-label configurations
then $\Cinout = \{-1 ,1\}^K$.  In such cases, refinement of the
inner/outer sets may be necessary; we outline an approach that considers
unions of such sets in Sec.~\ref{sec:union-inner-outer-sets} to come.  Yet
in problems for which we have strong predictors, we typically do not expect
``opposing'' configurations $y$ and $-y$ to both belong to $\Cinout(x)$,
which limits the increase in practice; moreover, in the case that the
standard implicit confidence set $\Cimplicit$ is a singleton, there is a
single $y \in \mc{Y}$ satisfying $s(x,y) \ge \what{q}_\alpha(x) -
\what{\quant}_{1-\alpha}$ and by definition $\yin(x) = \yout(x)$, so that
$\Cinout = \Cimplicit$.



\subsubsection{Unions of inner and outer sets}
\label{sec:union-inner-outer-sets}

As we note above, it can be beneficial to approximate the implicit
set~\eqref{eqn:standard-classification-set} more carefully; here, we
consider a union of easily representable sets.  The idea is that if two
tasks always have opposing labels, the confidence sets
should reflect that, yet it is possible that the naive
condition~\eqref{eqn:efficient-inner-outer-vecs} fails this check. For
example, consider a set $\Cimplicit(x)$ for which
any configuration $y \in \Cimplicit(x)$ satisfies $y_1 = -y_2$, but
which contains labels with both $y_1 = 1$ and $y_1 = -1$. In this case,
necessarily $\yin(x)_1 = \yin(x)_2 = -1$ and $\yout(x)_1 = \yout(x)_2 = 1$.
If instead we construct two sets of inner and outer vectors
$\yin^{(i)}, \yout^{(i)}$ for $i = 1, 2$, where
\begin{equation*}
  \yin^{(1)}(x)_1 = -\yin^{(1)}(x)_2 = 1
  ~~ \mbox{and} ~~
  \yin^{(2)}(x)_2 = -\yin^{(2)}(x)_2 = 1,
\end{equation*}
then choose the remaining labels $k = 3, \ldots, K$
so that $\yin^{(1)}(x) \preceq y \preceq \yin^{(1)}(x)$ for all
$y \in \Cimplicit(x)$ satisfying $y_1 = -y_2 = 1$, and vice-versa
for $\yin^{(2)}$ and $\yout^{(2)}$, we may evidently reduce the size of the
confidence set $\Cinout(x)$ by half while maintaining validity.

Extending this idea, let $I \subset [K]$ denote a set of indices. We
consider inner and outer sets $\yin$ and $\yout$ that index all
configurations of the labels $y_I = (y_i)_{i \in I} \in \{\pm 1\}^m$,
so that analogizing
the condition~\eqref{eqn:efficient-inner-outer-vecs}, we define the
$2^m$ inner and outer sets
\begin{subequations}
  \label{eqn:union-efficient-inner-outer}
  \begin{equation}
    \begin{split}
      \yin(x, y_I)_k & =
      \min\{y'_k \mid y' \in \Cimplicit(x), y'_I = y_I\} \\
      & = \begin{cases}
        -1 & \mbox{if~}
        \max_{y' : y'_k = -1, y'_I = y_I}
        s(x, y') \ge \what{q}_\alpha(x) - \what{\quant}_{1 - \alpha} \\
        1 & \mbox{if~}
        \max_{y' : y'_k = 1, y'_I = y_I}
        s(x, y') \ge \what{q}_\alpha(x) - \what{\quant}_{1 - \alpha}
        ~ \mbox{and~preceding~fails}\\
        + \infty & \mbox{otherwise},
      \end{cases}
    \end{split}
  \end{equation}
  and similarly
  \begin{equation}
    \yout(x, y_I)_k = \max\{y_k' \mid y' \in \Cimplicit(x), y'_I = y_I\}.
  \end{equation}
\end{subequations}
For any $I \subset [K]$ with $|I| = m$, we can then define the
index-based inner/outer confidence set
\begin{equation}
  \label{eqn:union-multilabel-io-set}
  \Cinout(x, I) \defeq
  \cup_{y_I \in \{\pm 1\}^m}
  \left\{y \in \mc{Y} \mid \yin(x, y_I) \preceq y \preceq \yout(x, y_I)
  \right\},
\end{equation}
which analogizes the function~\eqref{eqn:multilabel-io-set}.  When $m$ is
small, this union of rectangles is efficiently representable, and gives a
tighter approximation to $\Cimplicit$ than does the simpler
representation~\eqref{eqn:multilabel-io-set}; indeed, if for some pair $(i,
j) \in I$ we have $y_i = -y_j$ for all $y \in \Cimplicit(x)$, but for which
there are vectors $y \in \Cimplicit(x)$ realizing $y_i = 1$ and $y_i =
-1$, then $|\Cinout(x, I)| \le |\Cinout(x)| / 2$.  Moreover, no matter the
choice $I$ of the index set, we have the containment $\Cinout(x, I)
\supset \Cimplicit(x)$, so that
Corollary~\ref{corollary:smallest-set-valued} holds and $\Cinout$ provides
valid marginal coverage.  The sets~\eqref{eqn:union-efficient-inner-outer}
are efficiently computable for the scoring functions $s$ we consider
(cf.\ Sec.~\ref{sec:efficient-inner-outer-sets}).

The choice of the indices $I$ over which to split the rectangles requires
some care. A reasonable heuristic is to
obtain the inner/outer vectors $\yin(x)$ and $\yout(x)$ in
Alg.~\ref{alg:tree-cqc}, and if they provide too large a confidence set,
select a pair of variables $I = (i, j)$ for which $\yin(x)_{i,j} = (-1, -1)$
while $\yout(x)_{i, j} = (1, 1)$. To choose the pair, we suggest
the most negatively correlated pair of labels in the training data that
satisfy this joint inclusion; the heuristic is to find those pairings most
likely to yield empty confidence sets in the
collections~\eqref{eqn:union-efficient-inner-outer}.

\subsection{Efficient construction of inner
  and outer confidence sets}
\label{sec:efficient-inner-outer-sets}

With the validity of any inner/outer construction verifying the
conditions~\eqref{eqn:efficient-inner-outer-vecs} established, we turn to
two approaches to efficiently satisfy these.  The first focuses
on the scenario where a prediction method provides individual scoring
functions $s_k : \mc{X} \to \R$ for each task $k \in [K]$ (as frequent for
complex classifiers, such as random forests or deep
networks~\cite[e.g.][]{HerreraChRiDe16,ReadPfHoFr11}), while the second
considers the case when we have a scoring function $s : \mc{X} \times \mc{Y}
\to \R$ that is tree-structured in a sense we make precise; in the next
section, we will show how to construct such tree-structured scores using any
arbitrary multilabel prediction method.

\newcommand{\tinner}{t_{\textup{in}}}
\newcommand{\touter}{t_{\textup{out}}}
\newcommand{\tinnerhat}{\hat{t}_{\textup{in}}}
\newcommand{\touterhat}{\hat{t}_{\textup{out}}}

\paragraph{A direct inner/outer method using individual task scores}
We assume here that we observe individual scoring functions $s_k :
\mc{X} \to \R$ for each task. We can construct inner and outer
sets using only these scores while neglecting label correlations
by learning threshold functions
$\tinner \ge \touter : \mc{X} \to \R$, where we would like to have
the (true) labels $y_k$ satisfy
\begin{equation*}
  \sign(s_k(x) - \tinner(x)) \le y_k \le \sign(s_k(x) - \touter(x)).
\end{equation*}
In Algorithm~\ref{alg:direct-inner-outer}, we accomplish
this via quantile threshold functions on the maximal and minimal
values of the scores $s_k$ for positive and negative labels, respectively.

\algbox{\label{alg:direct-inner-outer} Split Conformalized
  Direct Inner/Outer method for Classification
  (CDioC).
}{
  \textbf{Input:} Sample $\{(X_i, Y_i)\}_{i = 1}^n$, disjoint
  index sets $\mc{I}_2, \mc{I}_3 \subset [n]$,
  quantile functions $\mc{Q}$, desired confidence level $\alpha$,
  and $K$ score functions $s_k : \mc{X} \to \R$.
  \begin{enumerate}[1.]
  \item Fit threshold functions (noting the sign conventions)
    \begin{equation}
      \begin{split}
        \tinnerhat & = -\argmin_{t \in \mc{Q}}
        \bigg\{\sum_{i \in \mc{I}_2} \pinball_{\alpha/2}
        \Big(\!-\max_k \{s_k(X_i) \mid Y_{i,k} = -1\} - t(X_i)\Big)
        \bigg\} \\
        \touterhat & =
        \argmin_{t \in \mc{Q}}
        \bigg\{\sum_{i \in \mc{I}_2} \pinball_{\alpha/2}
        \Big(\! \min_k \{s_k(X_i) \mid Y_{i,k} = 1\} - t(X_i)\Big)
        \bigg\}
      \end{split}
      \label{eqn:fit-inner-outer-thresholds}
    \end{equation}
  \item \label{item:direct-in-out-score}
    Define score
    $s(x, y) = \min\{ \min_{k : y_k = 1} s_k(x) - \touterhat(x),
    \tinnerhat(x) - \max_{k : y_k = -1} s_k(x)\}$ and
    compute conformity scores $\score_i \defeq -s(X_i, Y_i)$.
  \item Let $\what{\quant}_{1-\alpha}$ be the
    $(1 + 1 / |\mc{I}_3|)(1 - \alpha)$-empirical quantile of
    $\scores = \{\score_i\}_{i \in \mc{I}_3}$
    and
    \begin{equation*}
      \tinner(x) \defeq \tinnerhat(x) + \what{\quant}_{1 - \alpha}
      ~~ \mbox{and} ~~
      \touter(x) \defeq \touterhat(x) - \what{\quant}_{1-\alpha}.
    \end{equation*}
  \item Define $\yin(x)_k = \sign(s_k(x) - \tinner(x))$ and
    $\yout(x)_k = \sign(s_k(x) - \touter(x))$
    and return prediction set
    $\Cinout$ as in Eq.~\eqref{eqn:multilabel-io-set}.
  \end{enumerate}
}

The method is a slightly modified instantiation of the CQioC method
in Alg.~\ref{alg:tree-cqc} that allows easier computation. We
can also see that it guarantees validity.
\begin{corollary}
  Assume that $\{(X_i, Y_i)\}_{i = 1}^{n + 1} \sim P$ are exchangeable, where
  $P$ is an arbitrary distribution.  Then the
  confidence set $\Cinout$ in Algorithm~\ref{alg:direct-inner-outer}
  satisfies
  \begin{equation*}
    \P(Y_{n + 1} \in \Cinout(X_{n+1})) \ge 1 - \alpha.
  \end{equation*}  
\end{corollary}
\begin{proof}
  We show that the definitions of $\yin$ and $\yout$ in
  Alg.~\ref{alg:direct-inner-outer} are special cases of the
  condition~\eqref{eqn:efficient-inner-outer-vecs}, which then
  allows us to apply Corollary~\ref{corollary:smallest-set-valued}.
  We focus on the inner set, as the outer is similar.
  Suppose in Alg.~\ref{alg:direct-inner-outer}
  that $\yin(x)_k = 1$. Then
  $s_k(x) - \tinnerhat(x) - \what{\quant}_{1-\alpha} \ge 0$, which
  implies that
  $\tinnerhat(x) - s_k(x) \le -\what{\quant}_{1 - \alpha}$,
  and for any $y \in \mc{Y}$ satisfying $y_k = -1$, we have
  \begin{equation*}
    \tinnerhat(x) - \max_{l : y_l = -1} s_l(x) \le -\what{\quant}_{1-\alpha}.
  \end{equation*}
  For the scores $s(x, y)$ in
  line~\ref{item:direct-in-out-score} of Alg.~\ref{alg:direct-inner-outer},
  we then immediately obtain $s(x, y) \le - \what{\quant}_{1-\alpha}$ for
  any $y \in \mc{Y}$ with $y_k = -1$.  This is the first
  condition~\eqref{eqn:efficient-inner-outer-vecs}, so that (performing,
  \emph{mutatis-mutandis}, the same argument with $\yout(x)_k$)
  Corollary~\ref{corollary:smallest-set-valued} implies the validity of
  $\Cinout$ in Algorithm~\ref{alg:direct-inner-outer}.
\end{proof}

\paragraph{A prediction method for tree-structured scores}

We now turn to efficient computation of the inner and outer
vectors~\eqref{eqn:efficient-inner-outer-vecs} when the scoring function $s:
\mc{X} \times \mc{Y} \to \R$ is tree-structured. By this we mean that there
is a tree $\mc{T} = ([K], E)$, with nodes $[K]$ and edges $E \subset [K]^2$,
and an associated set of pairwise and singleton factors $\psi_e : \{-1,1\}^2
\times \mc{X} \to \R$ for $e \in E$ and $\varphi_k : \{-1, 1\} \times \mc{X}
\to \R$ for $k \in [K]$ such that
\begin{equation}
  \label{eqn:tree-structured-score}
  s(x, y) = \sum_{k = 1}^K \varphi_k(y_k, x)
  + \sum_{e = (k, l) \in E} \psi_e(y_k, y_l, x).
\end{equation}
Such a score allows us to consider interactions between tasks $k, l$ while
maintaining computational efficiency, and we will show how to construct such
a score function both from arbitrary multilabel prediction methods and from
those with individual scores as above.  When we have a tree-structured
score~\eqref{eqn:tree-structured-score}, we can use efficient
message-passing algorithms~\cite{Dawid92a, KollerFr09} to compute the
collections (maximum marginals) of scores
\begin{equation}
  \label{eqn:max-marginals}
  \mc{S}_- \defeq \Big\{\max_{y \in \mc{Y} : y_k = -1} s(x,y)\Big\}_{k = 1}^K
  ~~ \mbox{and} ~~
  \mc{S}_{+} \defeq \Big\{\max_{y \in \mc{Y} : y_k = 1} s(x,y)\Big\}_{k = 1}^K
\end{equation}
in time $O(K)$, from which it is immediate to construct
$\yin$ and
$\yout$ as in the conditions~\eqref{eqn:efficient-inner-outer-vecs}.
We outline the approach in
Appendix~\ref{sec:computation-of-min-max-tree}, as it is not the central
theme of this paper, though this efficiency highlights the importance of the
tree-structured scores for practicality.

\subsection{Building tree-structured scores}
\label{sec:multilabel-arbitrary}

With the descriptions of the generic multilabel conformalization method in
Alg.~\ref{alg:tree-cqc} and that we can efficiently compute predictions
using tree-structured scoring functions~\eqref{eqn:tree-structured-score},
we now turn to constructing such scoring functions from predictors, which
trade between label dependency structure, computational efficiency, and
accuracy of the predictive function. We begin with a general case of an
arbitrary predictor function, then describe a heuristic graphical model
construction when individual label scores are available (as we assume in
Alg.~\ref{alg:direct-inner-outer}).
 
\paragraph{From arbitrary predictions to scores}
We begin with the most general case that we have access only to a predictive
function $\what{y}: \mathcal{X} \to \R^K$. This prediction function is
typically the output of some learning algorithm, and in the generality here,
may either output real-valued scores $\what{y}_k(x) \in \R$ or simply output
$\what{y}_k(x) \in \{-1, 1\}$, indicating element $k$'s presence.

We compute a regularized scoring function based on a
tree-structured graphical model (cf.~\cite{KollerFr09}) as follows.
Given a tree $\mathcal{T}=([K], E)$ on the labels $[K]$ and parameters
$\alpha \in \R^K$, $\beta \in \R^E$, we define
\begin{align}
  \label{eqn:tree-score-function}
  s_{\mc{T},\alpha, \beta}(x,y) \defeq \sum_{k=1}^K \alpha_k y_k \what{y}_k(x)
  + \sum_{e=(k,l) \in E} \beta_e y_k y_l
\end{align}
for all $(x,y) \in \mathcal{X} \times \{-1, 1\}^K$, where we recall that we
allow $\what{y}_k(x) \in \R$. We will find the tree $\mc{T}$ assigning
the highest (regularized) scores to the true data $(X_i, Y_i)_{i=1}^n$
using efficient dynamic programs
reminiscent of the Chow-Liu algorithm~\cite{ChowLi68}.
To that end, we use Algorthim~\ref{alg:score-tree}.

\algbox{
  \label{alg:score-tree} Method to find optimal tree from
  arbitrary predictor $\what{y}$.
}{
  \textbf{Input:} Sample $\{(X_i, Y_i)\}_{i \in \mc{I}_1}$,
  regularizers $r_1, r_2 : \R \to \R$,
  predictor
  $\what{y} : \mc{X} \to \R^K$.

  Set
  \begin{equation}
    \label{eqn:tree-score-function-mle}
    (\what{\mc{T}}, \what{\alpha}, \what{\beta})
    \defeq \argmax_{\mc{T} = ([K], E), \alpha, \beta}
    \bigg\{\sum_{i = 1}^n s_{\mc{T}, \alpha, \beta}(X_i, Y_i) -
    \sum_{k = 1}^K r_1(\alpha_k)
    - \sum_{e \in E} r_2(\beta_e)\bigg\}.
  \end{equation}
  and return score function $s_{\what{\mc{T}}, \what{\alpha},
    \what{\beta}}$ of form~\eqref{eqn:tree-score-function}.
}

\noindent
Because the regularizers $r_1, r_2$ decompose along the edges and nodes of
the tree, we can implement Alg.~\ref{alg:score-tree} using a maximum
spanning tree algorithm. Indeed, recall~\cite{BoydVa04}
the familiar convex conjugate
$r^*(t) \defeq \sup_\alpha \{\alpha t - r(\alpha)\}$. Then
immediately
\begin{align*}
  \lefteqn{\sup_{\alpha, \beta} \bigg\{\sum_{i = 1}^n s_{\mc{T}, \alpha, \beta}(X_i, Y_i)
    - \sum_{k = 1}^K r_1(\alpha_1) - \sum_{e \in E} r_2(\beta_e)\bigg\}} \\
  & \qquad\qquad\qquad\qquad
  = \sum_{k = 1}^K r_1^*\left(\sum_{i = 1}^n Y_{i,k} \what{y}_k(X_i)\right)
  + \sum_{e = (k, l)\in E} r_2^*\left(\sum_{i = 1}^n Y_{i,k} Y_{i,l}\right),
\end{align*}
which decomposes along the edges of the putative tree. As a consequence, we
may solve problem~\eqref{eqn:tree-score-function-mle} by finding the maximum
weight spanning tree in a graph with edge weights $r_2^*(\sum_{i = 1}^n
Y_{i,k} Y_{i, l})$ for each edge $(k, l)$, then choosing $\alpha, \beta$ to
maximize the objective~\eqref{eqn:tree-score-function-mle}, which is a
collection of 1-dimensional convex optimization problems.

\paragraph{From single-task scores to a tree-based probabilistic model}
While Algorithm~\ref{alg:score-tree} will work regardless of the predictor
it is given---which may simply output a vector $\what{y} \in \{-1, 1\}^K$,
as in Alg.~\ref{alg:direct-inner-outer} it is frequently the case that
multilabel methods output scores $s_k : \mc{X} \to \R$ for each task.  To
that end, a natural strategy is to model the distribution of $Y \mid X$
directly via a tree-structured graphical model~\cite{LaffertyMcPer01}.
Similar to the score in Eq.~\eqref{eqn:tree-structured-score}, we define
interaction factors $\psi: \{-1,1\}^2 \rightarrow \R^4$ by $\psi(-1,-1) =
e_1$, $\psi(1,-1) = e_2$, $\psi(-1,1) = e_3$ and $\psi(1,1) = e_4$, the
standard basis vectors, and marginal factors
$\varphi_k : \{-1, 1\} \times \mc{X} \to \R^2$ with
\begin{equation*}
  \varphi_k(y_k, x)
  \defeq \half \left[\begin{matrix} 
      (y_k - 1) \cdot s_k(x) \\
      (y_k + 1) \cdot s_k(x) \end{matrix} \right],
\end{equation*}
incorporating information $s_k(x)$ provides on $y_k$.
For a tree $\mc{T} = ([K], E)$, the label model is
\begin{equation}
  \label{multilabel-pgmmethod-treelikelihood}
  p_{\mathcal{T}, \alpha, \beta} \left(y \mid x \right) 
  \propto 
  \exp \biggl(
  \sum_{e = (k,l) \in E} \beta_e^T \psi(y_k, y_l)
  + \sum_{k=1}^K \alpha_k^T \varphi_k(y_k,x)
  \biggr),
\end{equation}
where $(\alpha, \beta)$ is a set of parameters such
that, for each edge $e \in E$, $\beta_e \in \R^4$ and $\ones^T \beta_e = 0$
(for identifiability), while $\alpha_k \in \R^2$ for each
label $k \in [K]$.
%
Because we view this as a ``bolt-on'' approach, applicable to \emph{any}
method providing scores $s_k$, we include only pairwise label interaction
factors independent of $x$, allowing singleton factors to depend on the
observed feature vector $x$ through the scores $s_k$.

The log-likelihood $\log p_{\mc{T}, \alpha, \beta}$ is convex
in $(\alpha, \beta)$ for any fixed tree $\mc{T}$, and the
Chow-Liu decomposition~\cite{ChowLi68} of the likelihood of a tree
$\mc{T} = ([K], E)$ gives
\begin{equation}
  \label{eqn:tree-log-likelihood}
  \log p_{\mc{T}, \alpha,\beta}(y \mid x)
  = \sum_{k = 1}^K \log p_{\mc{T}, \alpha, \beta}(y_k \mid x)
  + \sum_{e = (k, l) \in E}
  \log \frac{p_{\mc{T}, \alpha, \beta}(y_k, y_l \mid x)}{
    p_{\mc{T}, \alpha, \beta}(y_k \mid x)
    p_{\mc{T}, \alpha, \beta}(y_l \mid x)},
\end{equation}
that is, the sum of the marginal log-likelihoods and pairwise mutual
information terms,
conditional on $X = x$.
Given a sample $(X_i, Y_i)_{i = 1}^n$, the goal is to then
solve
\begin{align}
  \label{eqn:multilabel-pgmmethod-treeMLE}
  \maximize_{\mathcal{T}, \alpha, \beta}
  L_n(\mathcal{T}, \alpha, \beta)  \defeq
  \sum_{i=1}^n \log p_{\mathcal{T}, \alpha,\beta} \left(Y_i \mid X_i \right).
\end{align}
When there is no conditioning on $x$, the pairwise mutual
information terms $\log \frac{p_{\mc{T},\alpha,\beta}(y_k, y_l)}{p(y_k)
  p(y_l)}$ are independent of the tree $\mc{T}$~\cite{ChowLi68}.
We heuristically compute empirical conditional mutual informations between
each pair $(k, l)$ of tasks,
choosing the tree $\mc{T}$ that maximizes these values to approximate
problem~\eqref{eqn:multilabel-pgmmethod-treeMLE} in
Algorithm~\ref{alg:multilabel-pgm-method}, using the selected tree
$\what{\mc{T}}$
to choose $\alpha, \beta$ maximizing $L_n(\what{\mc{T}}, \alpha, \beta)$.
(In the algorithm we superscript $Y$ to make task labels versus
observations clearer.)


\algbox{
  \label{alg:multilabel-pgm-method}
  Chow-Liu-type approximate Maximum Likelihood Tree and Scoring Function
}{
  \textbf{Input:}
  Sample $\{(X^{(i)}, Y^{(i)})\}_{i \in \mc{I}_1}$, and $K$ score functions $s_k : \mc{X} \to \R$.

  \textbf{For} each pair $e = (k,l) \in [K]^2$
  \begin{enumerate}
  \item Define the single-edge tree $\mc{T}_e = (\{k, l\}, \{e\})$
  \item \label{item:fit-tree-likelihood}
    Fit model~\eqref{eqn:tree-log-likelihood}
    for tree $\mc{T}_e$ via
    $(\what{\alpha}, \what{\beta}) \defeq
    \argmax_{\alpha,\beta} L_n(\mc{T}_e, \alpha, \beta)$
  \item Estimate edge empirical mutual information
    \begin{equation*}
      \what{I}_{e} \defeq \sum_{i=1}^n
      \log\bigg(
      \frac{
        p_{\mathcal{T}_e, \what{\alpha}, \what{\beta}}(Y_k^{(i)}, Y_l^{(i)} \mid X^{(i)})
      }{
        p_{\mathcal{T}_e, \what{\alpha}, \what{\beta}}(Y_k^{(i)} \mid X^{(i)})
        p_{\mathcal{T}_e, \what{\alpha}, \what{\beta}}(Y_l^{(i)} \mid X^{(i)})
      }
      \bigg)
    \end{equation*}
  \end{enumerate}

  \textbf{Set} $\what{\mathcal{T}} =
  \textsc{MaxSpanningTree}((\what{I}_e)_{e \in [K]^2})$ and $(\what{\alpha},
  \what{\beta}) = \argmax_{\alpha, \beta} L_n(\what{\mc{T}}, \alpha,
  \beta)$.

  \textbf{Return} scoring function
  \begin{equation*}
    s_{\what{\mc{T}}, \what{\alpha}, \what{\beta}}(x, y)
    \defeq 
    \sum_{e = (k, l) \in E}
    \what{\beta}_e^T \psi(y_k, y_l) + \sum_{k = 1}^K \what{\alpha}_k^T
    \varphi_k(y_k, x)
  \end{equation*}
}

\noindent
The algorithm takes time roughly $O(nK^2 + K^2 \log(K))$, as each
optimization step~\ref{item:fit-tree-likelihood}
solves an $O(1)$-dimensional concave maximization problem, which is
straightforward via a Newton method (or gradient descent).  The
approach does not guarantee recovery of the correct tree structure even if
the model is well-specified, as we neglect information coming
from labels other than $k$ and $l$ in the estimates $\what{I}_e$ for edge $e
= (k, l)$, although we expect the heuristic to return sufficiently
reasonable tree structures. In any case, the actual scoring function $s$ it
returns still allows efficient conformalization and valid predictions via
Alg.~\ref{alg:tree-cqc}, regardless of its accuracy; a more accurate tree
will simply allow smaller and more accurate confidence
sets~\eqref{eqn:efficient-inner-outer-vecs}.

\newcommand{\scorelesstree}{ArbiTree-CQC\xspace}
\newcommand{\scorelesstreeshort}{ArbiTree\xspace}
\newcommand{\pgmtree}{PGM-CQC\xspace}
\newcommand{\pgmtreeshort}{PGM\xspace}

\section{Experiments}
\label{sec:experiments}

\newcommand{\VC}{\textup{VC}}

Our main motivation is to design methods with more robust conditional
coverage than the ``marginal'' split-conformal
method~\eqref{eqn:marginal-method}. Accordingly, the
methods we propose in Sections~\ref{sec:multiclass} and~\ref{sec:multilabel}
fit conformalization scores that depend on 
features $x$ and, in some cases, model dependencies among $y$ variables.
Our experiments consequently focus on more robust notions of coverage
than the nominal marginal coverage the methods guarantee, and we
develop a new evaluation metric for validation and testing of coverage, looking
at connected subsets of the space $\mc{X}$ and studying coverage over these.
Broadly, we expect our proposed methods to maintain coverage of (nearly)
$1 - \alpha$ across subsets; the experiments
are consistent with this expectation. We include a few additional
plots in supplementary appendices for completeness.

\paragraph{Measures beyond marginal coverage} 
Except in simulated experiments, we cannot compute conditional coverage of
each instance, necessitating approximations that provide more
conditional-like measures of coverage---where methods providing weaker
marginal coverage may fail to uniformly cover---while still allowing
efficient computation.  To that end, we consider at coverage over slabs
\begin{equation*}
  S_{v,a,b} \defeq \left\{ x \in \R^d \mid a \leq v^T
  x \leq b \right\},
\end{equation*}
where $v \in \R^d$ and $a < b \in \R$, which satisfy these desiderata.  For
a direction $v$ and threshold $0 < \delta \le 1$, we consider the worst
coverage over all slabs containing $\delta$ mass in $\{X_i\}_{i=1}^n$,
defining
\begin{align}
  \label{eqn:worst-slab-computation}
  \text{WSC}_n(\what{C},v) \defeq
  \inf_{a < b}
  \left\{P_n(Y \in \what{C}(X) \mid a \le v^T X \le b)
  ~~~ \mbox{s.t.}~~
  P_n(a \le v^T X \le b) \ge \delta\right\},
\end{align}
where $P_n$ denotes the empirical distribution on $(X_i, Y_i)_{i=1}^n$,
which is efficiently computable in $O(n)$ time~\cite{ChungLu03}.  As long as
the mapping $\what{C} : \mc{X} \rightrightarrows \mc{Y}$ is constructed
independently of $P_n$, we can show that these quantities
concentrate. Indeed, let us temporarily assume that the confidence set has
the form $\what{C}(x) = \{y \mid s(x, y) \ge q(x)\}$ for an arbitrary
scoring function $s : \mc{X} \times \mc{Y} \to \R$ and threshold functions
$q$. Let $V \subset \R^d$; we abuse notation to let $\VC(V)$ be the
VC-dimension of the set of halfspaces it induces, where we note that $\VC(V)
\le \min\{d, \log_2 |V|\}$. Then for some numerical constant $C$, for all $t
> 0$
\begin{align}
  \sup_{v \in V, a \le b : P_n(X \in S_{v,a,b}) \ge \delta} &
  \left\{|P_n(Y \in \what{C}(X) \mid X \in S_{v,a,b})
  - P(Y \in \what{C}(X) \mid X \in S_{v, a, b})|
  \right\} \label{eqn:conditional-probs-close} \\
  & \qquad \qquad \le C \sqrt{\frac{\VC(V) \log n + t}{\delta n}}
  \le C \sqrt{\frac{\min\{d, \log |V|\} \log n + t}{\delta n}}
  \nonumber
\end{align}
with probability at least $1 - e^{-t}$.
(See Appendix~\ref{sec:localization-vc-argument} for a brief
derivation of inequality~\eqref{eqn:conditional-probs-close} and a few
other related inequalities.) 

Each of the confidence sets we develop in this paper satisfy $\what{C}(x)
\supset \{y \mid s(x, y) \ge q(x)\}$ for some scoring function $s$ and
function $q$. Thus, if $\what{C}: \mc{X} \rightrightarrows \mc{Y}$
effectively provides conditional coverage at level $1 - \alpha$, we should
observe that
\begin{equation*}
  \inf_{v \in V} \text{WSC}_n(\what{C},v)
  \ge 1 - \alpha - O(1) \sqrt{\frac{\VC(V) \log n}{\delta n}}
  \ge 1 - \alpha - O(1) \sqrt{\frac{\min\{d, \log |V|\} \log n}{\delta n}}.
\end{equation*}


In each of our coming experiments, we draw $M = 1000$ samples $v_j$
uniformly on $\sphere^{d-1}$, computing the worst-slab
coverage~\eqref{eqn:worst-slab-computation} for each $v_j$.  In the
multiclass case, we expect our conformal quantile classification (CQC,
Alg.~\ref{alg:cqc}) method to provide larger worst-slab coverage than the
standard marginal method~\eqref{eqn:marginal-method}, while in the
multilabel case, we expect that the combination of tree-based scores
(Algorithms~\ref{alg:score-tree} or \ref{alg:multilabel-pgm-method}) with
the conformalized quantile inner/outer classification (CQioC,
Alg.~\ref{alg:tree-cqc}) should provide larger worst-slab coverage than the
conformalized direct inner/outer classification (CDioC,
Alg.~\ref{alg:direct-inner-outer}).  In both cases, we expect that our more
sophisticated methods should provide confidence sets of comparable size to
the marginal methods.
%
%
In multiclass experiments, for comparison, we additionally include the
Generalized Inverse Quantile method (GIQ, Algorithm 1~\cite{RomanoSeCa20},
which appeared after the initial version of the current paper appeared on
the \texttt{arXiv}), which similarly targets improved conditional coverage.
Unlike in the multiclass case, we know of no baseline method for multilabel
problems, as the ``marginal'' method~\eqref{eqn:marginal-method} is
computationally inefficient when the number of labels grows.  For this
reason, we focus on the methods in this paper, highlighting the potential
advantages of each one while comparing them to an oracle (conditionally
perfect) confidence set in simulation.

\begin{figure}
  \centering
  \begin{tabular}{cc}
    \hspace{-.8cm}
    \begin{minipage}{.42\columnwidth}
      \caption{\label{fig:multiclass-simulation-scatter-dataset}
        Gaussian mixture with $\mu_0 = (1,0)$,
        $\mu_1=(-\half, \frac{\sqrt{3}}{2})$, $\mu_2 = (-\half, -
        \frac{\sqrt{3}}{2})$, and $\mu_3 = (-\half, 0)$; and $\Sigma_0 =
        \diag(0.3,1)$, $\Sigma_1 = \Sigma_2 = \diag(0.2,0.4)$, and $\Sigma_3 =
        \diag(0.2,0.5)$.}
    \end{minipage}
    &
    \hspace{-.75cm}
    \begin{minipage}{.65\columnwidth}
      \includegraphics[width=\columnwidth]{%
        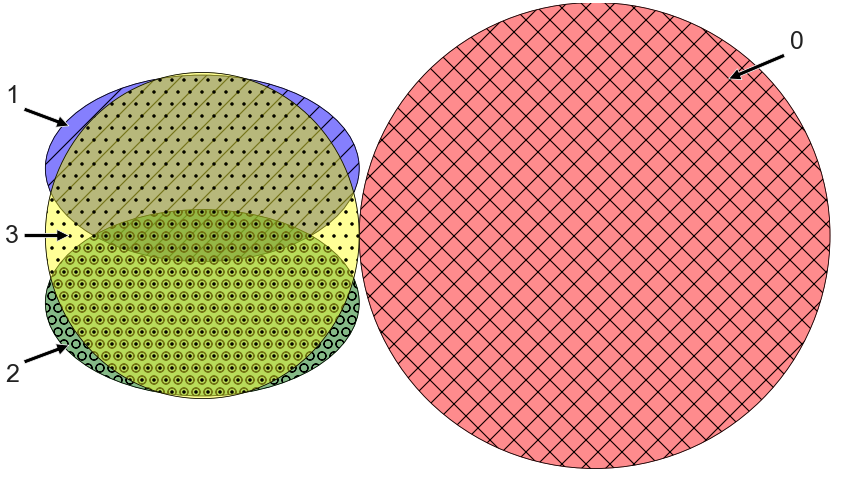}
    \end{minipage}
  \end{tabular}
  \vspace{-.2cm}
\end{figure}%

We present five experiments: two with simulated datasets and three with real
datasets (CIFAR-10~\cite{KrizhevskyHi09}, ImageNet~\cite{DengDoSoLiLiFe09} and Pascal VOC 2012~\cite{EveringhamVaWiWiZi12}), and both in multiclass and multilabel
settings.  In each experiment, a single trial corresponds to a realized random
split of the data between training, validation and calibration sets
$\mathcal{I}_1$, $\mathcal{I}_2$ and $\mathcal{I}_3$, and in each figure,
the red dotted line represents the desired level of coverage
$1-\alpha$. Unless otherwise specified, we summarize results via
boxplots that display the lower and upper quartiles as the hinges of the
box, the median as a bold line, and whiskers that extend to the $5\%$ and
$95\%$ quantiles of the statistic of interest (typically coverage
or average confidence set size).

\subsection{Simulation}
\label{sec:experiments-simulation}

\subsubsection{More uniform coverage on a multiclass example}

Our first simulation experiment allows us to compute the
conditional coverage of each sample and evaluate our CQC
method~\ref{alg:cqc}. We study its performance
on small sub-populations, a challenge for traditional machine learning
models~\cite{DuchiNa18, HashimotoSrNaLi18}.  In contrast to the traditional
split-conformal algorithm (method~\eqref{eqn:marginal-method},
cf.~\cite{VovkGaSh05}), we expect the quantile
estimator~\eqref{eqn:quantile-regression} in the CQC method~\ref{alg:cqc} to
better account for data heterogeneity, maintaining higher coverage on
subsets of the data, in particular in regions of the space where multiple classes coexist.

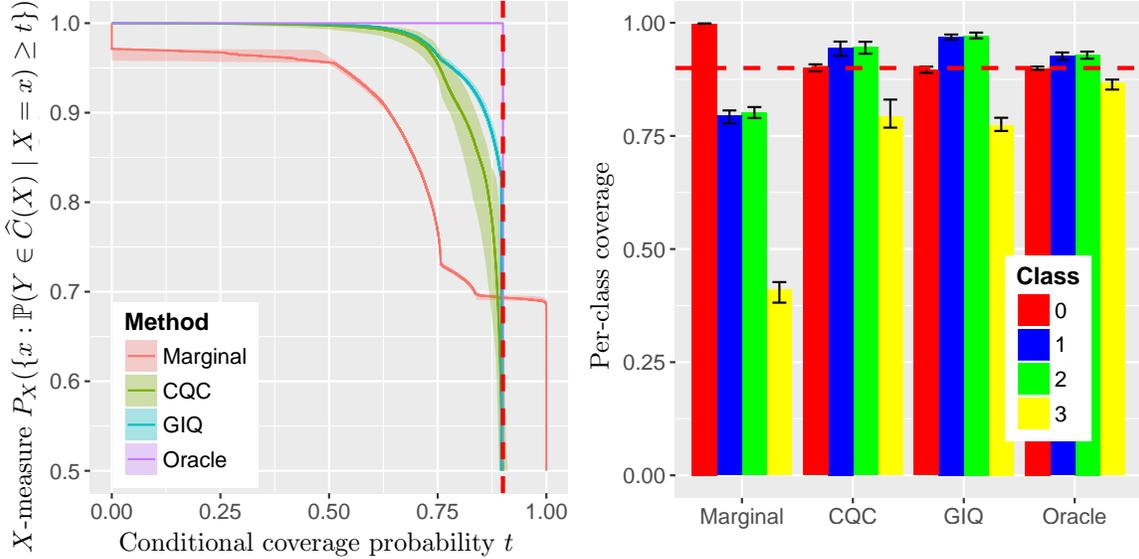
\begin{figure}
  \centering
  \begin{overpic}[scale=0.75]{
      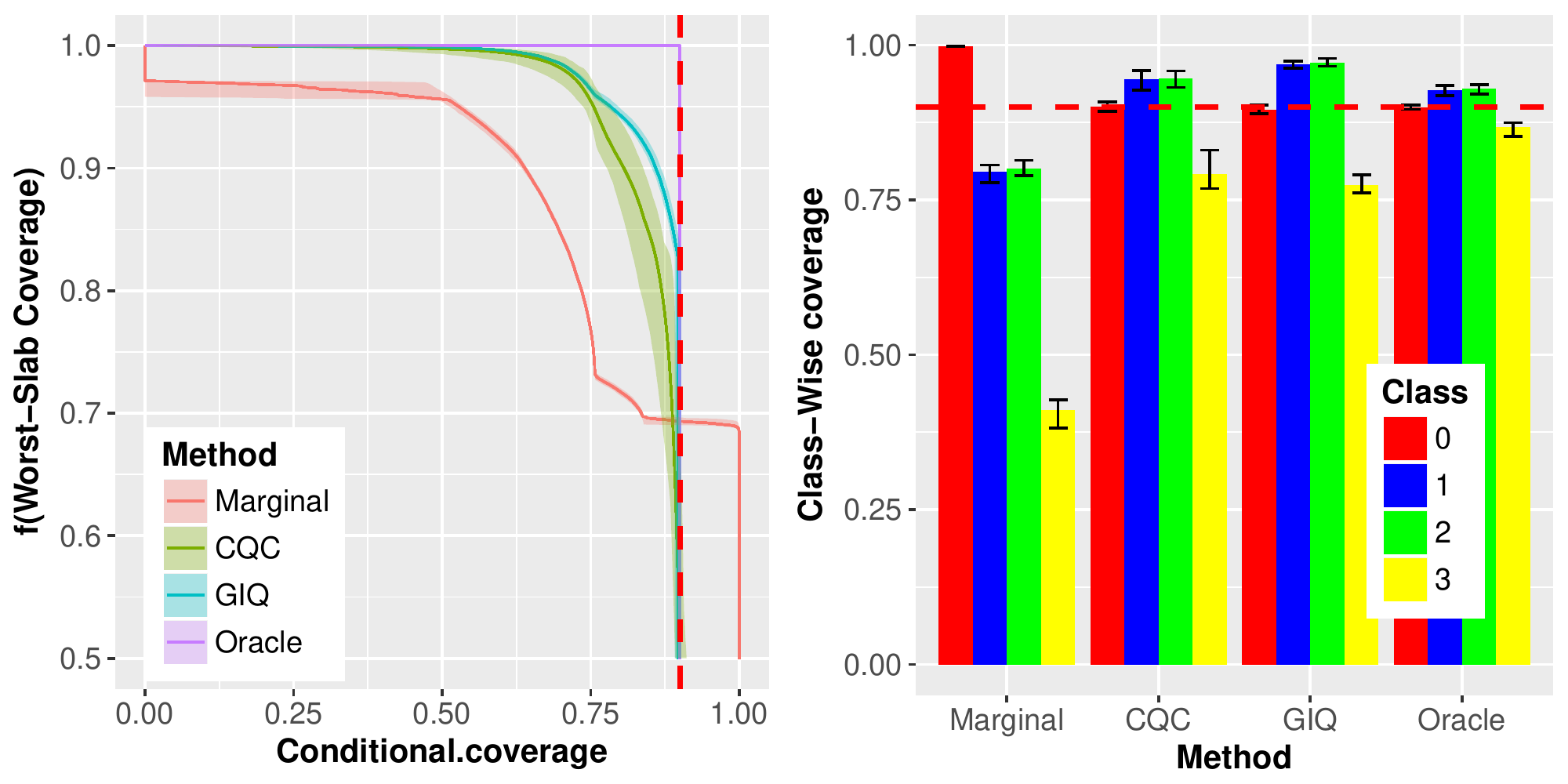}
    \put(0,10){
      \tikz{\path[draw=white, fill=white] (0, 0) rectangle (.3cm, 6cm)}
    }
    \put(0,0){\rotatebox{90}{
        \small $X$-measure $P_X(\{x :
        \P(Y \in \what{C}(X) \mid X = x) \ge t\})$}
    }
    \put(15, 0){
      \tikz{\path[draw=white, fill=white] (0, 0) rectangle (4cm, .5cm)}
    }
    \put(9, 1){
      \small Conditional coverage probability $t$}
    \put(70, 0){
      \tikz{\path[draw=white, fill=white] (0, 0) rectangle (4cm, .5cm)}
    }
    \put(49, 13){
      \tikz{\path[draw=white, fill=white] (0, 0) rectangle (.5cm, 5cm)}
    }
    \put(51, 16){\rotatebox{90}{
        \small Per-class coverage}}
  \end{overpic}
  \caption{Simulation results on multiclass problem. Left: $X$-probability
    of achieving a given level $t$ of conditional coverage versus coverage
    $t$, i.e., $t \mapsto P_X(\P(Y\in \what{C}_{\text{Method}}(X) \mid X)
    \geq t)$. The ideal is to observe $t \mapsto \indic{t \leq
      1-\alpha}$. Right: class-wise coverage $\P(Y\in
    \what{C}_{\text{Method}}(X) \mid Y=y)$ on the
    distribution~\eqref{eqn:dataset-multiclass-simulated} (as in
    Fig.~\ref{fig:multiclass-simulation-scatter-dataset}) for each
    method. Confidence bands and error bars display
    the range of the statistic over $M=20$ trials.}
\label{fig:multiclass-simulation-coverage}
\end{figure}

To test this, we generate $n=10^5$ data points $\{ X_i, Y_i \}_{i \in [N]}$
i.i.d.\ from a Gaussian mixture with one majority
group and three minority ones,
\begin{align}
  \label{eqn:dataset-multiclass-simulated}
  Y \sim \text{Mult}(\pi) \text{ and } X \mid Y=y \sim \normal( \mu_y, \Sigma_y ).
\end{align}
where $\pi = (.7, .1, .1, .1)$ (see
\figref{fig:multiclass-simulation-scatter-dataset}).  We purposely introduce
more confusion for the three minority groups (1, 2 and 3), whereas the
majority (0) has a clear linear separation.  We choose $\alpha= 10\%$, the
same as the size of the smaller sub-populations, then we apply our CQC
method and compare it to both the Marginal~\eqref{eqn:marginal-method} and
the Oracle (which outputs the smallest randomized conditionally valid $1-\alpha$
confidence set) methods.

The results in Figure~\ref{fig:multiclass-simulation-coverage} are
consistent with our expectations. The randomized oracle method provides
exact $1-\alpha$ conditional coverage, but CQC appears to provide more
robust coverage for individual classes than the marginal method.  While all
methods have comparable confidence set sizes and maintain $1-\alpha$
coverage marginally (see also supplementary
\figref{fig:multiclass-averagesize-coverage}), the left plot shows the CQC
and GIQ methods provide consistently better conditional coverage than the
marginal method. The latter~\eqref{eqn:marginal-method} has a coverage close
to $1$ for 70\% of the examples (typically all examples from the majority
class) and so undercovers the remaining 30\% minority examples, in
distinction from the CQC and GIQ methods, whose aims for conditional
coverage yield better coverage for minority classes (see right plot of
\figref{fig:multiclass-simulation-coverage}).

\subsubsection{Improved coverage with graphical models}

\label{subsec:experiments-multilabel-simulation}

\begin{figure}
 \centering
  \begin{overpic}[
  				scale=0.8]{%
     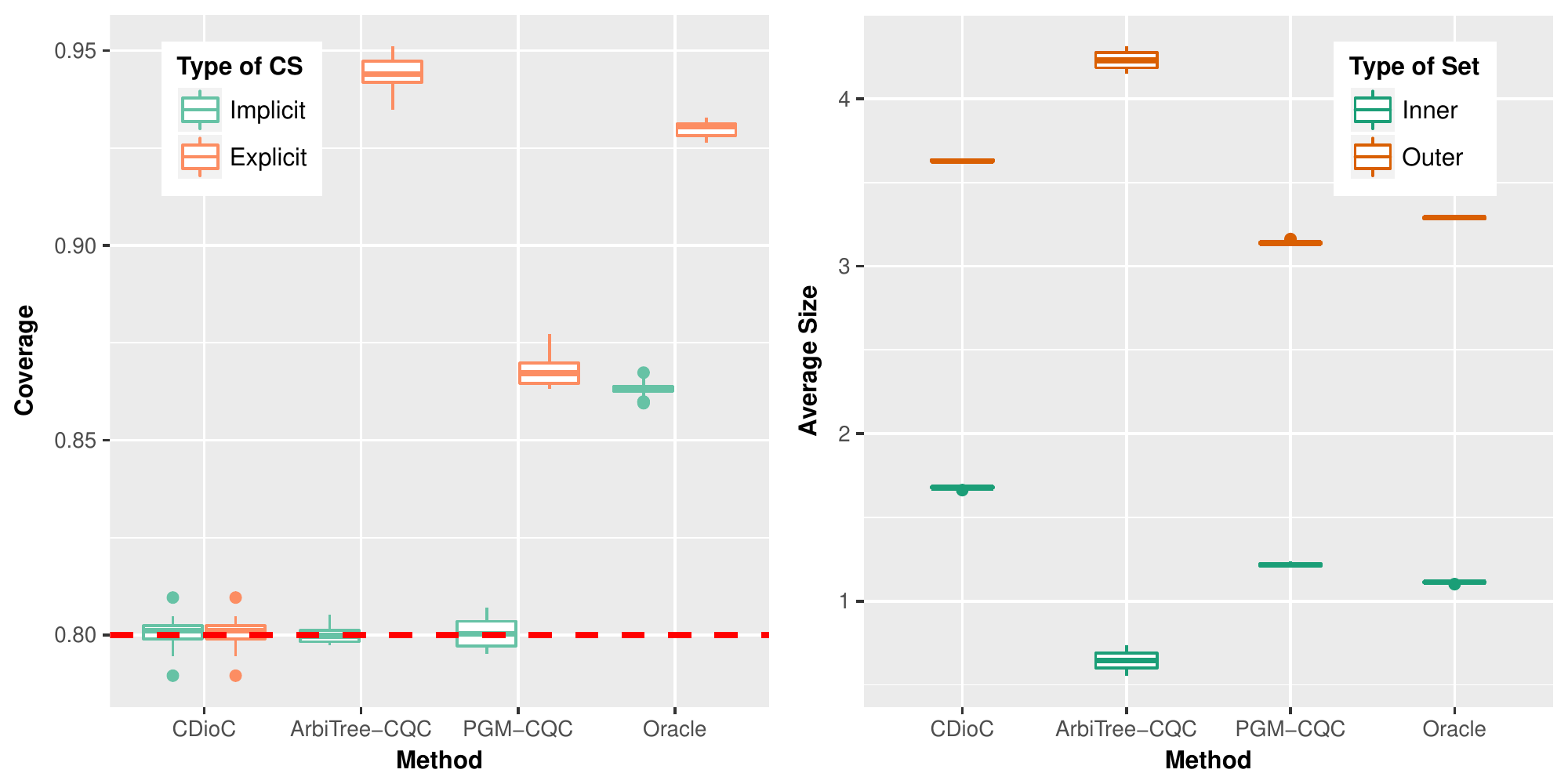}
    \put(0,10){
      \tikz{\path[draw=white, fill=white] (0, 0) rectangle (.3cm, 6cm)}
    }
    \put(0,15){\rotatebox{90}{
        \small $\P(Y \in \hat{C}(X))$}
    }
    \put(15, 0){
      \tikz{\path[draw=white, fill=white] (0, 0) rectangle (4cm, .4cm)}
    }
        
    \put(70, 0){
      \tikz{\path[draw=white, fill=white] (0, 0) rectangle (4cm, .4cm)}
    }
    \put(48, 13){
      \tikz{\path[draw=white, fill=white] (0, 0) rectangle (.6cm, 5cm)}
    }
    \put(51, 10){\rotatebox{90}{
        \small $\E\left|\yin(X)\right|$ and $\E\left|\yout(X)\right|$}}
  \end{overpic}
  \caption{Results for simulated multilabel experiment with label
    distribution~\eqref{eqn:multilabel-logreg}.  Methods are the true oracle
    confidence set; the conformalized direct inner/outer method (CDioC,
    Alg.~\ref{alg:direct-inner-outer}) and tree-based methods with implicit
    confidence sets $\Cimplicit$ or explicit inner/outer sets $\Cinout$,
    labeled \scorelesstreeshort and \pgmtreeshort (see
    Sec.~\ref{subsec:experiments-multilabel-simulation}).
    Left: Marginal coverage probability at level $1 - \alpha = .8$
    for different methods.
    Right: confidence set sizes of methods.
  }
  \label{fig:multilabel-simulation-averagesize-coverage}
\end{figure}

Our second simulation addresses the multilabel setting, where we have a
predictive model outputting scores $s_k(x) \in \R$ for each task $k$.  As a
baseline comparison, we compute oracle confidence sets (the smallest
$1-\alpha$-conditionally valid (non-randomized) confidence set in the
implicit case and the smaller inner and outer sets containing it in the
explicit case).  We run three methods, First, the direct Inner/Outer method
(CDioC), Alg.~\ref{alg:direct-inner-outer}. Second, we use the graphical
model score from the likelihood model in
Alg.~\ref{alg:multilabel-pgm-method} to choose a scoring function
$s_{\mc{T}} : \mc{X} \times \mc{Y} \to \R$, which we call the \pgmtree
method; we then either use the CQC method~\ref{alg:cqc} with this scoring
function directly, that is, the implicit confidence set
(recall~\eqref{eqn:standard-classification-set}) $\Cimplicit(x) = \{y \in
\mc{Y} \mid s_{\mc{T}}(x, y) \ge \hat{q}(x) - \what{\quant}_{1 - \alpha}\}$
or the explicit $\Cinout$ set of
Eqs.~\eqref{eqn:multilabel-io-set}--\eqref{eqn:efficient-inner-outer-vecs}.
Finally, we do the same except that we use the arbitrary predictor method
(Alg.~\ref{alg:score-tree}) to construct the score $s_{\mc{T}}$, where we
use the $\{\pm 1\}^K$ assignment $\what{y}$ instead of scores as input
predictors, which we term \scorelesstree.

We consider a misspecified logistic regression model, where hidden
confounders induce correlation between labels.  Because of the dependency
structure, we expect the tree-based methods to output smaller and more
robust confidence sets than the direct CDioC
method~\ref{alg:direct-inner-outer}. Comparing the two score-based methods
(CDioC and \pgmtree) with the scoreless tree method \scorelesstree is
perhaps unfair, as the latter uses less information---only signs of
predicted labels. Yet we expect the \scorelesstree method to leverage the
correlation between labels and mitigate this disadvantage, at least relative
to CDioC.

Our setting follows, where we consider $K=5$ tasks and dimension $d = 2$.
For each experiment, we choose a tree $\mc{T} = ([K], E)$ uniformly at
random, where each edge $e\in E$ has a correlation strength value $\tau_e
\sim \uniform[-5,5]$, and for every task $k$ we sample a vector $\theta_k
\sim \uniform(r \sphere^{d-1})$ with radius $r=5$.  For each observation, we
draw $X \simiid \normal (0,I_d)$ and draw independent uniform hidden
variables $H_e \in \{ -1, 1\}$ for each edge $e$ of the tree.  Letting $E_k$
be the set of edges adjacent to $k$, we draw $Y_k \in \{-1, 1\}$ from the
logistic model
\ifthenelse{\equal{\paperVersion}{icml}}{
  \begin{align}
    \label{eqn:multilabel-logreg}
    \begin{split}
      \P & \left( Y_k = y_k \mid X=x, H=h \right) \propto \\
      &\qquad \exp \biggl\{ -y_k \Bigl(1 + x^T \theta_k + \sum_{e \in E_k} \tau_e h_e \Bigr) \biggr\}.
    \end{split}
  \end{align}
}{ 
  \begin{align}
    \label{eqn:multilabel-logreg}
    \P  \left( Y_k = y_k \mid X=x, H=h \right) \propto \exp \biggl\{ -y_k \Bigl(1 + x^T \theta_k + \sum_{e \in E_k} \tau_e h_e \Bigr) \biggr\}.
  \end{align}
} We simulate $n_\text{total} = $ 200,000 data points, using $n_\text{tr}$
= 100,000 for training, $n_\text{v}= $ 40,000 for validation, $n_\text{c}= $
40,000 for calibration, and $n_\text{te}= $ 20,000 for testing.

The methods CDioC and \pgmtree require per-task scores $s_k$, so for each
method we fit $K$ separate logistic regressions of $Y_k$ against $X$
(leaving the hidden variables $H$ unobserved) on the training data. We then
use the fit parameters $\what{\theta}_k \in \R^d$ to define scores $s_k(x) =
\what{\theta}_k^T x$ (for the methods CDioC and \pgmtree) and the
``arbitrary'' predictor $\what{y}_k(x) = \sign(s_k(x))$ (for method
\scorelesstree).  We use a one-layer fully-connected neural network with
$16$ hidden neurons as the class $\mc{Q}$ in the quantile regression
step~\eqref{eqn:quantile-regression} of the methods; no matter our choice of
$\mc{Q}$, the final conformalization step guarantees (marginal) validity.


Figure~\ref{fig:multilabel-simulation-averagesize-coverage},
\ref{fig:multilabel-simulation-worst-slab-coverage}, and
\ref{fig:multilabel-simulation-conditional-coverage} show our results.  In
Fig.~\ref{fig:multilabel-simulation-averagesize-coverage}, we give the
marginal coverage of each method (left plot), including both the implicit
$\Cimplicit$ and explicity $\Cinout$ confidence sets, and the average
confidence set sizes for the explicit confidence set $\Cinout$ in the right
plot.  The explicit sets $\Cinout$ in the \scorelesstree and \pgmtree
methods both overcover, though not substantially more than the oracle
method; the sizes of $\Cinout$ for the \pgmtree and oracle methods are
similar (see supplementary
\figref{fig:multilabel-simulation-size-distribution}).  On the other hand,
the \pgmtree explicit confidence sets (which expand the implicit set
$\Cimplicit$ as in~\eqref{eqn:efficient-inner-outer-vecs}) cover more
than the direct Inner/Outer method CDioC.  The
confidence sets of the scoreless method \scorelesstree are wider, which is
consistent with the limitations of a method using only the
predictions $\what{y}_k = \sign(s_k)$.

\begin{figure}
 \centering
  \begin{overpic}[
  				scale=0.7]{%
     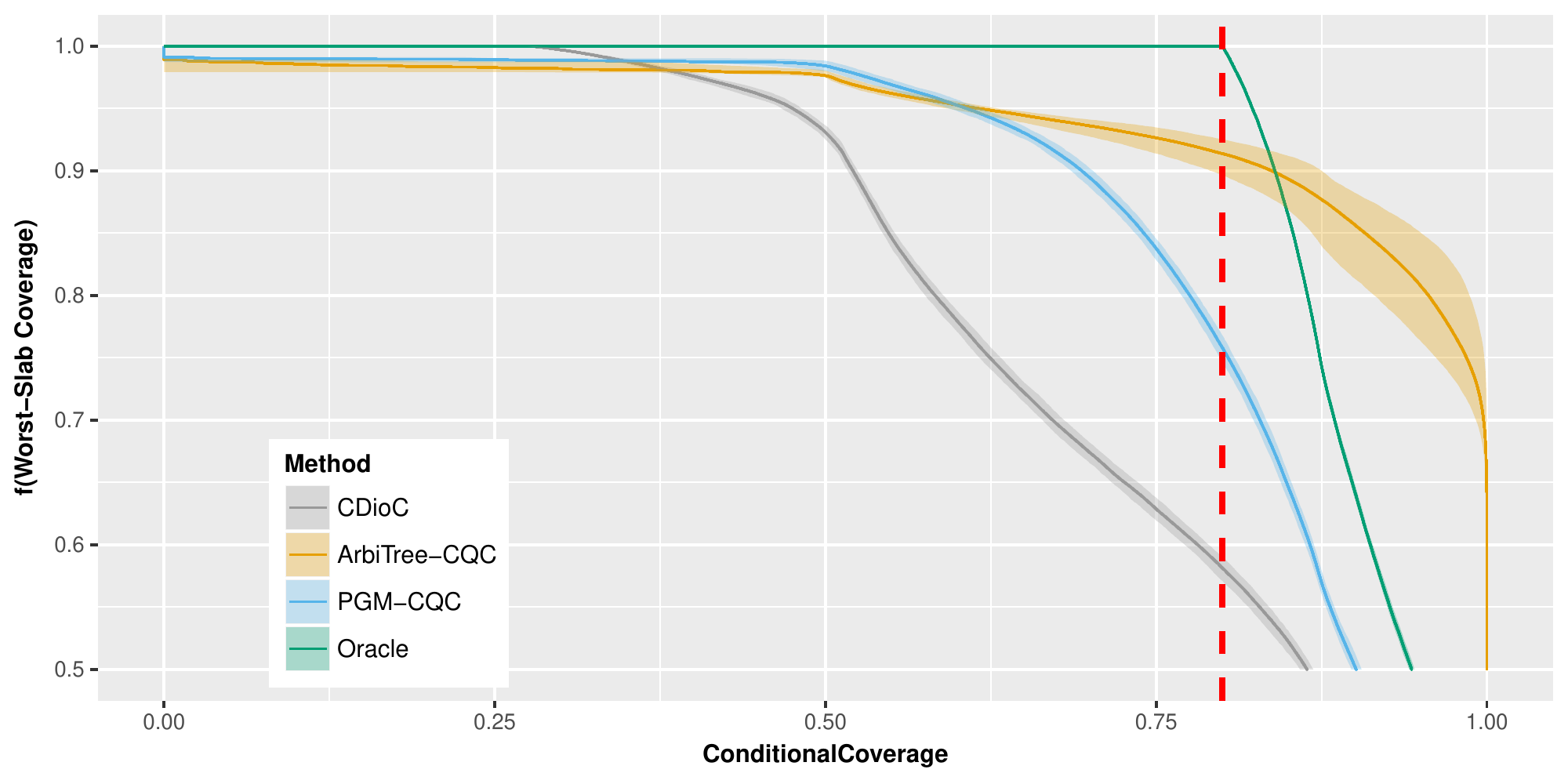}
    \put(0,5){
      \tikz{\path[draw=white, fill=white] (0, 0) rectangle (.3cm, 6cm)}
    }
    \put(0,5){\rotatebox{90}{
        \small $P_X(\{x :
        \P(Y \in \what{C}(X) \mid X = x) \ge t\})$}
    }
    \put(40, 0){
      \tikz{\path[draw=white, fill=white] (0, 0) rectangle (4cm, .4cm)}
    } 
    \put(30,1){
        \small Conditional coverage probability $t$
    }
  \end{overpic}
  \caption{Simulated multilabel experiment with label
    distribution~\eqref{eqn:multilabel-logreg}. The plot shows the
    $X$-probability of achieving a given level $t$ of conditional coverage
    versus coverage $t$, i.e., $t \mapsto P_X(\P(Y\in
    \what{C}_{\text{Method}}(X) \mid X) \geq t)$, using explicit confidence
    sets $\Cinout$. The ideal is to observe $t \mapsto \indic{t \leq
      1-\alpha}$. Confidence bands display the range of the statistic over
    $M=20$ trials.  }
  \label{fig:multilabel-simulation-conditional-coverage}
\end{figure}

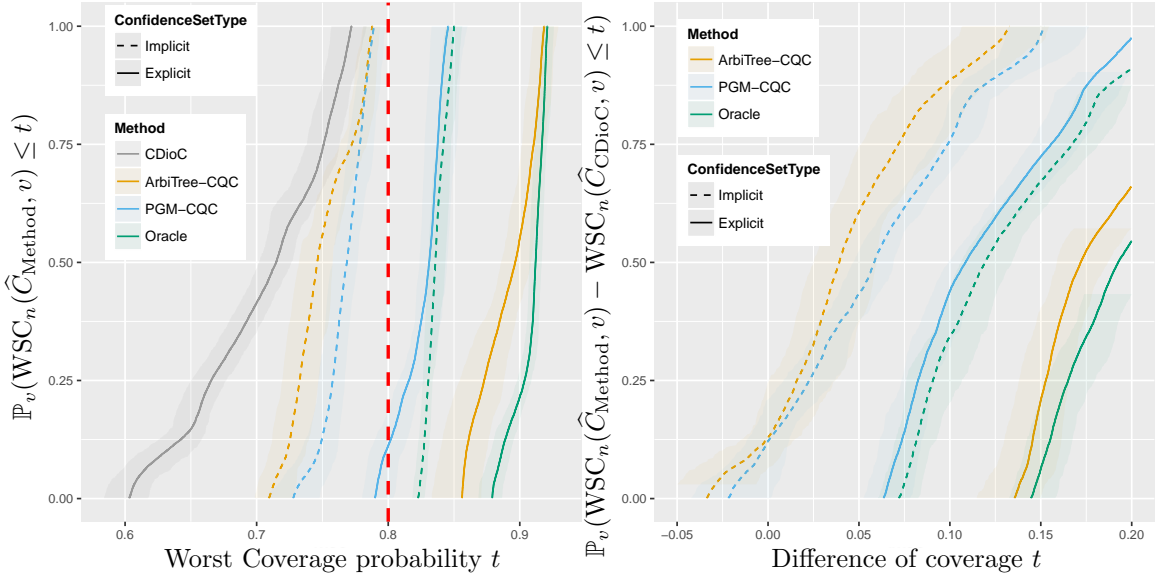
\begin{figure}[t]
  \centering
    \begin{overpic}[
  				scale=0.6]{%
     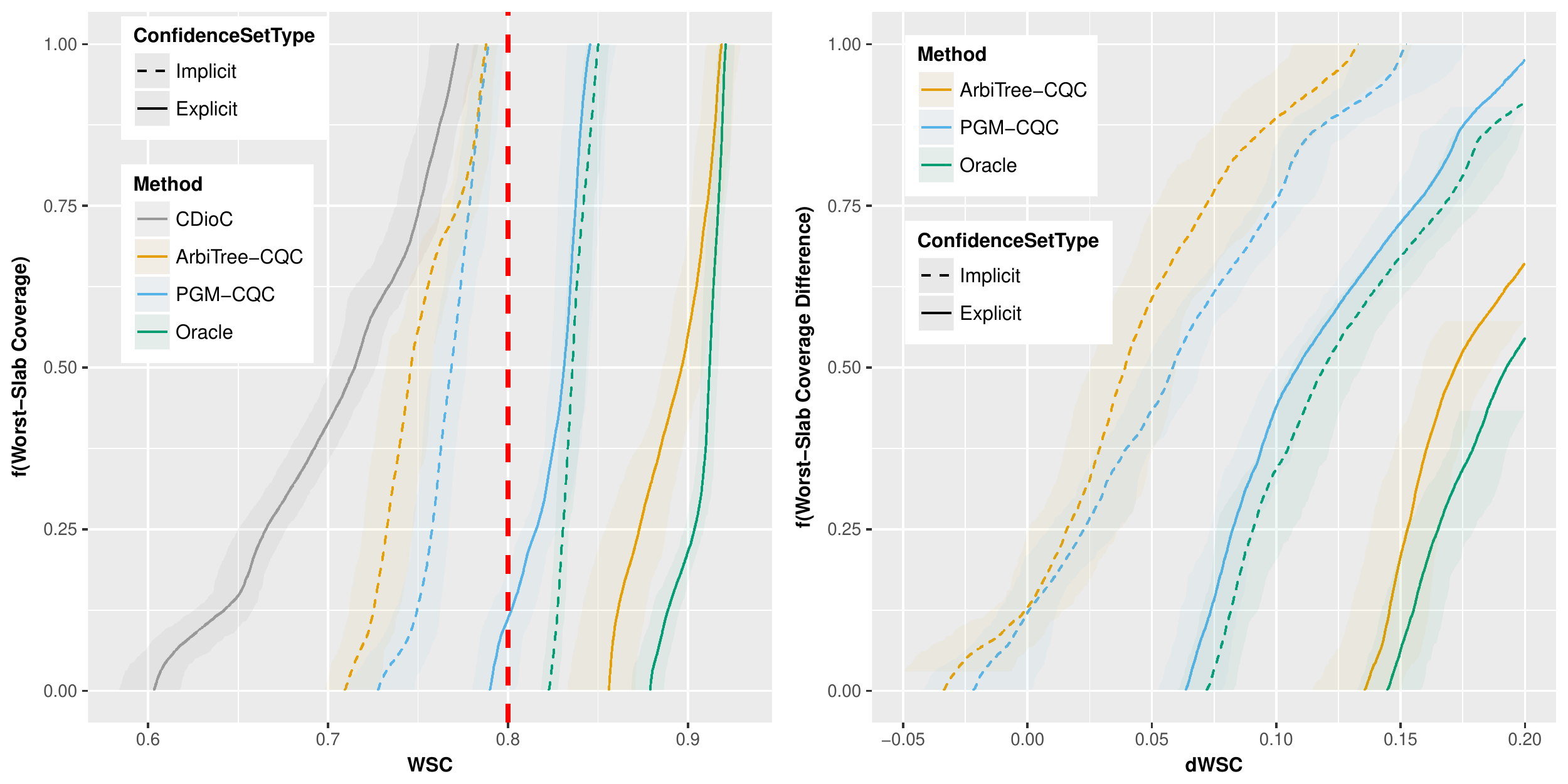}
    \put(-0.5,10){
      \tikz{\path[draw=white, fill=white] (0, 0) rectangle (.3cm, 6cm)}
    }
    \put(-1,12){\rotatebox{90}{
        \small $\P_v( \text{WSC}_n(\what{C}_{\text{Method}},v) \le t)$}
    }
    \put(15, 0){
      \tikz{\path[draw=white, fill=white] (0, 0) rectangle (4cm, .25cm)}
    }
        
    \put(60, 0){
      \tikz{\path[draw=white, fill=white] (0, 0) rectangle (4cm, .25cm)}
    }
    \put(49, 6){
      \tikz{\path[draw=white, fill=white] (0, 0) rectangle (.4cm, 7cm)}
    }
    \put(49, 0){\rotatebox{90}{
        \small $\P_v(\text{WSC}_n(\what{C}_{\textrm{Method}},v) -
    \text{WSC}_n(\what{C}_{\textrm{CDioC}},v) \le t) $}}
    
    \put(12, 0){
        \small Worst Coverage probability  $t$}
        
     \put(65, 0){
        \small Difference of coverage $t$}
    
  \end{overpic}
  \caption{\label{fig:multilabel-simulation-worst-slab-coverage} Cumulative
    distribution of worst-slab
    coverage~\eqref{eqn:worst-slab-computation} (with $\delta = 20\%$) on
    mis-specifed logistic regression model~\eqref{eqn:multilabel-logreg}
    over $M = 1000$ i.i.d.\ choices of direction $v \in \sphere^{d-1}$. We
    expect to have $\P_v(\text{WSC}_n(\what{C}, v) \le t ) = \indic{t \ge
      1-\alpha}$ if $\what{C}$ provides exact $1-\alpha$-conditional
    coverage. Left: CDF of $\text{WSC}_n(\what{C}, v)$.  Right:
    CDF of difference $\text{WSC}_n(\what{C}_{\textrm{Method}_1},v)
    - \text{WSC}_n(\what{C}_{\textrm{Method}_2},v)$, $\text{Method}_2$ is
    always the CDioC method (Alg.~\ref{alg:direct-inner-outer}), and the
    other four are \scorelesstree and \pgmtree with both implicit
    $\Cimplicit$ and explicit $\Cinout$ confidence sets.}
\end{figure}


Figures~\ref{fig:multilabel-simulation-conditional-coverage}
and~\ref{fig:multilabel-simulation-worst-slab-coverage} consider each
method's approximation to conditional coverage; the former with exact
calculations and the latter via worst-slab coverage
measures~\eqref{eqn:worst-slab-computation}.  Both plots dovetail with our
expectations that the \pgmtree and \scorelesstree methods are more robust
and feature-adaptive. Indeed,
Figure~\ref{fig:multilabel-simulation-conditional-coverage} shows that the
\pgmtree method provides at least coverage $1-\alpha$ for 75\% of the
examples, against only 60\% for the CDioC method, and has an overall
consistently higher coverage.  In
Figure~\ref{fig:multilabel-simulation-worst-slab-coverage}, each of the $M
=10^3$ experiments corresponds to a draw of $v \simiid
\uniform(\sphere^{d-1})$, then evaluating the worst-slab
coverage~\eqref{eqn:worst-slab-computation} with $\delta = .2$.  In the left
plot, we show its empirical cumulative distribution across draws of $v$ for
each method, which shows a substantial difference in coverage between the
CDioC method and the others.  We also perform direct comparisons in the
right plot: we draw the same directions $v$ for each method, and then (for
the given $v$) evaluate the difference $\textup{WSC}_n(\what{C}, v) -
\textup{WSC}_n(\what{C}', v)$, where $\what{C}$ and $\what{C}'$ are the
confidence sets each method produces, respectively. Thus, we see that both
tree-based methods always provide better worst-slab coverage, whether we use
the implicit confidence sets $\Cimplicit$ or the larger direct inner/outer
(explicit) confidence sets $\Cinout$, though in the latter case, some of the
difference likely comes from the differences in marginal coverage.  The
worst-slab coverage is consistent with the true conditional coverage
in that the relative ordering of method performance is consistent, suggesting
its usefulness as a proxy.

\subsection{More robust coverage on CIFAR 10 and ImageNet datasets}
\label{sec:experiments-cifar10}

\begin{figure}
    \begin{overpic}[
  				scale=0.6]{%
     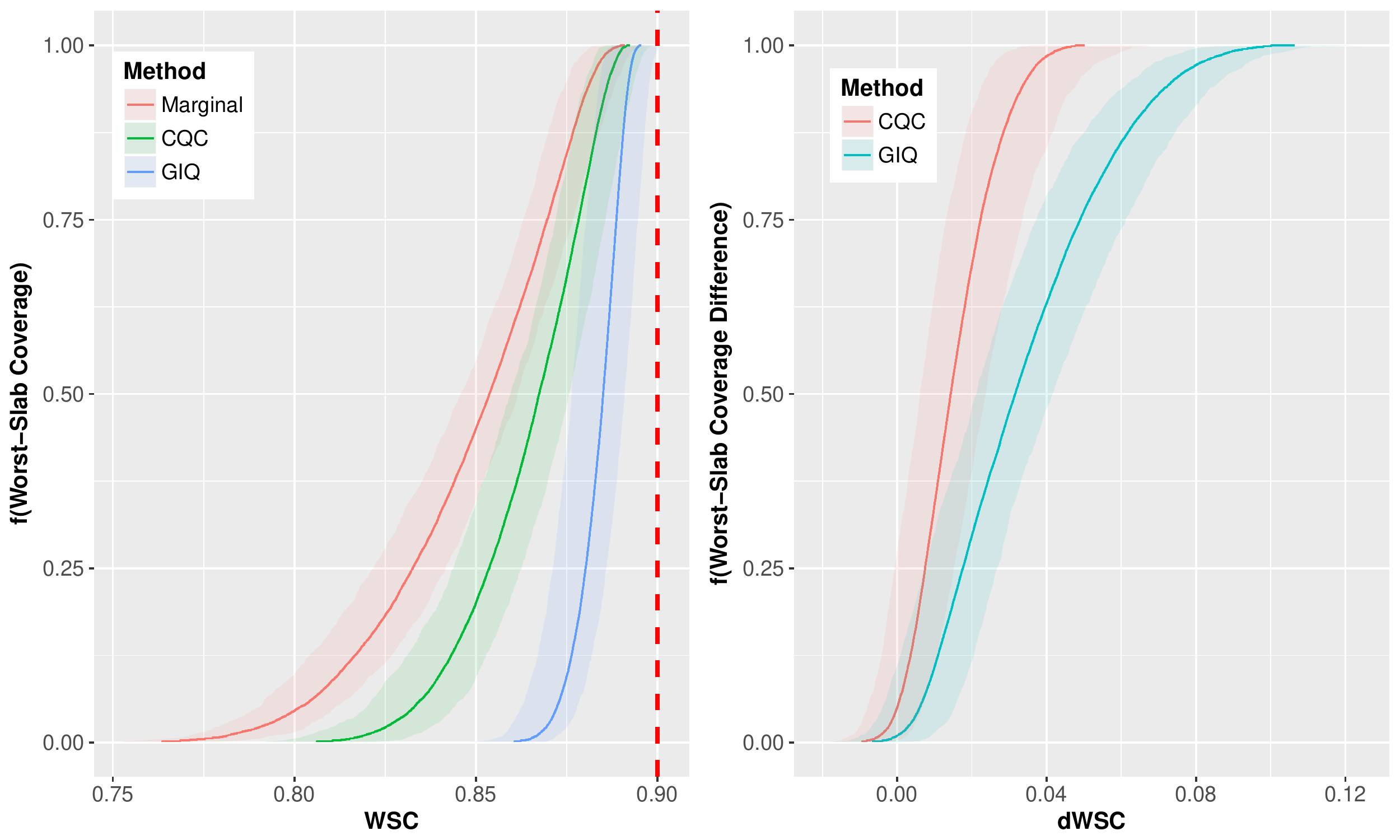}
    \put(-0.5,15){
      \tikz{\path[draw=white, fill=white] (0, 0) rectangle (.3cm, 6cm)}
    }
    \put(-1,16){\rotatebox{90}{
        \small $\P_v( \text{WSC}_n(\what{C}_{\text{Method}},v) \le t)$}
    }
    \put(15, 0){
      \tikz{\path[draw=white, fill=white] (0, 0) rectangle (4cm, .3cm)}
    }
    \put(60, 0){
      \tikz{\path[draw=white, fill=white] (0, 0) rectangle (4cm, .3cm)}
    }
    \put(49, 6){
      \tikz{\path[draw=white, fill=white] (0, 0) rectangle (.4cm, 7cm)}
    }
    \put(50, 7){\rotatebox{90}{
        \small $\P_v(\text{WSC}_n(\what{C}_{\textrm{Method}},v) -
    \text{WSC}_n(\what{C}_{\textrm{CDioC}},v) \le t) $}}
    \put(12, 0){
        \small Worst Coverage probability  $t$}
     \put(65, 0){
        \small Difference of coverage $t$}
  \end{overpic}
  
  \caption{Worst-slab coverage for ImageNet-10
    with $\delta=.2$ over $M = 1000$ draws
    $v \simiid \uniform(\sphere^{d-1})$.
    The dotted line is the desired (marginal) coverage.
    Left: CDF of worst-slab coverage.
    Right: CDF of
    coverage difference $\text{WSC}_n(\what{C}_{\textrm{CQC}},v) -
    \text{WSC}_n(\what{C}_{\textrm{Marginal}},v)$.}
  \label{fig:imagenet-worst-slab}
\end{figure}

In our first real experiments, we study two multiclass image classification
problems with the benchmark CIFAR-10~\cite{KrizhevskyHi09} and
ImageNet~\cite{DengDoSoLiLiFe09} datasets. We use similar approaches to
construct our feature vectors and scoring functions. With the CIFAR-10
dataset, which consists of $n =$ 60,000 $32 \times 32$ images across 10
classes, we use $n_{\text{tr}}$ = 50,000 of them to train the full model, a
validation set of size $n_{\text{v}} = $ 6,000 for fitting quantile
functions and hyperparameter tuning, $n_{\text{c}}= $ 3,000 for calibration
and the last $n_{\text{te}} = $ 1,000 for testing.  We train a standard
ResNet50~\cite{HeZhReSu16} architecture for $200$ epochs, obtaining test set
accuracy $92.5 \pm 0.5\%$, and use the $d=256$-dimensional output of the
final layer of the ResNet as the inputs $X$ to the quantile estimation. For
the ImageNet classification problem, we load a pre-trained
Inception-ResNetv2~\cite{SzegedyIoVaAl17} architecture, achieving top-1 test
accuracy $80.3 \pm 0.5 \%$, using the $d=1536$-dimensional output of the
final layer as features $X$. Splitting the original ImageNet validation set
containing 50,000 instances into 3 different sets, we fit our quantile
function on $n_{\text{v}}$ = 30,000 examples, calibrate on $n_{\text{c}}=$
10,000 and test on the last $n_{\text{te}} = $ 10,000 images.

We apply the CQC method~\ref{alg:cqc} with $\alpha=5\%$ for CIFAR-10 and
$\alpha=10\%$ for ImageNet and it to the benchmark marginal
method~\eqref{eqn:marginal-method} and GIQ~\cite{RomanoSeCa20}. We expect
the former to typically output small confidence sets, as the neural
network's accuracy is close to the confidence level $1-\alpha$; this allows
(typically) predicting a single label while maintaining marginal coverage.
Supplementary figures~\ref{fig:cifar10-averagesize-coverage}
and~\ref{fig:imagenet-averagesize-coverage} show this.  The worst-slab
coverage~\eqref{eqn:worst-slab-computation} tells a different story. In
Figure~\ref{fig:imagenet-worst-slab}, we compare worst-slab coverage over $M
= 1000$ draws of $v \simiid \uniform(\sphere^{d-1})$ on the ImageNet dataset
(see also Fig.~\ref{fig:cifar10-worst-slab} for the identical experiment
with CIFAR-10).  The CQC and GIQ methods provide significant 3--5\% and
5--7\% better coverage, respectively, in worst-slab coverage over the
marginal method.  The CQC and GIQ methods provide stronger gains on ImageNet
than on CIFAR-10, which we attribute to the relative easiness of CIFAR-10:
the accuracy of the classifier is high, allowing near coverage by Dirac
measures.


Figure~\ref{fig:imagenet-sizecdf} compares the confidence set sizes and
probability of coverage given confidence set size for the marginal, CQC, and
GIQ methods.  We summarize briefly. The CQC method gives confidence sets of
size at most 2 for 80\% of the instances---comparable to the marginal method
and more frequently than the GIQ method (which yields 75\% examples with
$|\what{C}(X)| \le 2$). Infrequently, the CQC and GIQ methods yield very
large confidence sets, with $|\what{C}(X)| \ge 200$ about 5\% of the time
for the GIQ method and the completely informative $|\what{C}(X)| = 1000$
about 2\% of the time for the CQC method. While the average confidence
set size $\E[|\what{C}(X)|]$ is smaller for the marginal method
(cf.\ supplementary Fig.~\ref{fig:imagenet-averagesize-coverage}),
this is evidently a very incomplete story.
The bottom plot in
Fig.~\ref{fig:imagenet-sizecdf} shows the behavior we expect for a marginal
method given a reasonably accurate classifier: it overcovers for examples
$x$ with $\what{C}(x)$ small. GIQ exhibits the opposite behavior,
overcovering when $\what{C}(x)$ is large and undercovering for small
$\what{C}(x)$, while CIQ provides nearly $1 - \alpha$ coverage roughly
independent of confidence set size, as one would expect for a method with
valid conditional coverage.  (In supplementary Fig.~\ref{fig:cifar-sizecdf},
we see similar but less-pronounced behavior on CIFAR-10.)



\begin{figure}
 \centering
  \begin{overpic}[
      scale=0.6]{%
     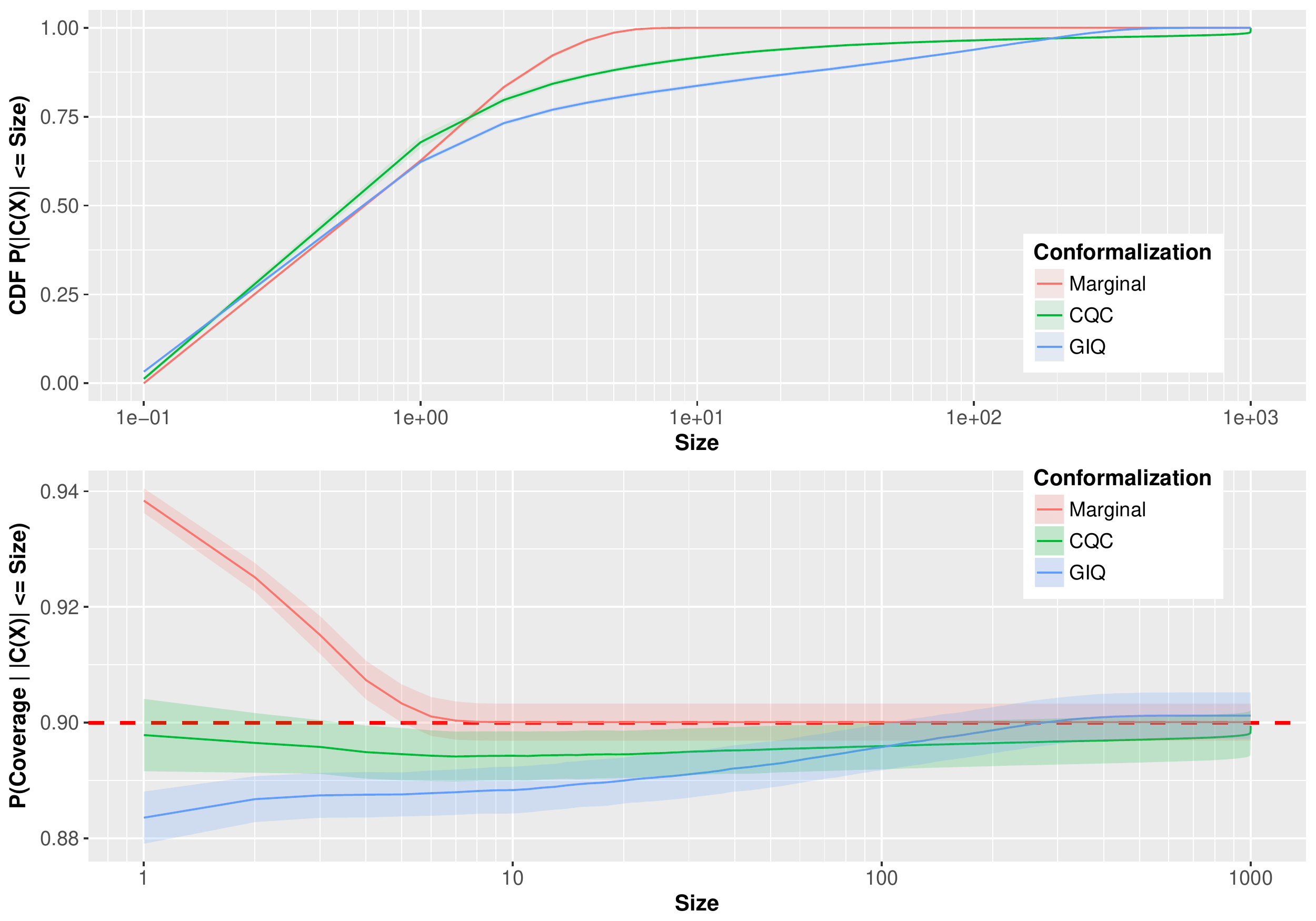}
    \put(-1,30){
      \tikz{\path[draw=white, fill=white] (0, 0) rectangle (.4cm, 6cm)}
    }
    \put(-1,0){
      \tikz{\path[draw=white, fill=white] (0, 0) rectangle (.4cm, 6cm)}
    }
    \put(-1,3){\rotatebox{90}{
        \small $\P( Y \in \what{C}(X) \mid \mbox{card}(\what{C}(X)) \le t)$}
    }
    \put(-1,44){\rotatebox{90}{
        \small $\P( \mbox{card}(\what{C}(X)) \le t)$}
    }   
    \put(0,37){
      \tikz{\path[draw=white, fill=white] (0, 0) rectangle (15cm, .3cm)}
    }
    \put(10.5, 37.4){\scriptsize 0}
    \put(31.5, 37.4){\scriptsize 1}
    \put(52, 37.4){\scriptsize 10}
    \put(72.6, 37.4){\scriptsize 100}
    \put(93.2, 37.4){\scriptsize 1000}
    \put(50, 0){
      \tikz{\path[draw=white, fill=white] (0, 0) rectangle (4cm, .3cm)}
    }
    \put(50, 35){
      \tikz{\path[draw=white, fill=white] (0, 0) rectangle (4cm, .3cm)}
    }
    \put(37,0){
        \small Confidence set size $t$ (log-scale)
    } 
        \put(37,35.5){
        \small Confidence set size $t$ (log-scale)
    }
  \end{overpic}
  \caption{ImageNet results over $20$ trials. Methods are the
    marginal method (Alg.~\ref{eqn:marginal-method}), the CQC
    method (Alg.~\ref{alg:cqc}), and the GIQ method
    (Alg. 1,~\cite{RomanoSeCa20}).  Top: cumulative distribution of
    confidence set size.  Bottom: probability of coverage
    conditioned on confidence set size.  }
  \label{fig:imagenet-sizecdf}
\end{figure}

\subsection{A multilabel image recognition dataset}
\label{sec:experiments-voc2012}

\begin{figure}
 \centering
  \begin{overpic}[
  				scale=0.75]{%
     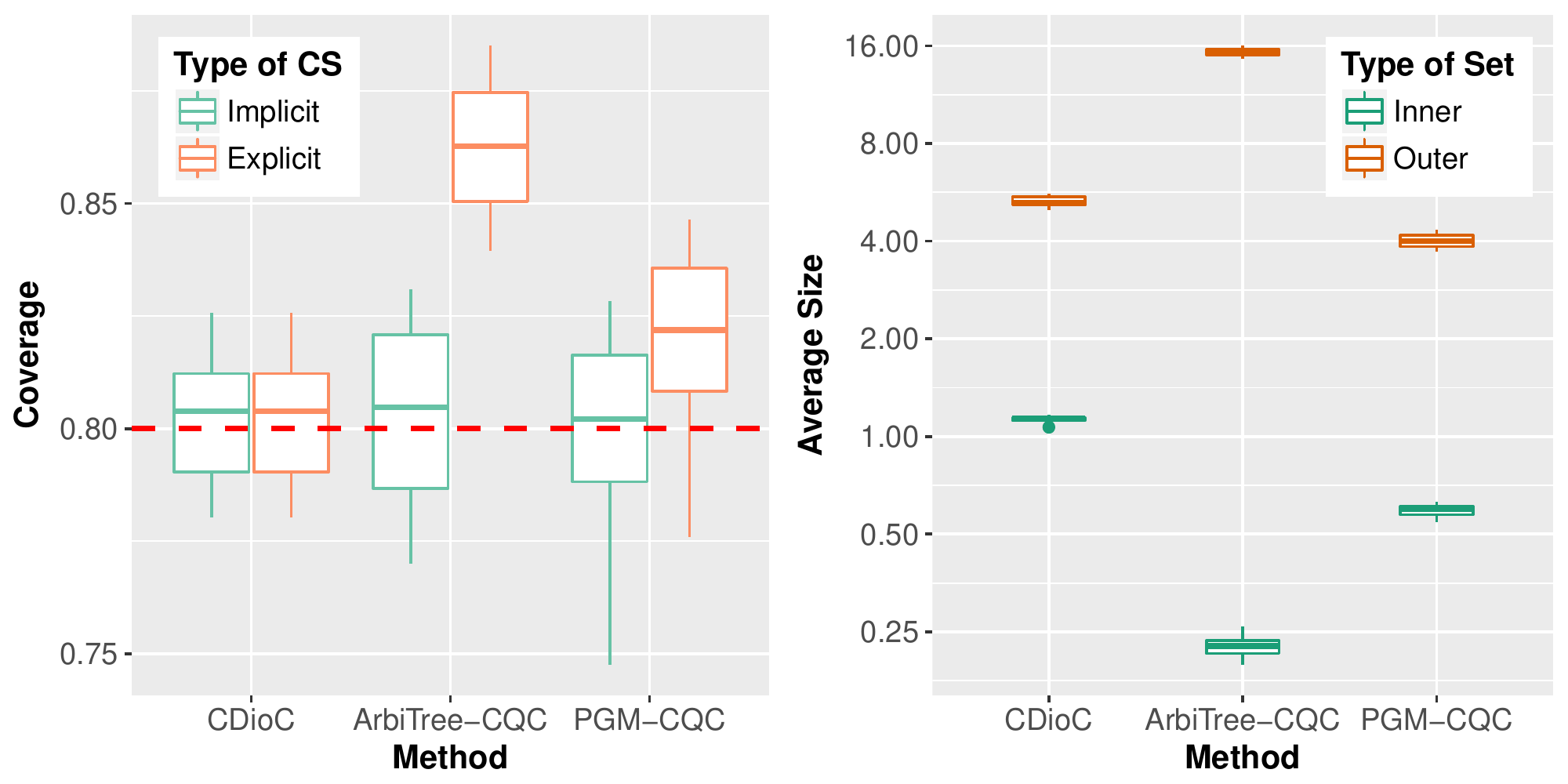}
    \put(0,10){
      \tikz{\path[draw=white, fill=white] (0, 0) rectangle (.3cm, 6cm)}
    }
    \put(0,22){\rotatebox{90}{
        \small $\P(Y \in \hat{C}(X))$}
    }
    \put(15, 0){
      \tikz{\path[draw=white, fill=white] (0, 0) rectangle (4cm, .4cm)}
    }
        
    \put(70, 0){
      \tikz{\path[draw=white, fill=white] (0, 0) rectangle (4cm, .4cm)}
    }
    \put(48, 13){
      \tikz{\path[draw=white, fill=white] (0, 0) rectangle (.6cm, 5cm)}
    }
    \put(51, 6){\rotatebox{90}{
        \small $\E\left|\yin(X)\right|$ and $\E\left|\yout(X)\right|$ (log-scale)}}
  \end{overpic}
  \caption{Pascal-VOC dataset~\cite{EveringhamVaWiWiZi12} results over $20$
    trials.  Methods are the conformalized direct inner/outer method (CDioC,
    Alg.~\ref{alg:direct-inner-outer}) and tree-based methods with implicit
    confidence sets $\Cimplicit$ or explicit inner/outer sets $\Cinout$,
    labeled \scorelesstreeshort and \pgmtreeshort (see
    Sec.~\ref{subsec:experiments-multilabel-simulation}).  Left: Marginal
    coverage probability at level $1 - \alpha = .8$.
    Right: confidence set sizes of methods.  }
    \label{fig:voc-averagesize}
\end{figure}

\begin{figure}
\centering
  \begin{overpic}[
  				scale=0.62]{%
     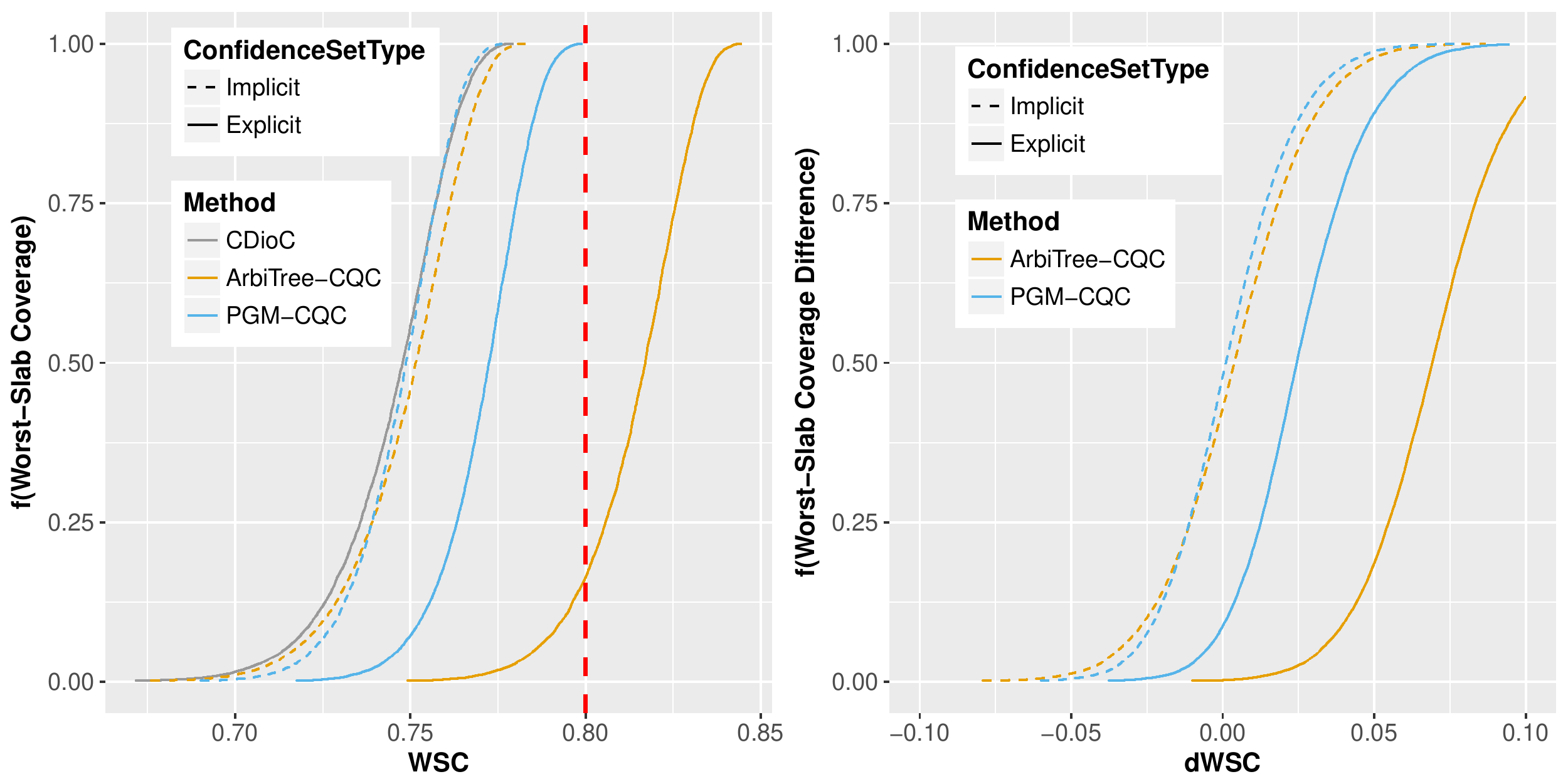}
        \put(-0.5,10){
      \tikz{\path[draw=white, fill=white] (0, 0) rectangle (.3cm, 6cm)}
    }
    \put(-1,15){\rotatebox{90}{
        \small $\P_v( \text{WSC}_n(\what{C}_{\text{Method}},v) \le t)$}
    }
    \put(15, 0){
      \tikz{\path[draw=white, fill=white] (0, 0) rectangle (4cm, .33cm)}
    }
        
    \put(60, 0){
      \tikz{\path[draw=white, fill=white] (0, 0) rectangle (4cm, .33cm)}
    }
    \put(49, 6){
      \tikz{\path[draw=white, fill=white] (0, 0) rectangle (.4cm, 7cm)}
    }
    \put(50, 1){\rotatebox{90}{
        \small $\P_v(\text{WSC}_n(\what{C}_{\textrm{Method}},v) -
    \text{WSC}_n(\what{C}_{\textrm{CDioC}},v) \le t) $}}
    
    \put(10, 0){
        \small Worst Coverage probability  $t$}
        
     \put(65, 0){
        \small Difference of coverage  $t$}
  \end{overpic}
\caption{ Worst-slab coverage~\eqref{eqn:worst-slab-computation} for
  Pascal-VOC with $\delta=.2$ over $M = 1000$ draws $v \simiid
  \uniform(\sphere^{d-1})$.  For tree-structured methods ArbiTree-CQC and
  PGM-CQC, we compute the worst-slab coverage using implicit confidence sets
  $\Cimplicit$ and explicit inner/outer sets $\Cinout$.  The dotted line is
  the desired (marginal) coverage.  Left: distribution of worst-slab
  coverage.  Right: distribution of the coverage difference
  $\text{WSC}_n(\what{C}_{\textrm{Method}},v) -
  \text{WSC}_n(\what{C}_{\textrm{CDioC}},v)$ for $\text{Method}_i \in \{
  \text{ArbiTree-CQC}, \text{PGM-CQC}\}$ with implicit $\Cimplicit$ or
  explicit inner/outer $\Cinout$ confidence sets.  }
\label{fig:voc-worst-slab-coverage}
\end{figure}

Our final set of experiments considers the multilabel image classification
problems in the PASCAL VOC 2007 and VOC 2012
datasets~\cite{EveringhamVaWiWiZi07, EveringhamVaWiWiZi12}, which consist of
$n_{2012} = 11540$ and $n_{2007}=9963$ distinct $224 \times 224$ images,
where the goal is to predict the presence of entities from $K=20$
different classes, (e.g.\ birds, boats, people).
We compare the direct inner outer method
(CDioC)~\ref{alg:direct-inner-outer} with the split-conformalized
inner/outer method (CQioC)~\ref{alg:tree-cqc}, where we use the tree-based
score functions that Algorthims~\ref{alg:score-tree} and
\ref{alg:multilabel-pgm-method} output. For the PGM method, which performs best in practice, we additionally compare the performance of standard inner and outer sets (see Alg.~\ref{alg:tree-cqc}), to the refinement that we describe in section~\ref{sec:union-inner-outer-sets}, where we instead output a confidence set as a union of inner and outer sets. Here, we choose $m=2$, which corresponds to outputting a union of $4$ inner/outer sets in equation~\ref{eqn:union-multilabel-io-set}, and select the indices $I$ according to the heuristics that we describe in that same section.

In this problem, we use the $d=2048$ dimensional output of a ResNet-101 with
pretrained weights on the ImageNet dataset~\cite{DengDoSoLiLiFe09} as our
feature vectors $X$. For each of the $K$ classes, we fit a binary logistic
regression model $\what{\theta}_k \in \R^d$ of $Y_k \in \{\pm 1\}$ against
$X$, then use $s_k(x) = \what{\theta}_k^T x$ as the scores for the \pgmtree
method (Alg.~\ref{alg:multilabel-pgm-method}).  We use $\what{y}_k(x) =
\sign(s_k(x))$ for the \scorelesstree method~\ref{alg:score-tree}.
The fit predictors have F1-score $0.77$ on held-out data, so we do not
expect our confidence sets to be uniformly small while maintaining the
required level of coverage, in particular for \scorelesstree.
We use a validation set of $n_{\text{v}} = 3493$ images to
fit the quantile functions~\eqref{eqn:quantile-regression} as above
using a one layer fully-connected neural network
with $16$ hidden neurons and tree parameters.

The results from Figure~\ref{fig:voc-worst-slab-coverage} are consistent
with our hypotheses that the tree-based models should improve robustness of
the coverage. Indeed, while all methods ensure marginal coverage at level
$\alpha = .8$ (see \figref{fig:voc-averagesize}),
Figure~\ref{fig:voc-worst-slab-coverage} shows that worst-case
slabs~\eqref{eqn:worst-slab-computation} for the tree-based methods have
closer to $1-\alpha$ coverage.  In particular, for most random slab
directions $v$, the tree-based methods have higher worst-slab coverage than
the direct inner/outer method (CDioC, Alg.~\ref{alg:direct-inner-outer}).
At the same time, both the CDioC method and \pgmtree method (using
$\Cinout$) provide similarly-sized confidence sets, as the inner set of the
\pgmtree method is typically smaller, as is its outer set
(cf.\ \figref{fig:multilabel-voc-size-distribution}).
The \scorelesstree method performs poorly on this example: its coverage is
at the required level, but the confidence sets are too large and thus
essentially uninformative.

\begin{figure}
\centering
  \begin{overpic}[
  				scale=0.62]{%
     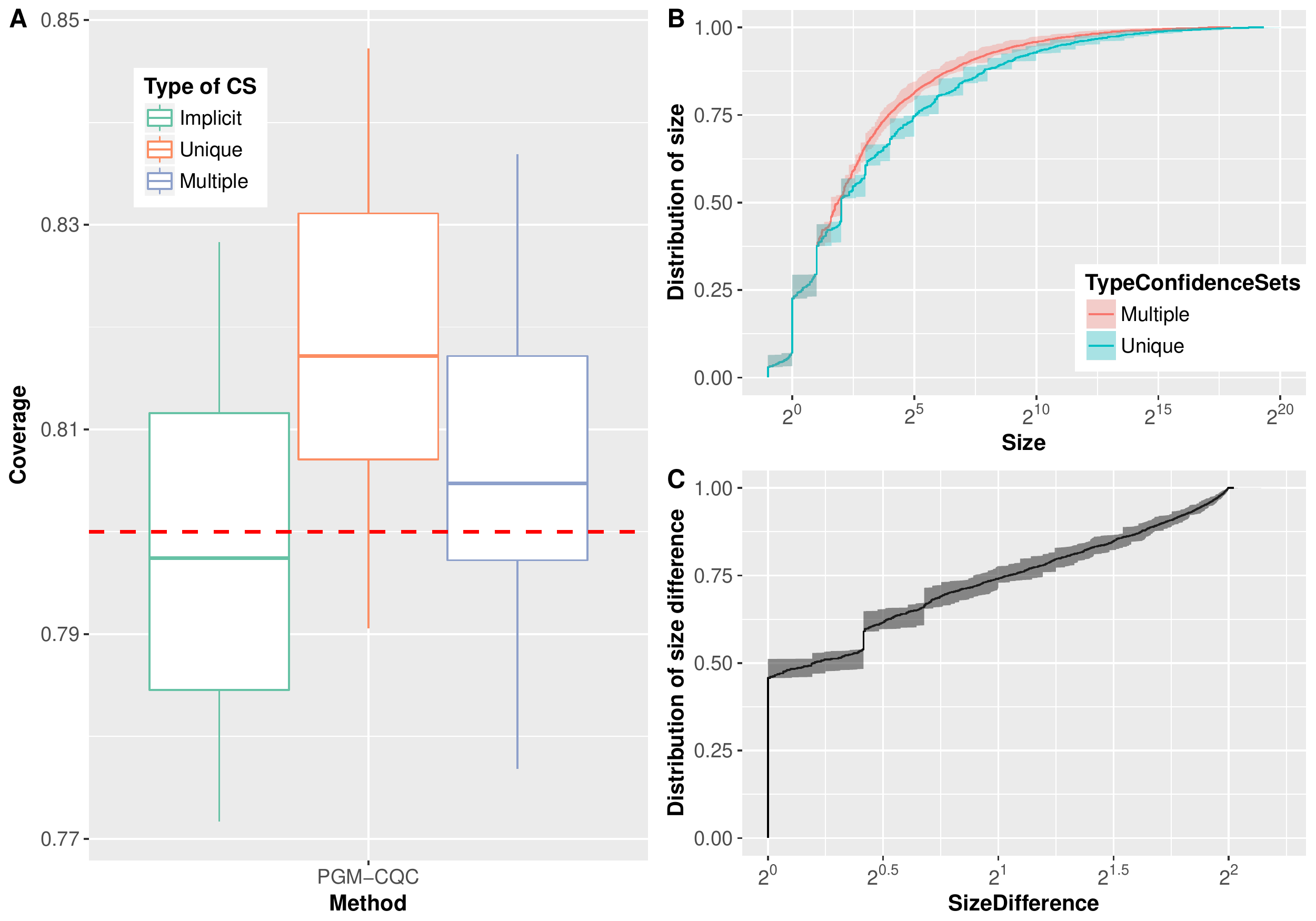}
        \put(-0.5,10){
      \tikz{\path[draw=white, fill=white] (0, 0) rectangle (.3cm, 6cm)}
    }
    \put(-1,28){\rotatebox{90}{
        \small $\P( Y \in \what{C}(X)$}
    }
    \put(15, 0){
      \tikz{\path[draw=white, fill=white] (0, 0) rectangle (4cm, .33cm)}
    }        
    \put(60, 0){
      \tikz{\path[draw=white, fill=white] (0, 0) rectangle (4cm, .33cm)}
    }    
    \put(60, 35){
      \tikz{\path[draw=white, fill=white] (0, 0) rectangle (4cm, .33cm)}
    }    
    \put(49, 6){
      \tikz{\path[draw=white, fill=white] (0, 0) rectangle (.4cm, 4cm)}
    }    
    \put(49, 46){
      \tikz{\path[draw=white, fill=white] (0, 0) rectangle (.4cm, 3cm)}
    }
    \put(49.5, 3){\rotatebox{90}{
        \small $\P( |\Cinout(x)| / |\Cinout(x, I)| \le t ) $}}
    \put(49.5, 45){\rotatebox{90}{
        \small $\P( |\what{C}(x)| \le t ) $}}
    \put(15, 0){
        \small Worst Coverage probability  $t$}
     \put(75, 36){
        \small Size  $t$}
      \put(70, 0){
        \small Ratio of size  $t$}
  \end{overpic}
\caption{ Comparison between standard inner/outer
  sets~\eqref{eqn:multilabel-io-set} (``Unique" in the figure) and union of
  inner/outer sets~\eqref{eqn:union-multilabel-io-set} (``Multiple"), on
  Pascal-VOC dataset.  The dotted line is the desired (marginal)
  coverage. The label index set $I$ contains two labels for each instance
  $x$ (see Sec.~\ref{sec:union-inner-outer-sets}).  A: summary box plot of
  coverage for each type of confidence set over $20$ trials.  B: cumulative
  distribution of the number of total configurations in each confidence set.  C:
  cumulative distribution of the ratio of size of confidence
  sets~\eqref{eqn:multilabel-io-set} and~\eqref{eqn:union-multilabel-io-set}
  (with $m=2$).}
\label{fig:voc-comparison-union-io-sets}
\end{figure}
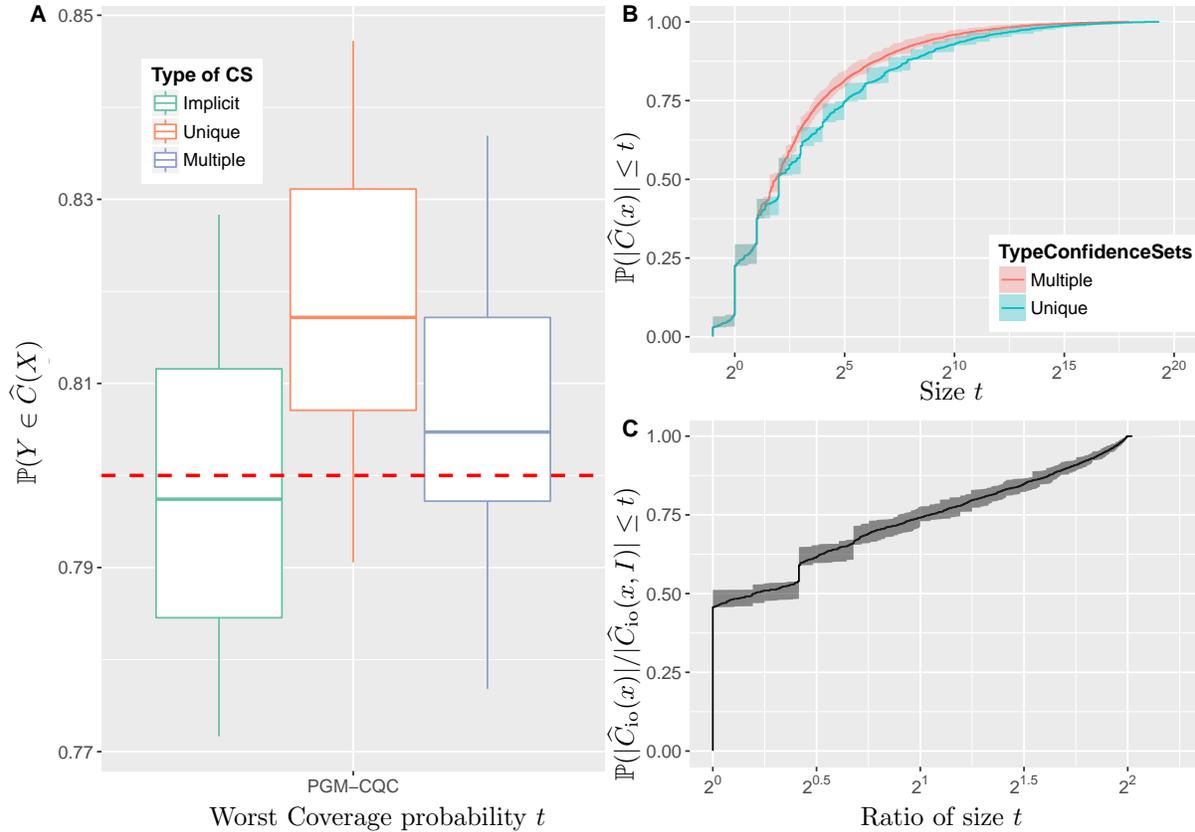

Finally, Figure~\ref{fig:voc-comparison-union-io-sets} investigates refining
the confidence sets as we suggest in
Section~\ref{sec:union-inner-outer-sets} by taking pairs of the most
negatively correlated labels $i, j$ satisfying $\yin(x)_{(i,j)} = -1$ and
$\yout(x)_{(i,j)} = 1$. Moving from a single inner/outer set to the union of
4 inner/outer sets non only mitigates increased coverage, moving closer to
the target level $1-\alpha$ and to the coverage of the implicit confidence
sets, but in addition, for more than half of the examples, it
decreases (by a factor up to $4=2^m$) the number of configurations
in the confidence set.


\section{Conclusions}
\label{sec:conclusion}

As long as we have access to a validation set, independent of (or at least
exchangeable with) the sets used to fit a predictive model, split conformal
methods guarantee marginal validity, which gives great freedom in
modeling. It is thus of interest to fit models more adaptive to the signal
inputs $x$ at hand---say, by quantile predictions as we have done, or other
methods yet to be discovered---that can then in turn be conformalized. As
yet we have limited understanding of what ``near'' conditional coverage
might be possible: \citet{BarberCaRaTi19a} provide procedures that can
guarantee coverage uniformly across subsets of $X$-space as long as those
subsets are not too complex, but it appears computationally challenging to
fit the models they provide, and our procedures empirically appear to have
strong coverage across slabs of the data as in
Eq.~\eqref{eqn:worst-slab-computation}. Our work, then, is a stepping stone
toward more uniform notions of validity, and we believe
exploring approaches to this---perhaps by distributionally robust
optimization~\cite{DelageYe10, DuchiNa18}, perhaps by uniform
convergence arguments~\cite[Sec.~4]{BarberCaRaTi19a}---will be both
interesting and essential for trusting applications of statistical
machine learning.


\setlength{\bibsep}{2pt}
\bibliographystyle{plainnat} 
\bibliography{bib}

\begin{thebibliography}{43}
\providecommand{\natexlab}[1]{#1}
\providecommand{\url}[1]{\texttt{#1}}
\expandafter\ifx\csname urlstyle\endcsname\relax
  \providecommand{\doi}[1]{doi: #1}\else
  \providecommand{\doi}{doi: \begingroup \urlstyle{rm}\Url}\fi

\bibitem[Barber et~al.(2019)Barber, Cand\`{e}s, Ramdas, and
  Tibshirani]{BarberCaRaTi19a}
Rina~Foygel Barber, Emmanuel~J. Cand\`{e}s, Aaditya Ramdas, and Ryan~J.
  Tibshirani.
\newblock The limits of distribution-free conditional predictive inference.
\newblock \emph{arXiv:1903.04684v2 [math.ST]}, 2019.

\bibitem[Bartlett et~al.(2006)Bartlett, Jordan, and McAuliffe]{BartlettJoMc06}
P.~L. Bartlett, M.~I. Jordan, and J.~McAuliffe.
\newblock Convexity, classification, and risk bounds.
\newblock \emph{Journal of the American Statistical Association}, 101:\penalty0
  138--156, 2006.

\bibitem[Boucheron et~al.(2005)Boucheron, Bousquet, and
  Lugosi]{BoucheronBoLu05}
St\'ephane Boucheron, Olivier Bousquet, and G\'abor Lugosi.
\newblock Theory of classification: a survey of some recent advances.
\newblock \emph{ESAIM: Probability and Statistics}, 9:\penalty0 323--375, 2005.

\bibitem[Boutell et~al.(2004)Boutell, Luo, Shen, and Brown]{BoutellLuShBr04}
Matthew~R Boutell, Jiebo Luo, Xipeng Shen, and Christopher~M Brown.
\newblock Learning multi-label scene classification.
\newblock \emph{Pattern Recognition}, 37\penalty0 (9):\penalty0 1757--1771,
  2004.

\bibitem[Boyd and Vandenberghe(2004)]{BoydVa04}
Stephen Boyd and Lieven Vandenberghe.
\newblock \emph{Convex Optimization}.
\newblock Cambridge University Press, 2004.

\bibitem[Chow and Liu(1968)]{ChowLi68}
C.~K. Chow and C.~N. Liu.
\newblock Approximating discrete probability distributions with dependence
  trees.
\newblock \emph{IEEE Transactions on Information Theory}, 14\penalty0
  (3):\penalty0 462--467, 1968.

\bibitem[Dawid(1992)]{Dawid92a}
A.~P. Dawid.
\newblock Applications of a general propagation algorithm for probabilistic
  expert systems.
\newblock \emph{Statistics and Computing}, 2\penalty0 (1):\penalty0 25--36,
  1992.

\bibitem[Delage and Ye(2010)]{DelageYe10}
Erick Delage and Yinyu Ye.
\newblock Distributionally robust optimization under moment uncertainty with
  application to data-driven problems.
\newblock \emph{Operations Research}, 58\penalty0 (3):\penalty0 595--612, 2010.

\bibitem[Deng et~al.(2009)Deng, Dong, Socher, Li, Li, and
  Fei-Fei]{DengDoSoLiLiFe09}
J.~Deng, W.~Dong, R.~Socher, L.~Li, K.~Li, and L.~Fei-Fei.
\newblock Image{N}et: a large-scale hierarchical image database.
\newblock In \emph{Proceedings of the IEEE Conference on Computer Vision and
  Pattern Recognition}, pages 248--255, 2009.

\bibitem[Duchi and Namkoong(2018)]{DuchiNa18}
John~C. Duchi and Hongseok Namkoong.
\newblock Learning models with uniform performance via distributionally robust
  optimization.
\newblock \emph{arXiv:1810.08750 [stat.ML]}, 2018.

\bibitem[Esteva et~al.(2017)Esteva, Kuprel, Novoa, Ko, Swetter, Blau, and
  Thrun]{EstevaKuNoKoSwBlTh17}
Andre Esteva, Brett Kuprel, Roberto~A. Novoa, Justin Ko, Susan~M. Swetter,
  Helen~M. Blau, and Sebastian Thrun.
\newblock Dermatologist-level classification of skin cancer with deep neural
  networks.
\newblock \emph{Nature}, 542:\penalty0 115--118, 2017.

\bibitem[Everingham et~al.(2007)Everingham, Van~Gool, Williams, Winn, and
  Zisserman]{EveringhamVaWiWiZi07}
M.~Everingham, L.~Van~Gool, C.~K.~I. Williams, J.~Winn, and A.~Zisserman.
\newblock The {PASCAL} {V}isual {O}bject {C}lasses {C}hallenge 2007 {(VOC2007)}
  {R}esults, 2007.
\newblock URL
  \url{http://www.pascal-network.org/challenges/VOC/voc2007/workshop/index.html}.

\bibitem[Everingham et~al.(2012)Everingham, Van~Gool, Williams, Winn, and
  Zisserman]{EveringhamVaWiWiZi12}
M.~Everingham, L.~Van~Gool, C.~K.~I. Williams, J.~Winn, and A.~Zisserman.
\newblock The {PASCAL} {V}isual {O}bject {C}lasses {C}hallenge 2012 {(VOC2012)}
  {R}esults, 2012.
\newblock URL
  \url{http://www.pascal-network.org/challenges/VOC/voc2012/workshop/index.html}.

\bibitem[Hardt et~al.(2016)Hardt, Price, and Srebro]{HardtPrSr16}
Moritz Hardt, Eric Price, and Nati Srebro.
\newblock Equality of opportunity in supervised learning.
\newblock In \emph{Advances in Neural Information Processing Systems 29}, 2016.

\bibitem[Hashimoto et~al.(2018)Hashimoto, Srivastava, Namkoong, and
  Liang]{HashimotoSrNaLi18}
Tatsunori Hashimoto, Megha Srivastava, Hongseok Namkoong, and Percy Liang.
\newblock Fairness without demographics in repeated loss minimization.
\newblock In \emph{Proceedings of the 35th International Conference on Machine
  Learning}, 2018.

\bibitem[He et~al.(2016)He, Zhang, Ren, and Sun]{HeZhReSu16}
Kaiming He, Xiangyu Zhang, Shaoqing Ren, and Jian Sun.
\newblock Deep residual learning for image recognition.
\newblock In \emph{Proceedings of the IEEE Conference on Computer Vision and
  Pattern Recognition}, pages 770--778, 2016.

\bibitem[Hechtlinger et~al.(2019)Hechtlinger, P\'{o}czos, and
  Wasserman]{HechtlingerPoWa19}
Yotam Hechtlinger, Barnab\'{a}s P\'{o}czos, and Larry Wasserman.
\newblock Cautious deep learning.
\newblock \emph{arXiv:1805.09460 [stat.ML]}, 2019.

\bibitem[Herrera et~al.(2016)Herrera, Charte, Rivera, and
  Del~Jesus]{HerreraChRiDe16}
Francisco Herrera, Francisco Charte, Antonio~J Rivera, and Mar{\'\i}a~J
  Del~Jesus.
\newblock Multilabel classification.
\newblock In \emph{Multilabel Classification}, pages 17--31. Springer, 2016.

\bibitem[Joulin et~al.(2016)Joulin, Grave, Bojanowski, and
  Mikolov]{JoulinGrBoMi16}
Armand Joulin, Edouard Grave, Piotr Bojanowski, and Tomas Mikolov.
\newblock Bag of tricks for efficient text classification.
\newblock \emph{arXiv:1607.01759 [cs.CL]}, 2016.

\bibitem[Kalra and Paddock(2016)]{KalraPa16}
Nidhi Kalra and Susan~M Paddock.
\newblock Driving to safety: How many miles of driving would it take to
  demonstrate autonomous vehicle reliability?
\newblock \emph{Transportation Research Part A: Policy and Practice},
  94:\penalty0 182--193, 2016.

\bibitem[Koenker and Jr.(1978)]{KoenkerBa78}
Roger Koenker and Gilbert~Bassett Jr.
\newblock Regression quantiles.
\newblock \emph{Econometrica: Journal of the Econometric Society}, 46\penalty0
  (1):\penalty0 33--50, 1978.

\bibitem[Koller and Friedman(2009)]{KollerFr09}
Daphne Koller and Nir Friedman.
\newblock \emph{Probabilistic Graphical Models: Principles and Techniques}.
\newblock MIT Press, 2009.

\bibitem[Krizhevsky and Hinton(2009)]{KrizhevskyHi09}
Alex Krizhevsky and Geoffrey Hinton.
\newblock Learning multiple layers of features from tiny images.
\newblock Technical report, University of Toronto, 2009.

\bibitem[Lafferty et~al.(2001)Lafferty, McCallum, and Pereira]{LaffertyMcPer01}
J.~Lafferty, A.~McCallum, and F.~Pereira.
\newblock Conditional random fields: Probabilistic models for segmenting and
  labeling sequence data.
\newblock In \emph{Proceedings of the Eighteenth International Conference on
  Machine Learning}, pages 282--289, 2001.

\bibitem[Lin et~al.(2014)Lin, Maire, Belongie, Hays, Perona, Ramanan,
  Doll{\'a}r, and Zitnick]{LinMaBeHaPeRaDoZi14}
Tsung-Yi Lin, Michael Maire, Serge Belongie, James Hays, Pietro Perona, Deva
  Ramanan, Piotr Doll{\'a}r, and C.~Lawrence Zitnick.
\newblock Microsoft {COCO}: Common objects in context.
\newblock In \emph{Computer Vision -- ECCV 2014}, pages 740--755, Cham, 2014.
  Springer International Publishing.

\bibitem[Lodhi et~al.(2002)Lodhi, Shawe-Taylor, Cristianini, and
  Watkins]{LodhiSTCrWa02}
Huma Lodhi, John Shawe-Taylor, Nello Cristianini, and Christopher J. C.~H.
  Watkins.
\newblock Text classification using string kernels.
\newblock \emph{Journal of Machine Learning Research}, 2:\penalty0 419--444,
  2002.

\bibitem[McCallum and Nigam(1998)]{McCallumNi98}
Andrew McCallum and Kamal Nigam.
\newblock Employing {EM} in pool-based active learning for text classification.
\newblock In \emph{Machine Learning: Proceedings of the Fifteenth International
  Conference}, 1998.

\bibitem[min Chung and Lu(2003)]{ChungLu03}
Kai min Chung and Hsueh-I Lu.
\newblock An optimal algorithm for the maximum-density segment problem.
\newblock In \emph{Proceedings of the 11th Annual European Symposium on
  Algorithms}, 2003.

\bibitem[Oakden-Rayner et~al.(2020)Oakden-Rayner, Dunnmon, Carneiro, and
  R{\'e}]{OakdenDuCaRe20}
Luke Oakden-Rayner, Jared Dunnmon, Gustavo Carneiro, and Christopher R{\'e}.
\newblock Hidden stratification causes clinically meaningful failures in
  machine learning for medical imaging.
\newblock In \emph{Proceedings of the ACM Conference on Health, Inference, and
  Learning}, pages 151--159, 2020.

\bibitem[Read et~al.(2011)Read, Pfahringer, Holmes, and Frank]{ReadPfHoFr11}
Jesse Read, Bernhard Pfahringer, Geoff Holmes, and Eibe Frank.
\newblock Classifier chains for multi-label classification.
\newblock \emph{Machine Learning}, 85:\penalty0 333--359, 2011.

\bibitem[Redmon et~al.(2016)Redmon, Divvala, Girshick, and
  Farhadi]{RedmonDiGiFa16}
Joseph Redmon, Santosh Divvala, Ross Girshick, and Ali Farhadi.
\newblock You only look once: Unified, real-time object detection.
\newblock In \emph{Proceedings of the 26th IEEE Conference on Computer Vision
  and Pattern Recognition}, pages 779--788, 2016.

\bibitem[Romano et~al.(2019)Romano, Patterson, and Cand\`{e}s]{RomanoPaCa19}
Yaniv Romano, Evan Patterson, and Emmanuel~J. Cand\`{e}s.
\newblock Conformalized quantile regression.
\newblock In \emph{Advances in Neural Information Processing Systems 32}, 2019.
\newblock URL \url{https://arxiv.org/abs/1905.03222}.

\bibitem[Romano et~al.(2020)Romano, Sesia, and Cand\`{e}s]{RomanoSeCa20}
Yaniv Romano, Matteo Sesia, and Emmanuel~J. Cand\`{e}s.
\newblock Classification with valid and adaptive coverage.
\newblock \emph{arXiv:2006.02544 [stat.ME]}, 2020.

\bibitem[Sadinle et~al.(2019)Sadinle, Lei, and Wasserman]{SadinleLeWa19}
Mauricio Sadinle, Jing Lei, and Larry Wasserman.
\newblock Least ambiguous set-valued classifiers with bounded error levels.
\newblock \emph{Journal of the American Statistical Association}, 114\penalty0
  (525):\penalty0 223--234, 2019.
\newblock \doi{10.1080/01621459.2017.1395341}.
\newblock URL \url{https://doi.org/10.1080/01621459.2017.1395341}.

\bibitem[Sesia and Cand{\`e}s(2019)]{SesiaCa19}
Matteo Sesia and Emmanuel~J. Cand{\`e}s.
\newblock A comparison of some conformal quantile regression methods.
\newblock \emph{arXiv:1909.05433 [stat.ME]}, 2019.

\bibitem[Szegedy et~al.(2017)Szegedy, Ioffe, Vanhoucke, and
  Alemi]{SzegedyIoVaAl17}
Christian Szegedy, Sergey Ioffe, Vincent Vanhoucke, and Alexander~A. Alemi.
\newblock Inception-v4, {I}nception-{R}es{N}et and the impact of residual
  connections on learning.
\newblock In \emph{Proceedings of the Thirty-Third AAAI Conference on
  Artificial Intelligence}, pages 4278–--4284, 2017.

\bibitem[Tan et~al.(2019)Tan, Pang, and Le]{TanPaLe19}
Mingxing Tan, Ruoming Pang, and Quoc~V Le.
\newblock Efficient{D}et: Scalable and efficient object detection.
\newblock \emph{arXiv:1911.09070 [cs.CV]}, 2019.

\bibitem[Tewari and Bartlett(2007)]{TewariBa07}
A.~Tewari and P.~L. Bartlett.
\newblock On the consistency of multiclass classification methods.
\newblock \emph{Journal of Machine Learning Research}, 8:\penalty0 1007--1025,
  2007.

\bibitem[Vovk(2012)]{Vovk12}
Vladimir Vovk.
\newblock Conditional validity of inductive conformal predictors.
\newblock In \emph{Proceedings of the Asian Conference on Machine Learning},
  volume~25 of \emph{Proceedings of Machine Learning Research}, pages 475--490,
  2012.
\newblock URL \url{http://proceedings.mlr.press/v25/vovk12.html}.

\bibitem[Vovk et~al.(2005)Vovk, Grammerman, and Shafer]{VovkGaSh05}
Vladimir Vovk, Alexander Grammerman, and Glenn Shafer.
\newblock \emph{Algorithmic Learning in a Random World}.
\newblock Springer, 2005.

\bibitem[Wainwright and Jordan(2008)]{WainwrightJo08}
Martin~J.\ Wainwright and Michael~I.\ Jordan.
\newblock Graphical models, exponential families, and variational inference.
\newblock \emph{Foundations and Trends in Machine Learning}, 1\penalty0
  (1--2):\penalty0 1--305, 2008.

\bibitem[Zhang and Zhou(2006)]{ZhangZh06}
Min-Ling Zhang and Zhi-Hua Zhou.
\newblock Multilabel neural networks with applications to functional genomics
  and text categorization.
\newblock \emph{IEEE transactions on Knowledge and Data Engineering},
  18\penalty0 (10):\penalty0 1338--1351, 2006.

\bibitem[Zhang(2004)]{Zhang04a}
T.~Zhang.
\newblock Statistical analysis of some multi-category large margin
  classification methods.
\newblock \emph{Journal of Machine Learning Research}, 5:\penalty0 1225--1251,
  2004.

\end{thebibliography}

\appendix

\section{Technical proofs and appendices}


\subsection{Proof of Theorem~\ref{thm:oracle}}
\label{sec:proof-oracle}

Our proof adapts arguments similar to those that \citet{SesiaCa19}
use in the regression setting, repurposing and modifying them
for classification.
We have that $\ltwopxs{\what{s} - s} \cp 0$ and
$\ltwopxs{\what{q}_\alpha^\sigma - q^\sigma_\alpha} \cp 0$.
We additionally have that
\begin{equation}
  \what{Q} \defeq
  Q_{1-\alpha}(E^\sigma, \mc{I}_3)
  \cp 0.
  \label{eqn:quantile-error-zeros}
\end{equation}
This is item~(ii) in the proof of Theorem 1 of
\citet{SesiaCa19} (see Appendix~A of their paper), which proves the
convergence~\eqref{eqn:quantile-error-zeros} of the marginal
quantile of the error precisely when $\what{q}_\alpha^\sigma$ and
$\what{s}$ are $L^2$-consistent in probability (when
the randomized quantity
$s(x, Y) + \sigma Z$ has a density), as in our
Assumption~\ref{assumption:consistency-scores}.

Recalling that $\what{s}, \what{q}$ tacitly depend on the sample
size $n$, 
let $\epsilon > 0$ be otherwise arbitrary, and define the
sets
\begin{equation*}
  B_n \defeq \left\{
  x \in \mc{X} \mid \linfs{\what{s}(x, \cdot) - s(x, \cdot)} > \epsilon^2
  ~ \mbox{or} ~
  |\what{q}^\sigma_\alpha(x) - q^\sigma_\alpha(x)| > \epsilon^2
  \right\}.
\end{equation*}
Then $B_n \subset \mc{X}$ is measurable, and by Markov's inequality,
\begin{equation*}
  P_X(B_n)
  \le \frac{\ltwopxs{\what{s} - s}}{\epsilon}
  + \frac{\ltwopxs{\what{q}^\sigma_\alpha - q^\sigma_\alpha}}{\epsilon},
\end{equation*}
so
\begin{equation}
  \label{eqn:measure-of-bad-sets-zeros}
  \P(P_X(B_n) \ge \epsilon)
  \le \P(\ltwopxs{\what{s} - s} > \epsilon^2) + \P(\ltwopxs{
    \what{q}^\sigma_\alpha - q^\sigma_\alpha} > \epsilon^2) \to 0.
\end{equation}
Thus, the measure of the sets $B_n$ tends to zero in probability, i.e.,
$P_X(B_n) \cp 0$.

Now recall the shorthand~\eqref{eqn:quantile-error-zeros} that $\what{Q} =
Q_{1-\alpha}(E^\sigma, \mc{I}_3)$.  Let us consider the event that
$\what{C}^\sigma_{1-\alpha}(x, z) \neq C^\sigma_{1-\alpha}(x, z)$.  If this
is the case, then we must have one of
\begin{equation}
  \label{eqn:failure-confidence-equality}
  \begin{split}
    A_1(x, k, z)
    & \defeq \left\{
    \what{s}(x, k) + \sigma z \ge \what{q}^\sigma_\alpha(x)
    - \what{Q}
    ~ \mbox{and} ~ s(x, k) + \sigma z < q^\sigma_\alpha(x)
    \right\} ~~ \mbox{or} \\
    A_2(x, k, z)
    & \defeq \left\{
    \what{s}(x, k) + \sigma z < \what{q}^\sigma_\alpha(x)
    - \what{Q}
    ~ \mbox{and} ~ s(x, k) + \sigma z \ge q^\sigma_\alpha(x)
    \right\}.
  \end{split}
\end{equation}
We show that the probability of the set $A_1$ is small; showing that
the probability of set $A_2$ is small is similar.
Using the convergence~\eqref{eqn:quantile-error-zeros}, let us assume
that $|\what{Q}| \le \epsilon$, and
suppose that $x \not\in B_n$. Then for $A_1$ to occur, we must have
both
$s(x, k) + \epsilon + \sigma z \ge q_\alpha^\sigma(x) - 2 \epsilon$
and
$s(x, k) + \sigma z < q_\alpha^\sigma(x)$,
or
\begin{equation*}
  q_\alpha^\sigma(x) -s(x, k) -3 \epsilon
  \le \sigma z < q_\alpha^\sigma(x) - s(x, k).
\end{equation*}
As $Z$ has a bounded density, 
we have
$\limsup_{\epsilon \to 0}
\sup_{a \in \R} P_Z(a \le \sigma Z \le a + 3 \epsilon) = 0$,
or (with some notational abuse)
$\limsup_{\epsilon \to 0} \sup_{x \not \in B_n} P_Z(A_1(x, k, Z)) = 0$.

Now, let $\mc{F}_n = \sigma(\{X_i\}_{i=1}^n, \{Y_i\}_{i=1}^n,
\{Z_i\}_{i=1}^n)$ be the $\sigma$-field of the observed sample.  Then by the
preceding derivation (\emph{mutatis mutandis} for the set $A_2$ in
definition~\eqref{eqn:failure-confidence-equality}) for any $\eta > 0$,
there is an $\epsilon > 0$ (in the definition of $B_n$) such that
\begin{align*}
  \lefteqn{\sup_{x \in \mc{X}}
    \P\left(\what{C}_{1-\alpha}^\sigma (x, Z_{n+1}) \neq
    C_{1-\alpha}^\sigma(x, Z_{n+1}) \mid \mc{F}_n\right)
    \indic{x \not \in B_n} \indic{|\what{Q}| \le \epsilon}} \\
  & \qquad \le
  \sup_{x \not \in B_n}
  \sum_{k = 1}^K \P\left(A_1(x, k, Z_{n+1}) ~ \mbox{or}~
  A_2(x, k, Z_{n+1}) \mid \mc{F}_n\right)
  \indic{|\what{Q}| \le \epsilon} \le \eta.
\end{align*}
In particular, by integrating
the preceding inequality, we have
\begin{equation*}
  \P\left(\what{C}_{1-\alpha}^\sigma(X_{n+1}, Z_{n+1})
  \neq \what{C}(X_{n+1}, Z_{n+1}), X_{n+1} \not \in B_n,
  |\what{Q}| \le \epsilon \right) \le \eta.
\end{equation*}
As $\P(|\what{Q}| \le \epsilon) \to 1$ and $\P(X_{n+1} \not\in B_n) \to 1$
by the convergence guarantees~\eqref{eqn:quantile-error-zeros} and
\eqref{eqn:measure-of-bad-sets-zeros}, we have the theorem.

\subsection{Proof of Proposition~\ref{proposition:super-oracle}}
\label{sec:proof-super-oracle}

We wish to show that $\lim_{\sigma \to 0}
\P\left( C^\sigma_{1-\alpha}(X,Z) \subset
C_{1-\alpha}(X) \right) = 1$, which,
by a union bound over $k \in [K]$, is equivalent to showing
\begin{equation*}
  \P \left(s(X,k) + \sigma Z \geq q^\sigma_\alpha(X), s(X,k) < q_\alpha(X) \right)
  \underset{\sigma \to 0}{\longrightarrow} 0
\end{equation*}
for all $k \in [K]$.
Fix $k \in [K]$, and define the events
\begin{equation*}
  A \defeq \left\{ s(X,k) < q_\alpha(X) \right\}
\end{equation*} 
and 
\begin{equation*}
  B^\sigma \defeq \{ s(X,k) + \sigma Z \geq q^\sigma_\alpha(X) \}.
\end{equation*}
Now, for any $\delta > 0$, consider $A_\delta \defeq \{ s(X,k) \leq q_\alpha(X) - \delta \}$.
On the event $A_\delta \cap B^\sigma$, it must hold that
\begin{equation*}
  \delta \leq q_\alpha(X) - q_\alpha^\sigma(X) + \sigma Z.
\end{equation*}
The following lemma---whose proof we defer to
Section~\ref{subsec:proofs-lemmas-quantiles}---shows that the latter can
only occur with small probability.
\begin{lemma}
  \label{convergence-quantile-function}
  With probability 1 over $X$, the quantile function satisfies
  \begin{equation*}
    \liminf_{\sigma \to 0} q_\alpha^\sigma(X) \geq q_\alpha(X),
  \end{equation*}
  and hence $\limsup_{\sigma \to 0} \left\{ q_\alpha(X) - q_\alpha^\sigma(X)
  + \sigma Z \right\} \leq 0$ almost surely.
\end{lemma}

Lemma~\ref{convergence-quantile-function} implies that
\begin{equation*}
\P\left( q_\alpha(X) - q_\alpha^\sigma(X) + \sigma Z \geq \delta \right) \underset{\sigma \to 0}{\longrightarrow} 0,
\end{equation*}
which, in turn, shows that, for every fixed $\delta > 0$,
\begin{equation*}
\P \left( A_\delta \cap B^\sigma \right) \underset{\sigma \to 0}{\longrightarrow} 0.
\end{equation*}

To conclude the proof, fix $\epsilon > 0$. 
The event $A_\delta$ increases to $A$ when $\delta \to 0$,  so there exists $\delta>0$ so that
\begin{equation*}
\P\left( A \setminus A_\delta \right) \leq \epsilon.
\end{equation*}
Finally,
\begin{align*}
\limsup_{\sigma \to 0} \, &\P\left( A \cap B^\sigma \right) 
\leq \P\left( A \setminus A_\delta \right) + \limsup_{\sigma \to 0}  \P\left( A_\delta \cap B^\sigma \right) \leq \epsilon,
\end{align*}
as $\limsup_{\sigma \to 0} \P\left( A_\delta \cap B^\sigma \right) = 0$. 
We conclude the proof by sending $\epsilon \to 0$.

\subsubsection{Proof of Lemma~\ref{convergence-quantile-function}}
\label{subsec:proofs-lemmas-quantiles}

Fix $x \in \mc{X}$. Let $F_{\sigma,x}$ and $F_x$ be the respective
cumulative distribution functions of $s^{\sigma}(x,Y,Z)$ and $s(x,Y)$
conditionally on $X=x$, and define the (left-continuous) inverse CDFs
\begin{align*}
F_{\sigma,x}^{-1}(u)=\inf\{t \in \R : u \leq F_{\sigma,x}(t)\}
~~~ \mbox{and} ~~~
F^{-1}_x(u)=\inf\{t \in \R : u \leq F_x(t)\}.
\end{align*}
We use a standard lemma about the convergence of inverse CDFs, though
for lack of a proper reference, we
include a proof in
Section~\ref{subsec:proof-lemma-inverse_df_convergence}.

\begin{lemma}\label{lem:inverse_df_convergence}
  Let $(F_n)_{n \geq 1}$ be a sequence of cumulative distribution functions
  converging weakly to $F$, with inverses
  $F_n^{-1}$ and $F^{-1}$.  Then for each $u \in
  (0,1)$,
  \begin{equation}\label{inverse_df}
    F^{-1}(u) \leq \liminf_{n \to \infty} F_{n}^{-1}(u) \leq \limsup_{n \to
      \infty} F_{n}^{-1}(u) \leq F^{-1}(u+).
  \end{equation}
\end{lemma}
\noindent
As $F_{\sigma,x}$ converges weakly to $F_x$ as $\sigma \to 0$,
Lemma~\ref{lem:inverse_df_convergence} implies that
\begin{equation*}
F^{-1}_x(\alpha) \leq \liminf_{\sigma \to 0} F_{\sigma,x}^{-1}(\alpha).
\end{equation*}
But observe that $q^{\sigma}_{\alpha}(x)=F_{\sigma,x}^{-1}(\alpha)$ and $q_{\alpha}(x)=F^{-1}_{x}(\alpha)$, so that we have the desired result
$q_{\alpha}(x) \leq \liminf_{\sigma \to 0} q^{\sigma}_{\alpha}(x)$.

\subsubsection{Proof of Lemma~\ref{lem:inverse_df_convergence}}
 \label{subsec:proof-lemma-inverse_df_convergence}
We prove only the first inequality, as the last inequality is similar.
Fix $u \in (0,1)$ and $\epsilon >0$. $F$ is right-continuous and
non-decreasing, so its set of continuity points is dense in $\R$; thus,
there exists a continuity point $w$ of $F$ such that $w < F^{-1}(u)
\leq w + \epsilon $.

Since $w < F^{-1}(u)$, it must hold that $F(w) < u$, by definition of
$F^{-1}$.  As $w$ is a continuity point of $F$, $\lim_{n \to \infty}
F_n(w) = F(w) < u$, which means that $F_n(w) < u$ for large enough $n$, or
equivalently, that $w < F_n^{-1}(u)$.
We can thus conclude that
\begin{equation*}
\liminf F_n^{-1}(u) \geq w \geq F^{-1}(u) - \epsilon.
\end{equation*}
Taking $\epsilon \to 0$ proves the first inequality.


\subsection{Efficient computation of maximal marginals
  for condition~\eqref{eqn:efficient-inner-outer-vecs}}
\label{sec:computation-of-min-max-tree}

\newcommand{\Tdown}{\mc{T}_{\textup{down}}}
\newcommand{\Tup}{\mc{T}_{\textup{up}}}

We describe a more or less standard dynamic programming approach to
efficiently compute the maximum marginals~\eqref{eqn:max-marginals}
(i.e.\ maximal values of a tree-structured
score $s : \mc{Y} = \{-1, 1\}^K \to \R$), referring to standard references on
max-product message passing~\cite{Dawid92a, WainwrightJo08, KollerFr09} for
more. Let $\mc{T}
= ([K], E)$ be a tree with nodes $[K]$ and undirected edges $E$, though we
also let $\textup{edges}(\mc{T})$ and $\textup{nodes}(\mc{T})$ denote the
edges and nodes of the tree $\mc{T}$. Assume
\begin{equation*}
  s(y) = \sum_{k = 1}^K \varphi_k(y_k) + \sum_{e = (k, l) \in E} \psi_e(y_k, y_l).
\end{equation*}
To compute the \emph{maximum marginals}
$\max_{y \in \mc{Y}} \{s(y) \mid y_k = \hat{y}_k\}$, we perform two message
passing steps on the tree: an upward and downward pass. Choose a node $r$
arbitrarily to be the root of $\mc{T}$, and let $\Tdown$ be the directed tree
whose edges are $E$ with $r$ as its root (which is evidently unique);
let $\Tup$ be the directed tree with all edges reversed from $\Tdown$.

A maximum marginal message passing algorithm then computes a single
downward and a single upward pass through each tree,
each in topological order of the tree. The downard messages
$m_{l \to k} : \{-1, 1\} \to \R$ are defined for $\hat{y} \in \{-1, 1\}$ by
\begin{equation*}
  m_{l \to k}(\hat{y}) = \max_{y_l \in \{-1, 1\}} \left\{
  \varphi_l(y_l) + \psi_{(l,k)}(y_l, \hat{y})
  + \sum_{i : (i \to l) \in \textup{edges}(\Tdown)} m_{i \to l}(y_l)
  \right\},
\end{equation*}
while the upward pass is defined similarly except that 
$\Tup$ replaces $\Tdown$.
After a single downward and upward pass through the tree,
which takes time $O(K)$,
the invariant of message passing on the tree~\cite[Ch.~13.3]{KollerFr09}
then guarantees that for each $k \in [K]$,
\begin{equation}
  \label{eqn:maxima-tree-scores}
  \max_{y \in \mc{Y}}
  \{s(y) \mid y_k = \hat{y}_k\}
  = \varphi_k(\hat{y}_k) + \sum_{e = (l, k) \in \mc{T}}
  m_{l \to k} (\hat{y}_k),
\end{equation}
where we note that there exists exactly one message to $k$ (whether from the
downward or upward pass) from each node $l$ neighboring $k$ in $\mc{T}$.

We can evidently then compute all of these maximal values (the sets
$\mc{S}_{\pm}$ in Eq.~\eqref{eqn:max-marginals}) simultaneously in time of
the order of the number of edges in the tree $\mc{T}$, or $O(K)$, as each
message $m_{l \to k}$ can appear in at most one of the
maxima~\eqref{eqn:maxima-tree-scores}, and there are $2K$ messages.

\subsection{Concentration of coverage quantities}
\label{sec:localization-vc-argument}

We sketch a derivation of
inequality~\eqref{eqn:conditional-probs-close}; see also \cite[Theorem
  5]{BarberCaRaTi19a} for related arguments.

We begin with a technical lemma that is the basis for our result.  In the
lemma, we abuse notation briefly, and let $\mc{F} \subset \mc{Z} \to \{0,
1\}$ be a collection of functions with VC-dimension $d$. We define $Pf =
\int f(z) dP(z)$ and $P_n f = \frac{1}{n} \sum_{i=1}^n f(Z_i)$, as is
standard.
\begin{lemma}[Relative concentration bounds,
    e.g.~\cite{BoucheronBoLu05}, Theorem 5.1]
  \label{lemma:relative-concentration}
  Let $\VC(\mc{F}) \le d$. There is a numerical constant $C$ such
  that for any $t > 0$, with probability at least $1 - e^{-t}$,
  \begin{equation*}
    \sup_{f \in \mc{F}}
    \left\{|Pf - P_n f| - C \sqrt{\min\{Pf, P_n f\} \frac{d \log n + t}{n}
      }\right\} \le C \frac{d \log n + t}{n}.
  \end{equation*}
\end{lemma}
\begin{proof}
  By~\citet[Thm.~5.1]{BoucheronBoLu05} for $t > 0$, with probability at
  least $1 - e^{-t}$ we have
  \begin{equation*}
    \sup_{f \in \mc{F}} \frac{P f - P_n f}{\sqrt{Pf}}
    \le C \sqrt{\frac{d \log n + t}{n}}
    ~~ \mbox{and} ~~
    \sup_{f \in \mc{F}} \frac{P_n f - P f}{\sqrt{P_n f}}
    \le C \sqrt{\frac{d \log n + t}{n}}.
  \end{equation*}
  Let $\varepsilon_n = C \sqrt{(d \log n + t) / n}$ for shorthand. Then
  the second inequality is equivalent to the statement that
  for all $f \in \mc{F}$,
  \begin{equation*}
    P_n f - P f 
    \le \frac{1}{2 \eta} P_n f + \frac{\eta}{2} \varepsilon_n^2
    ~~ \mbox{for~all~} \eta > 0.
  \end{equation*}
  Rearranging the preceding display, we have
  \begin{equation*}
    \left(1 - \frac{1}{2 \eta}\right) (P_n f - P f)
    \le \frac{1}{2 \eta} P f + \frac{\eta}{2} \varepsilon_n^2.
  \end{equation*}
  If $\sqrt{Pf} \ge \varepsilon_n$, we set
  $\eta = \sqrt{Pf} / \varepsilon_n$ and obtain
  $\half (P_n f - Pf) \le \sqrt{P f} \varepsilon_n$, while if
  $\sqrt{Pf} < \varepsilon_n$, then setting $\eta = 1$ yields
  $\half (P_n f - P f) \le \half (Pf + \varepsilon_n^2)
  \le \half (\sqrt{Pf} \varepsilon_n + \varepsilon_n^2)$. In either case,
  \begin{equation*}
    P_n f - P f \le C \left[ \sqrt{P f \frac{d \log n + t}{n}}
      + \frac{d \log n + t}{n} \right].
  \end{equation*}
  A symmetric argument replacing each $P_n$ with $P$ (and vice
  versa) gives the lemma.
\end{proof}

We can now demonstrate
inequality~\eqref{eqn:conditional-probs-close}.  Let $V \subset \R^d$ and
$\mc{V} \defeq \{\{x \in \R^d \mid v^T x \le 0\}\}_{v \in V}$
the collection of halfspaces it induces.
The collection $\mc{S} =
\{S_{v,a,b}\}_{v \in V, a < b}$ of slabs has VC-dimension $\VC(\mc{S}) \le
O(1) \VC(\mc{V})$. Let $f : \mc{X} \times \mc{Y} \to \R$ and $c : \mc{S} \to \R$
be arbitrary functions. If for $S \subset \mc{X}$ we define $S^+
\defeq \{(x, y) \mid x \in S, f(x, y) \ge c(S)\}$
and the collection
$\mc{S}^+ \defeq \{S^+ \mid S \in \mc{S}\}$, then $\VC(\mc{S}^+) \le
\VC(\mc{S}) + 1$~\cite[Lemma~5]{BarberCaRaTi19a}.  As a consequence,
for the conformal sets inequality~\eqref{eqn:conditional-probs-close}
specifies, for any $t > 0$ we have with
probability at least $1 - e^{-t}$ that
\begin{align*}
  \lefteqn{
    \left|P_n(Y \in \what{C}(X), X \in S_{v,a,b}) - P(Y \in \what{C}(X),
    X \in S_{v, a, b})\right|}  \\
  & \le O(1) \left[\sqrt{
      \min\{P(Y \in \what{C}(X), X \in S_{v,a,b}),
      P_n(Y \in \what{C}(X), X \in S_{v, a, b})\}
      \frac{\VC(\mc{V}) \log n + t}{n}}\right]\\
      & \hspace{4in} +O(1)\left[ \frac{\VC(\mc{V}) \log n + t}{n}\right]
\end{align*}
simultaneously for all $v \in V, a < b \in \R$, and similarly
\begin{align*}
  \lefteqn{\left|P_n(X \in S_{v,a,b}) - P(X \in S_{v,a,b})\right|} \\
  & \le O(1) \left[\sqrt{\min\{P(X \in S_{v,a,b}),
      P_n(X \in S_{v,a,b})\} \frac{\VC(\mc{V}) \log n + t}{n}}
    + \frac{\VC(\mc{V}) \log n + t}{n} \right].
\end{align*}

Now, we use the following simple observation. For any $\varepsilon$ and
nonnegative $\alpha, \beta, \gamma$ with $\alpha \le \gamma$ and
$2|\varepsilon| \le \gamma$,
\begin{equation*}
  \left|\frac{\alpha}{\gamma + \varepsilon} - \frac{\beta}{\gamma}\right|
  \le \frac{|\alpha - \beta|}{\gamma}
  + \frac{|\alpha \epsilon|}{\gamma^2 + \gamma \epsilon}
  \le \frac{|\alpha - \beta|}{\gamma}
  + \frac{2 |\epsilon|}{\gamma}.
\end{equation*}
Thus, as soon as $\delta \ge \frac{\VC(\mc{V}) \log n + t}{n}$,
we have with probability at least $1 - e^{-t}$ that
\begin{align*}
  \lefteqn{|P_n(Y \in \what{C}(X) \mid X \in S_{v,a,b})
    - P(Y \in \what{C}(X) \mid X \in S_{v,a,b})|} \\
  & = \left|\frac{P_n(Y \in \what{C}(X), X \in S_{v, a, b})}{
    P_n(X \in S_{v,a,b})}
  - \frac{P(Y \in \what{C}(X), X \in S_{v,a,b})}{P(X \in S_{v,a,b})}\right| \\
  & \le O(1)
  \left[\sqrt{\frac{\VC(\mc{V}) \log n + t}{\delta n}}
    + \frac{\VC(\mc{V}) \log n + t}{\delta n}\right]
\end{align*}
simultaneously for all $v \in V, a < b \in \R$,
where we substitute $\gamma = P(X \in S_{v,a,b})$,
$\alpha = P(Y \in \what{C}(X), X \in S_{v,a,b})$,
$\beta = P_n(Y \in \what{C}(X), X \in S_{v,a,b})$, and
$\varepsilon = (P_n - P)(X \in S_{v,a,b})$.
Note that if $\frac{\VC(\mc{V}) \log n + t}{\delta n} \ge 1$, the bound
is vacuous in any case.


\section{Additional Figures}



\begin{figure}[h]
 \centering
  \begin{overpic}[scale=0.72]{
     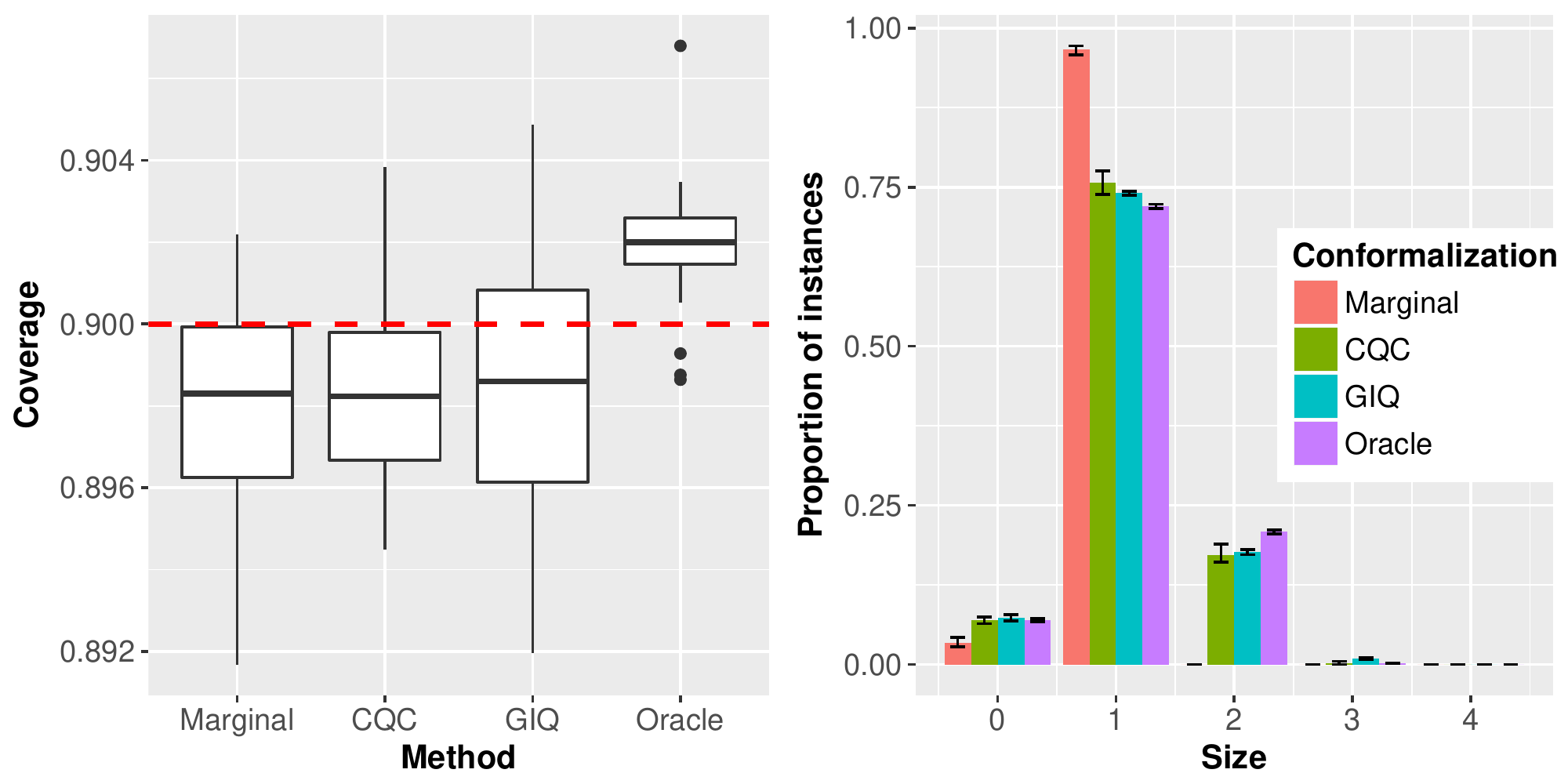}
    \put(0,10){
      \tikz{\path[draw=white, fill=white] (0, 0) rectangle (.4cm, 6cm)}
    }
    \put(0,20){\rotatebox{90}{
        \small $\P(Y \in \hat{C}(X))$}
    }
    \put(15, 0){
      \tikz{\path[draw=white, fill=white] (0, 0) rectangle (4cm, .4cm)}
    }
    \put(70, 0){
      \tikz{\path[draw=white, fill=white] (0, 0) rectangle (4cm, .5cm)}
    }
    \put(50, 13){
      \tikz{\path[draw=white, fill=white] (0, 0) rectangle (.4cm, 5cm)}
    }
	\put(50,20){\rotatebox{90}{
        \small $\P(|\hat{C}(X)| = t)$}
    }    
    \put(78, 1){
        \small $t$}  
  \end{overpic}
  \caption{Marginal coverage and distribution of the confidence set size in the multiclass simulation~\eqref{eqn:dataset-multiclass-simulated} over $M=20$ trials.}
  \label{fig:multiclass-averagesize-coverage}
\end{figure}

\begin{figure}[h]
 \centering
  \begin{overpic}[
  				scale=0.6]{%
     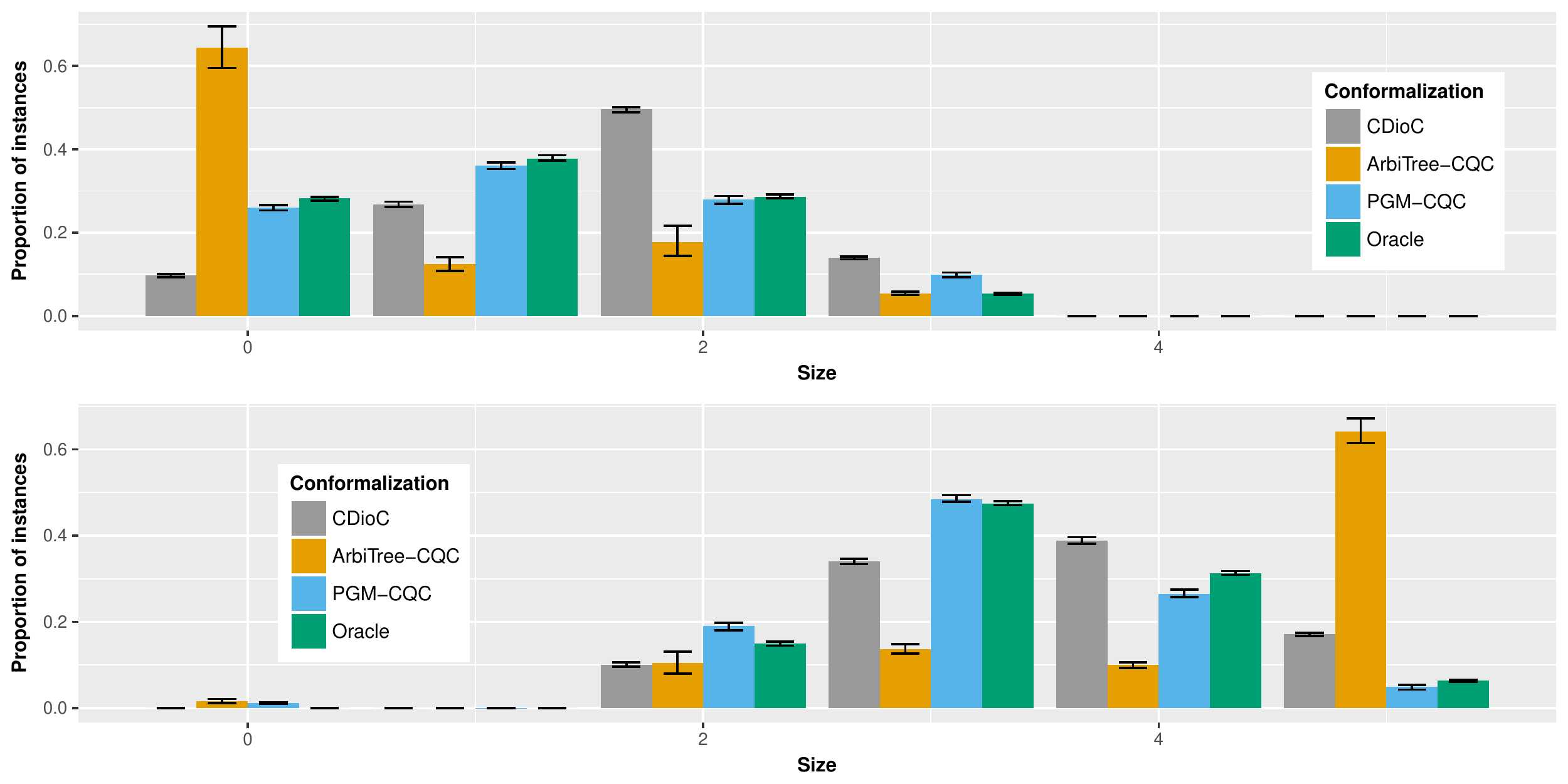}
    \put(-1,30){
      \tikz{\path[draw=white, fill=white] (0, 0) rectangle (.3cm, 6cm)}
    }
    \put(-1,5){
      \tikz{\path[draw=white, fill=white] (0, 0) rectangle (.3cm, 6cm)}
    }
    \put(-1,5){\rotatebox{90}{
        \small $\P( \left|\yout(X)\right| = t)$}
    }
    \put(-1,31){\rotatebox{90}{
        \small $\P( \left|\yin(X)\right| = t)$}
    }   
    \put(50, 0){
      \tikz{\path[draw=white, fill=white] (0, 0) rectangle (4cm, .3cm)}
    }        
    \put(50, 25){
      \tikz{\path[draw=white, fill=white] (0, 0) rectangle (4cm, .3cm)}
    }  
    \put(43,0){
        \small Outer set size $t$
    }    
        \put(43,25){
        \small Inner set size $t$
    }    

  \end{overpic}
  \caption{Simulated multilabel experiment with label
    distribution~\eqref{eqn:multilabel-logreg}.  Methods are the true oracle
    confidence set; the conformalized direct inner/outer method (CDioC),
    Alg.~\ref{alg:direct-inner-outer}; and tree-based methods with implicit
    confidence sets $\Cimplicit$ or explicit inner/outer sets $\Cinout$,
    labeled \scorelesstreeshort and \pgmtreeshort~(see description in
    Sec.~\ref{subsec:experiments-multilabel-simulation}).
    Top: distribution of the inner sets $\yin$ sizes.
    Bottom: distribution of the outer sets $\yout$ sizes.
  }
  \label{fig:multilabel-simulation-size-distribution}
\end{figure}


\begin{figure}
 \centering
  \begin{overpic}[
  				scale=0.60]{%
     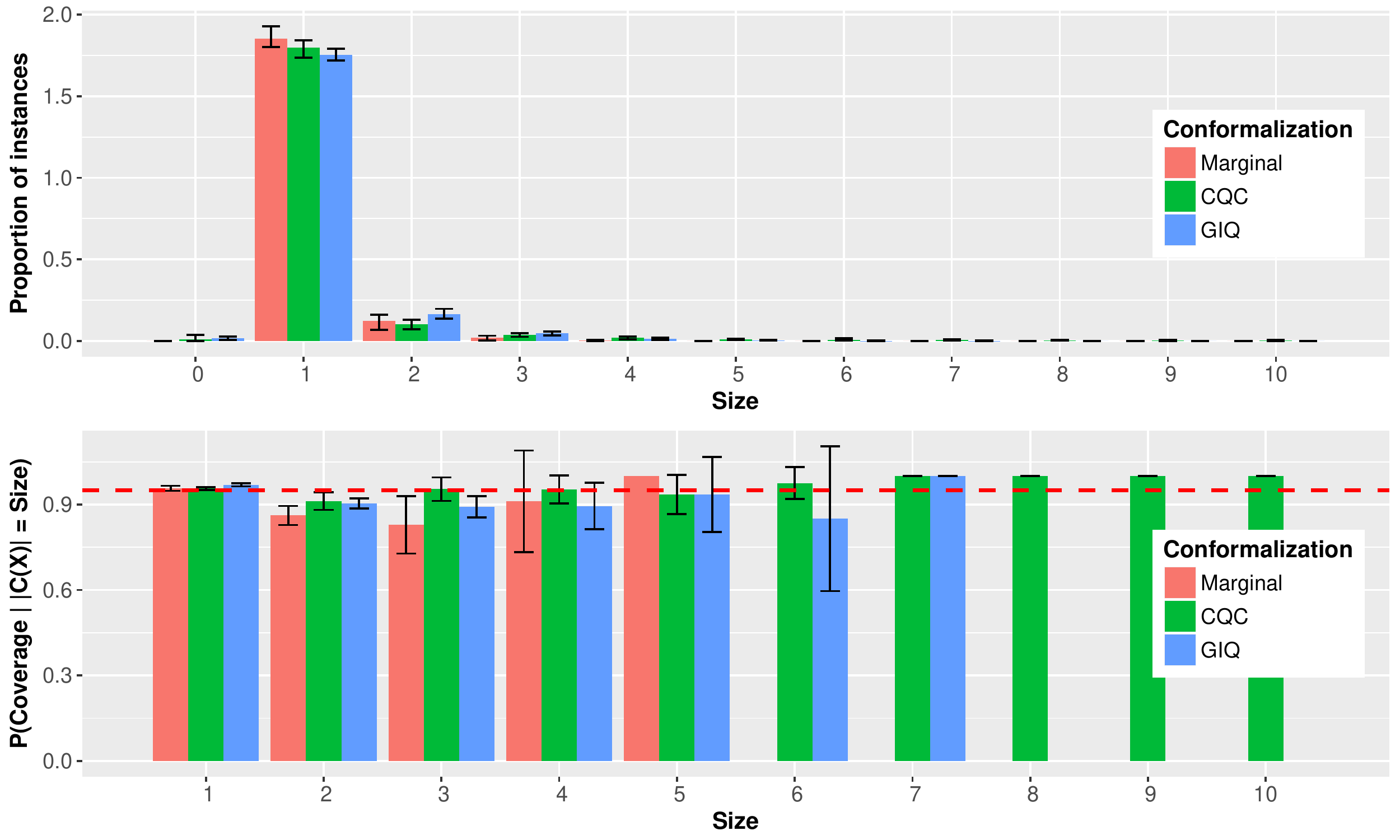}
    \put(-1,30){
      \tikz{\path[draw=white, fill=white] (0, 0) rectangle (.4cm, 6cm)}
    }
    \put(-1,0){
      \tikz{\path[draw=white, fill=white] (0, 0) rectangle (.4cm, 6cm)}
    }
    \put(-1,3){\rotatebox{90}{
        \small $\P( Y \in \what{C}(X) \mid |\what{C}(X)|=t)$}
    }
    \put(-1,40){\rotatebox{90}{
        \small $\P( |\what{C}(X)|=t)$}
    }   
    \put(50, 0){
      \tikz{\path[draw=white, fill=white] (0, 0) rectangle (4cm, .3cm)}
    }
    \put(50, 30){
      \tikz{\path[draw=white, fill=white] (0, 0) rectangle (4cm, .3cm)}
    }
    \put(40,0){
        \small Confidence set size $t$
    } 
        \put(40,30){
        \small Confidence set size $t$
    }
  \end{overpic}
  \caption{Results for CIFAR-10 dataset over $M=20$ trials. Methods are the marginal method (Marginal, procedure~\ref{eqn:marginal-method}), the CQC method (Alg.~\ref{alg:cqc}), and the GIQ method (Alg. 1,~\cite{RomanoSeCa20}).
    Top: Distribution of the confidence set size $|\what{C}(X)|$.
    Bottom: Probability of coverage conditioned on the size of the confidence set.
  }
  \label{fig:cifar-sizecdf}
\end{figure}

\begin{figure}[h]
 \centering
  \begin{overpic}[
  				scale=0.72]{%
     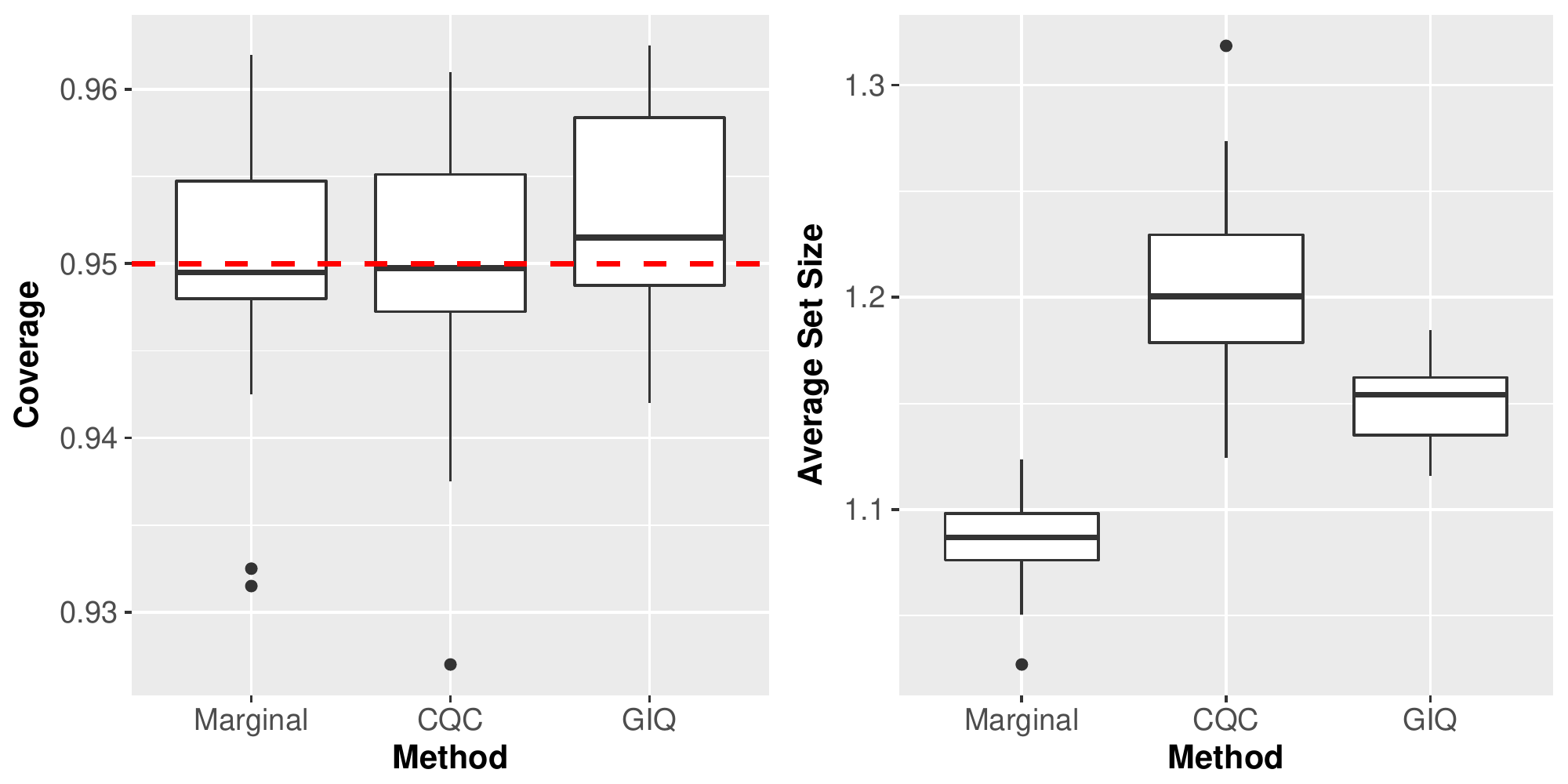}
    \put(0,10){
      \tikz{\path[draw=white, fill=white] (0, 0) rectangle (.4cm, 6cm)}
    }
    \put(0,20){\rotatebox{90}{
        \small $\P(Y \in \hat{C}(X))$}
    }
    \put(15, 0){
      \tikz{\path[draw=white, fill=white] (0, 0) rectangle (4cm, .4cm)}
    }
    \put(70, 0){
      \tikz{\path[draw=white, fill=white] (0, 0) rectangle (4cm, .4cm)}
    }
    \put(48, 13){
      \tikz{\path[draw=white, fill=white] (0, 0) rectangle (.6cm, 5cm)}
    }
    \put(49, 22){\rotatebox{90}{
        \small $\E\left|\hat{C}(X)\right|$}}
  \end{overpic}
  \caption{Marginal coverage and average confidence set size on
    CIFAR-10 over $M=20$ trials.}
  \label{fig:cifar10-averagesize-coverage}
\end{figure}

\begin{figure}
 \centering
  \begin{overpic}[
  				scale=0.7]{%
     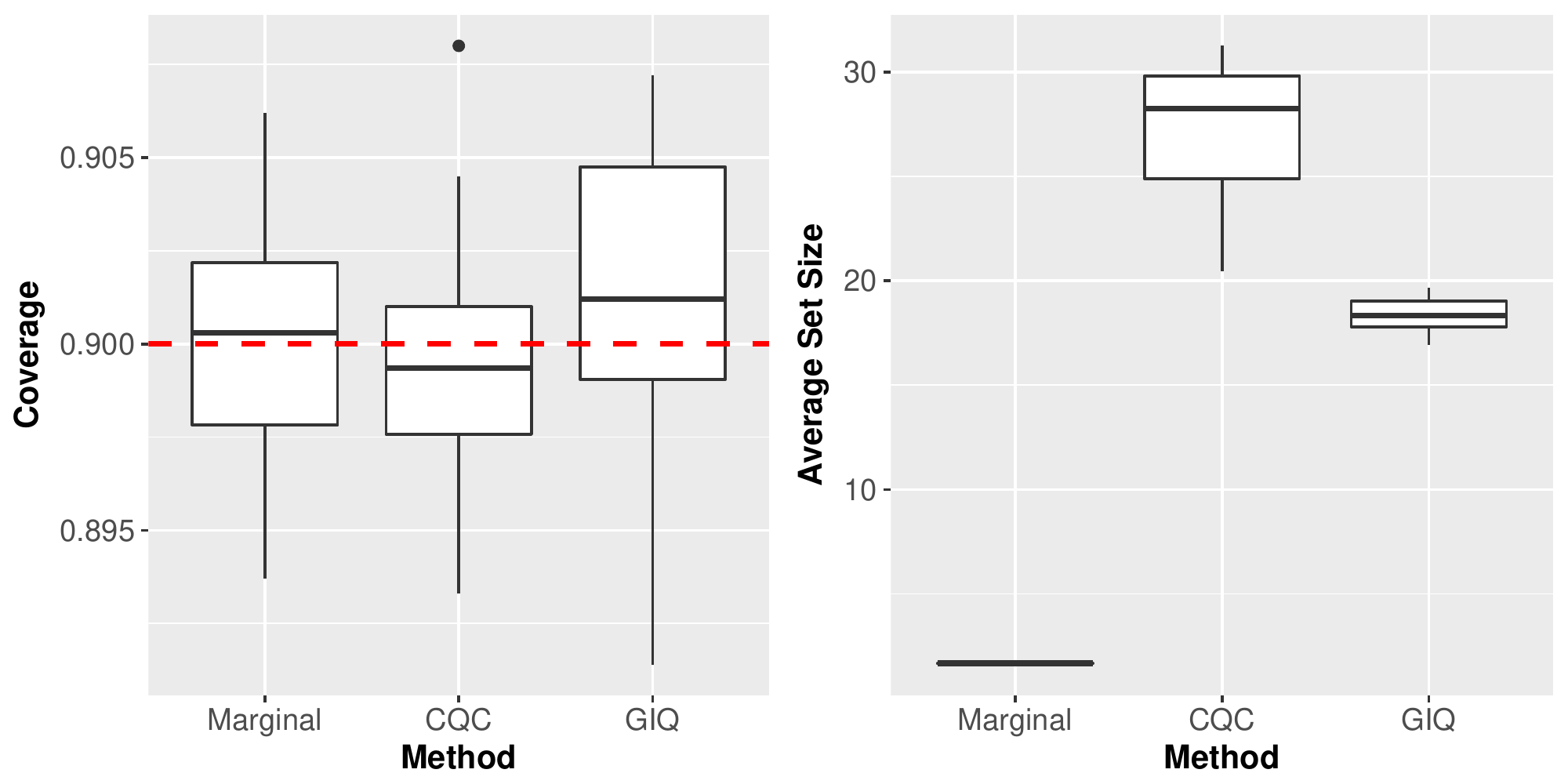}
    \put(0,10){
      \tikz{\path[draw=white, fill=white] (0, 0) rectangle (.4cm, 6cm)}
    }
    \put(0,20){\rotatebox{90}{
        \small $\P(Y \in \hat{C}(X))$}
    }
    \put(15, 0){
      \tikz{\path[draw=white, fill=white] (0, 0) rectangle (4cm, .4cm)}
    }
    \put(70, 0){
      \tikz{\path[draw=white, fill=white] (0, 0) rectangle (4cm, .4cm)}
    }
    \put(48, 13){
      \tikz{\path[draw=white, fill=white] (0, 0) rectangle (.6cm, 5cm)}
    }
    \put(49, 22){\rotatebox{90}{
        \small $\E\left|\hat{C}(X)\right|$}}
  \end{overpic}
  \caption{Marginal coverage and average confidence set size on
    ImageNet over $M=20$ trials.}
  \label{fig:imagenet-averagesize-coverage}
\end{figure}

\begin{figure}
    \begin{overpic}[
  				scale=0.6]{%
     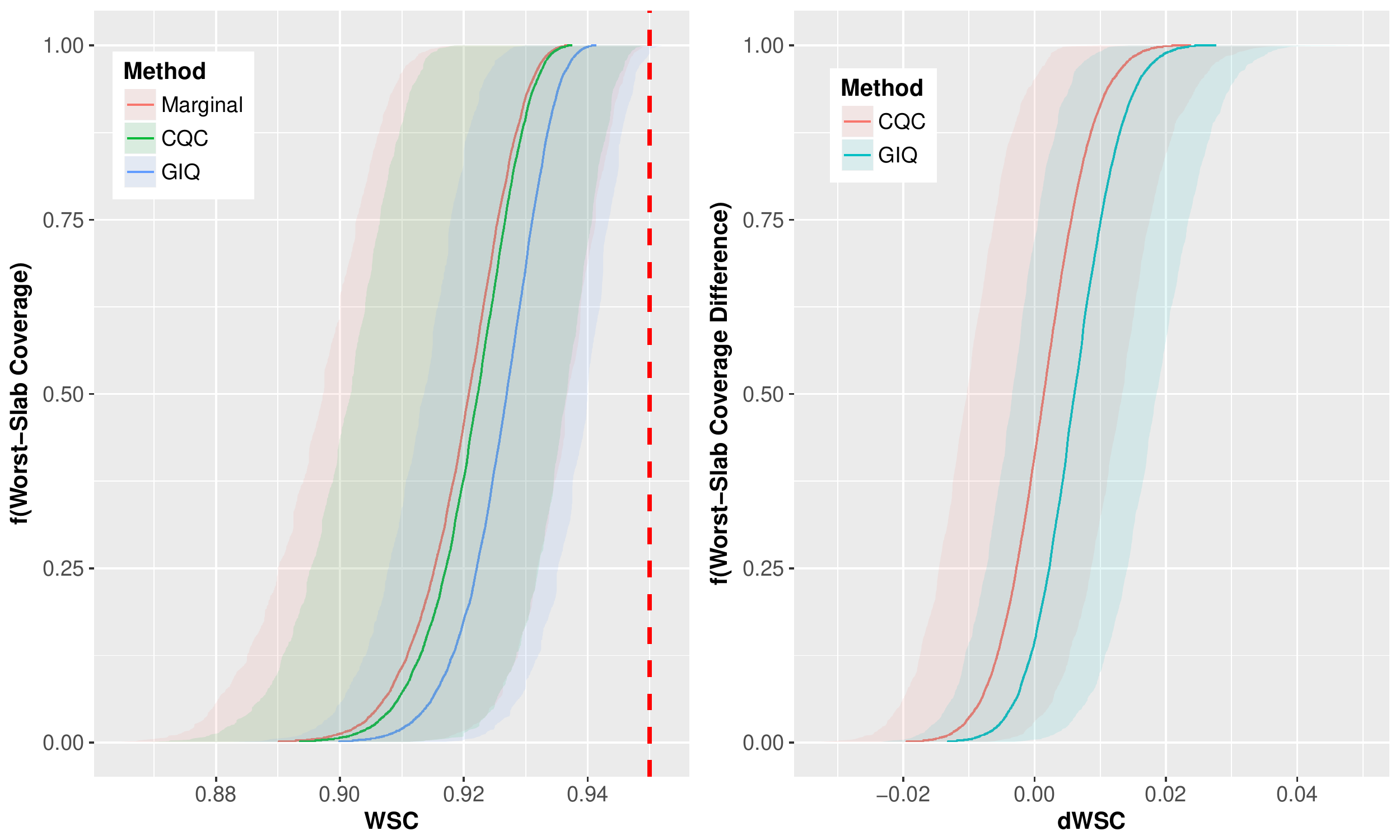}
    \put(-0.5,15){
      \tikz{\path[draw=white, fill=white] (0, 0) rectangle (.3cm, 6cm)}
    }
    \put(-1,12){\rotatebox{90}{
        \small $\P_v( \text{WSC}_n(\what{C}_{\text{Method}},v) \le t)$}
    }
    \put(15, 0){
      \tikz{\path[draw=white, fill=white] (0, 0) rectangle (4cm, .3cm)}
    }
        
    \put(60, 0){
      \tikz{\path[draw=white, fill=white] (0, 0) rectangle (4cm, .3cm)}
    }
    \put(49, 6){
      \tikz{\path[draw=white, fill=white] (0, 0) rectangle (.4cm, 7cm)}
    }
    \put(50, 6){\rotatebox{90}{
        \small $\P_v(\text{WSC}_n(\what{C}_{\textrm{Method}},v) -
    \text{WSC}_n(\what{C}_{\textrm{CDioC}},v) \le t) $}}
    
    \put(12, 0){
        \small Worst Coverage probability  $t$}
        
     \put(65, 0){
        \small Difference of coverage $t$}
    
  \end{overpic}
  
  \caption{Worst-slab coverage for CIFAR-10
    with $\delta=.2$ over $M = 1000$ draws
    $v \simiid \uniform(\sphere^{d-1})$.
    The dotted line is the desired (marginal) coverage.
    Left: distribution of worst-slab coverage.
    Right: distribution of the
    coverage difference $\text{WSC}_n(\what{C}_{\textrm{CQC}},v) -
    \text{WSC}_n(\what{C}_{\textrm{Marginal}},v)$.}
\label{fig:cifar10-worst-slab}
\end{figure}

  
\begin{figure}
 \centering
  \begin{overpic}[
  				scale=0.6]{%
     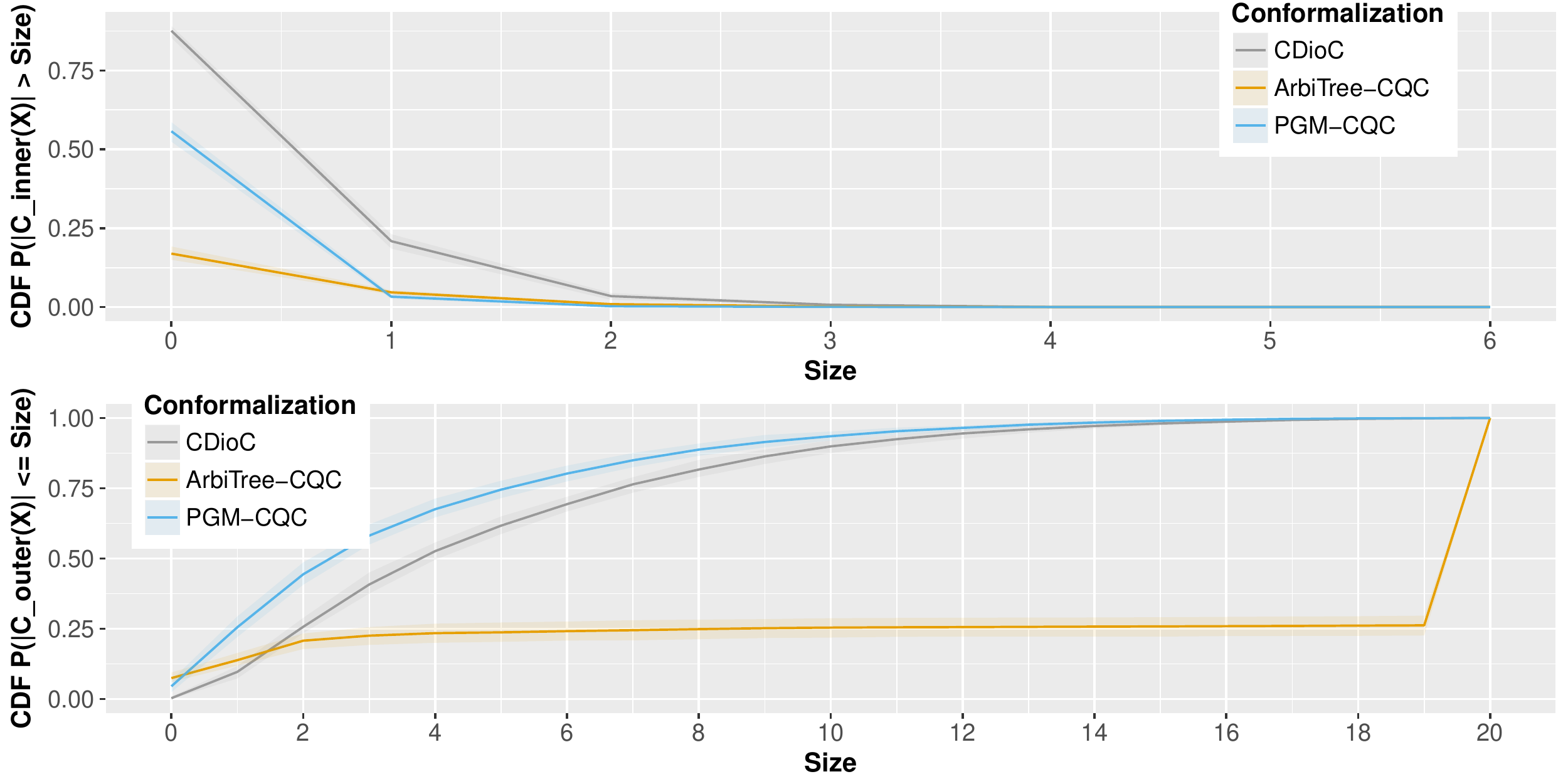}
    \put(-1,30){
      \tikz{\path[draw=white, fill=white] (0, 0) rectangle (.4cm, 6cm)}
    }
    \put(-1,0){
      \tikz{\path[draw=white, fill=white] (0, 0) rectangle (.4cm, 6cm)}
    }
    \put(0,5){\rotatebox{90}{
        \small $\P( \left|\yout(X)\right| \le t)$}
    }
    \put(0,31){\rotatebox{90}{
        \small $\P( \left|\yin(X)\right| > t)$}
    }   
    \put(50, 0){
      \tikz{\path[draw=white, fill=white] (0, 0) rectangle (4cm, .3cm)}
    }
    \put(50, 25){
      \tikz{\path[draw=white, fill=white] (0, 0) rectangle (4cm, .3cm)}
    }
    \put(43,0){
        \small Outer set size $t$
    } 
        \put(43,25){
        \small Inner set size $t$
    }

  \end{overpic}
  \caption{Results for PASCAL VOC dataset~\cite{EveringhamVaWiWiZi12} over $M=20$ trials. Methods are the conformalized direct inner/outer method (CDioC),
    Alg.~\ref{alg:direct-inner-outer}; and tree-based methods with implicit
    confidence sets $\Cimplicit$ or explicit inner/outer sets $\Cinout$,
    labeled \scorelesstreeshort and \pgmtreeshort (see description in
    Sec.~\ref{subsec:experiments-multilabel-simulation}).
    Top: $1-$ empirical c.d.f of sizes for the inner sets $\yin$ .
    Bottom: empirical c.d.f of sizes for the outer sets $\yout$.
  }
  \label{fig:multilabel-voc-size-distribution}
\end{figure}
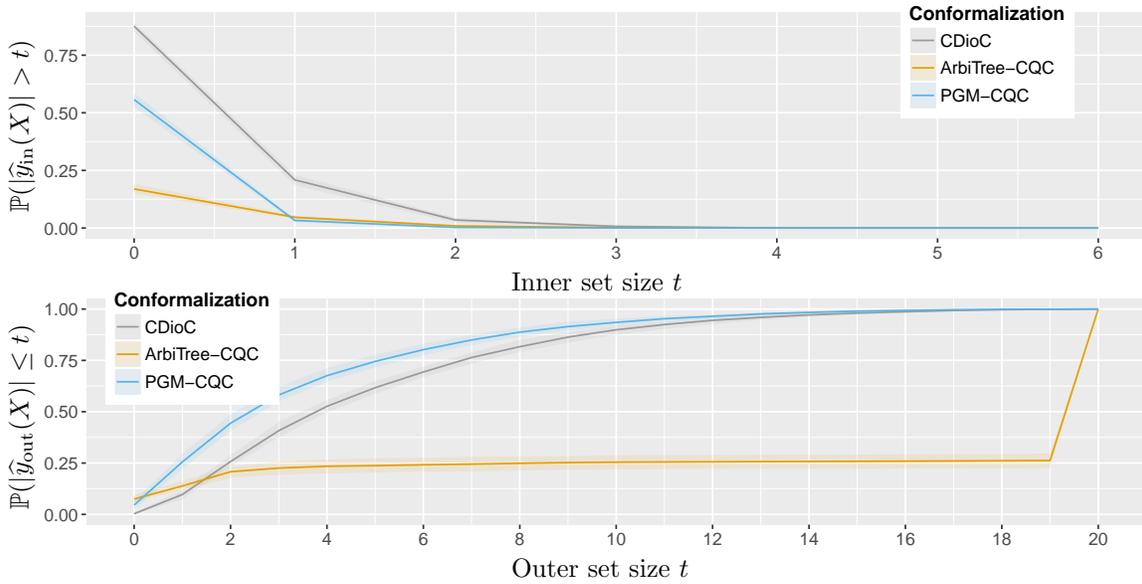

\end{document}